\documentclass[accepted]{uai2026}

\usepackage[american]{babel}

\usepackage[round,sort&compress,comma]{natbib}
\usepackage{amsmath,amssymb,amsthm}
\usepackage{mathtools}
\usepackage{dsfont}
\usepackage{bm}

\usepackage{algorithm}
\usepackage{algpseudocode}

\usepackage{booktabs}
\usepackage{graphicx}
\usepackage{xcolor}
\usepackage{subcaption}
\usepackage{multirow}

\usepackage{placeins}

\usepackage{cleveref}

\newtheorem{theorem}{Theorem}[section]
\newtheorem{lemma}[theorem]{Lemma}
\newtheorem{proposition}[theorem]{Proposition}
\newtheorem{corollary}[theorem]{Corollary}
\theoremstyle{definition}
\newtheorem{definition}[theorem]{Definition}
\newtheorem{assumption}[theorem]{Assumption}

\theoremstyle{remark}
\newtheorem{remark}[theorem]{Remark}

\crefname{theorem}{Theorem}{Theorems}
\Crefname{theorem}{Theorem}{Theorems}
\crefname{lemma}{Lemma}{Lemmas}
\Crefname{lemma}{Lemma}{Lemmas}
\crefname{proposition}{Proposition}{Propositions}
\Crefname{proposition}{Proposition}{Propositions}
\crefname{corollary}{Corollary}{Corollaries}
\Crefname{corollary}{Corollary}{Corollaries}
\crefname{definition}{Definition}{Definitions}
\Crefname{definition}{Definition}{Definitions}
\crefname{assumption}{Assumption}{Assumptions}
\Crefname{assumption}{Assumption}{Assumptions}
\crefname{remark}{Remark}{Remarks}
\Crefname{remark}{Remark}{Remarks}
\crefname{example}{Example}{Examples}
\Crefname{example}{Example}{Examples}
\crefname{algorithm}{Algorithm}{Algorithms}
\Crefname{algorithm}{Algorithm}{Algorithms}

\DeclareMathOperator{\Var}{Var}
\DeclareMathOperator{\Cov}{Cov}
\DeclareMathOperator{\tr}{tr}
\DeclareMathOperator{\diag}{diag}

\newcommand{\R}{\mathbb{R}}
\newcommand{\E}{\mathbb{E}}
\newcommand{\Prob}{\mathbb{P}}
\newcommand{\Id}{\mathbf{I}}

\newcommand{\bX}{\mathbf{X}}
\newcommand{\bY}{\mathbf{Y}}
\newcommand{\bH}{\mathbf{H}}
\newcommand{\bU}{\mathbf{U}}
\newcommand{\bV}{\mathbf{V}}
\newcommand{\bSigma}{\boldsymbol{\Sigma}}
\newcommand{\bK}{\mathbf{K}}
\newcommand{\bW}{\mathbf{W}}

\newcommand{\beps}{\boldsymbol{\varepsilon}}
\newcommand{\cD}{\mathcal{D}}
\newcommand{\cC}{\mathcal{C}}
\newcommand{\cX}{\mathcal{X}}

\newcommand{\fhat}{\hat{f}}
\newcommand{\qhat}{\hat{q}}

\newenvironment{proofsketch}{\begin{proof}[Proof sketch]}{\end{proof}}

\AtBeginDocument{%
  \abovedisplayskip=9pt plus 3pt minus 5pt
  \abovedisplayshortskip=0pt plus 3pt
  \belowdisplayskip=9pt plus 3pt minus 5pt
  \belowdisplayshortskip=5pt plus 3pt minus 2pt
}
\mathtoolsset{showonlyrefs}
\usepackage{mleftright}
\mleftright
\allowdisplaybreaks[1]

\title{Leverage-Weighted Conformal Prediction}

\author[1]{Shreyas Fadnavis}
\affil[1]{%
    \texttt{shreyasfadnavis@gmail.com}
}

\begin{document}
\maketitle

\begin{abstract}\sloppy
Split conformal prediction provides distribution-free prediction intervals with finite-sample marginal coverage, but produces constant-width intervals that overcover in low-variance regions and undercover in high-variance regions. Existing adaptive methods require training auxiliary models. We propose Leverage-Weighted Conformal Prediction (LWCP), which weights nonconformity scores by a function of the statistical leverage---the diagonal of the hat matrix---deriving adaptivity from the geometry of the design matrix rather than from auxiliary model fitting. We prove that LWCP preserves finite-sample marginal validity for any weight function; achieves asymptotically optimal conditional coverage at essentially no width cost when heteroscedasticity factors through leverage; and recovers the form and width of classical prediction intervals under Gaussian assumptions while retaining distribution-free guarantees. We further establish that randomized leverage approximations preserve coverage exactly with controlled width perturbation, and that vanilla CP suffers a persistent, sample-size-independent conditional coverage gap that LWCP eliminates. The method requires no hyperparameters beyond the choice of weight function and adds negligible computational overhead to vanilla CP. Experiments on synthetic and real data confirm the theoretical predictions, demonstrating substantial reductions in conditional coverage disparity across settings.
\end{abstract}

\section{Introduction}\label{sec:intro}

Conformal prediction (CP) has become the dominant paradigm for distribution-free uncertainty quantification \citep{vovk2005algorithmic, shafer2008tutorial}. In regression, split conformal prediction \citep{papadopoulos2002inductive, lei2018distribution} produces intervals $\hat{\cC}_n(x) = \fhat(x) \pm \qhat$, where $\qhat$ is an empirical quantile of calibration residuals. While guaranteeing marginal coverage $\Prob(Y_{n+1} \in \hat{\cC}_n(X_{n+1})) \geq 1{-}\alpha$, these intervals have \emph{constant width}---a well-documented limitation \citep{lei2018distribution, romano2019conformalized, tibshirani2023conformal}---and thus overcover in low-variance regions while undercovering in high-variance regions.

Existing adaptive methods address this by training auxiliary models: conformalized quantile regression \citep[CQR;][]{romano2019conformalized} fits quantile regressors; studentized residual methods \citep{papadopoulos2008normalized, lei2018distribution} fit a variance estimator $\hat{\sigma}(x)$; more recent approaches use Bayesian processes \citep{cabezas2025epicscore} or random forests \citep{amoukou2023adaptive}. All require additional hyperparameters, risk overfitting, and incur non-trivial computational overhead.

We propose a fundamentally different strategy: deriving the normalizer from the \emph{geometry of the design matrix} using \textbf{statistical leverage scores}---the diagonal of the hat matrix $\bH = \bX(\bX^\top\bX)^{-1}\bX^\top$. Leverage scores quantify how far each point lies from the feature distribution centroid and directly govern prediction variance \citep{hoaglin1978hat, chatterjee1986influential}. They provide a closed-form, model-free weighting function for conformal scores that requires no auxiliary model fitting.

\paragraph{Contributions.}
\begin{enumerate}[leftmargin=*,nosep]
    \item \textbf{Framework (\Cref{sec:method}).} We weight nonconformity scores by a function of leverage, yielding prediction intervals of width $2\qhat/w(h(x))$ that adapt to the geometry of the design matrix.

    \item \textbf{Finite-sample coverage (\Cref{thm:coverage}).} LWCP preserves marginal validity ($\geq 1{-}\alpha$, tight to $O(1/n_2)$) for \emph{any} weight function, as a direct consequence of exchangeability.

    \item \textbf{Efficiency (\Cref{thm:efficiency,thm:width_finite}).} Under a linear model with leverage-dependent heteroscedasticity, the variance-stabilized weight $w^*(h) = 1/\sqrt{g(h)}$ achieves asymptotically optimal conditional coverage at negligible width cost (deviation controlled by $\rho^2 = \mathrm{CV}^2(\sqrt{g(H)})$).

    \item \textbf{Classical recovery (\Cref{thm:gaussian_recovery}).} Under Gaussian assumptions, LWCP recovers the form and asymptotic width of classical prediction intervals while retaining distribution-free validity.

    \item \textbf{Approximate leverage (\Cref{thm:approx}).} Randomized SVD approximations preserve coverage exactly with controlled width perturbation, enabling $O(np\log p)$ computation.
\end{enumerate}

\subsection{Related Work}\label{sec:related}

\paragraph{Conformal prediction.} Split conformal prediction \citep{papadopoulos2002inductive, vovk2005algorithmic, lei2018distribution} guarantees marginal coverage but produces constant-width intervals. Distribution-free conditional coverage is impossible in general \citep{vovk2012conditional, barber2021limits}; see \citet{tibshirani2023conformal} for a survey.

\paragraph{Adaptive and localized methods.} Normalized nonconformity scores \citep{papadopoulos2008normalized} and CQR \citep{romano2019conformalized} achieve adaptivity through auxiliary models. Jackknife+ \citep{barber2021jackknife} provides $1{-}2\alpha$ coverage without splitting; localized CP \citep{guan2023localized} uses locality kernels for spatial adaptivity; RLCP \citep{hore2025conformal} provides randomized guarantees; Mondrian CP \citep{vovk2005algorithmic} achieves group-conditional coverage via partitioning. Optimization-based approaches \citep{gibbs2023conformal} pursue group- or feature-conditional coverage; boosted CP \citep{xie2024boosted} iteratively transforms scores to improve conditional coverage; self-calibrating CP \citep{vanderlaan2024selfcalibrating} targets prediction-conditional guarantees; and \citet{plassier2025probabilistic} provide non-asymptotic conditional coverage bounds. LWCP achieves approximate conditional coverage through closed-form geometric reweighting, limited to leverage-dependent heteroscedasticity (\Cref{thm:misspec} quantifies when this suffices via $\eta_{\rm lev}$). Bayesian-conformal hybrids \citep{cabezas2025epicscore} and CLAPS \citep{claps2025laplace} use epistemic uncertainty estimates; \citet{rossellini2024uncertainty} decompose uncertainty into epistemic and aleatoric components for CQR---a decomposition that parallels LWCP+'s use of leverage and variance estimation. CLAPS's Gram matrix $\Phi^\top\Phi$ is precisely the object from which our feature-space leverage is computed (\Cref{app:connections}). CIR/CIR+ \citep{cir2026interquantile} target distributional skewness---a different source of inefficiency than leverage-dependent scale. All these methods require auxiliary model training, kernel design, or posterior computation; LWCP requires none.

\paragraph{Leverage scores.} Introduced in regression diagnostics \citep{hoaglin1978hat, chatterjee1986influential}, leverage scores have found broad use in randomized numerical linear algebra \citep{drineas2006sampling, mahoney2011randomized, drineas2012fast} and ridge regression \citep{alaoui2015fast}. In GLMs, the \emph{working leverage} governs prediction variance analogously \citep{mccullagh1989generalized}; we extend LWCP to this setting in \Cref{app:glm}. For differentiable models, the \emph{gradient leverage} captures prediction variance via the NTK linearization (\Cref{app:gradient_leverage}).

The closest prior use of leverage in conformal inference is Vovk's RRCM \citep{vovk2005algorithmic, burnaev2014efficiency}, which uses LOO scores $|e_i|/(1{-}h_i)$ to correct training-side variance attenuation (\Cref{prop:mismatch}, Part~1). LWCP instead corrects the test-side amplification by $(1{+}h)$ (\Cref{prop:mismatch}, Part~2)---a distinction we formalize in \Cref{app:connections}. Studentized scores \citep{lei2018distribution} replace our closed-form $w(h)$ with a learned $\hat\sigma(x)$; local prediction bands \citep{lei2014distribution} localize via kernel smoothing, while LWCP localizes through the one-dimensional projection $h(x)$, avoiding bandwidth selection. Under collinearity, ridge leverage provides principled stabilization (\Cref{app:collinearity}). To our knowledge, leverage scores have not previously been used as conformal score weights.

\section{Leverage-Weighted Conformal Prediction}\label{sec:method}

\paragraph{Setup.} Data $\{(X_i, Y_i)\}_{i=1}^n$ with $X_i \in \R^p$, $Y_i \in \R$ are split into training $\cD_1$ ($n_1$ points) and calibration $\cD_2$ ($n_2$ points). Features are column-standardized using \emph{training-set} statistics, applied identically to $\cD_2$ and test points; intercepts are absorbed by centering. After standardization, the leverage score $h(x) = x^\top(\bX_1^\top\bX_1)^{-1}x$ equals the squared sample Mahalanobis distance from the training centroid, scaled by $1/n_1$. Preprocessing affects $h(x)$ values but not coverage (\Cref{thm:coverage}); we recommend column-standardization as default (\Cref{app:preprocessing}). The predictor $\fhat$ is trained on $\cD_1$. Via the thin SVD $\bX_1 = \bU_1\bSigma_1\bV_1^\top$, leverage is computed as $h(x) = \|\bSigma_1^{-1}\bV_1^\top x\|_2^2$.

\begin{definition}[Leverage-weighted nonconformity score]\label{def:lwcp_score}
For a measurable function $w : [0, \infty) \to \R_+$, the leverage-weighted score is $V_w(x, y) = |y - \fhat(x)| \cdot w(h(x))$.
\end{definition}

\begin{definition}[LWCP prediction interval]\label{def:lwcp_interval}
The LWCP interval at level $1-\alpha$ is
\begin{equation}\label{eq:lwcp_interval}
\hat{\cC}_n^w(x) = \biggl[\fhat(x) - \frac{\qhat_w}{w(h(x))},\;\; \fhat(x) + \frac{\qhat_w}{w(h(x))}\biggr],
\end{equation}
where $\qhat_w$ is the $\lceil(1-\alpha)(n_2+1)\rceil$-th smallest value among $\{V_w(X_i, Y_i)\}_{i \in \cD_2}$.
\end{definition}

The width at $x$ is $2\qhat_w / w(h(x))$ (cf.\ \Cref{def:lwcp_score,def:lwcp_interval}). Choosing $w$ decreasing in $h$ gives wider intervals for high-leverage points (where prediction is harder) and narrower ones for low-leverage points.

\paragraph{Canonical weighting functions.}
(1)~\textbf{Constant} ($w(h) = 1$): recovers vanilla CP.
(2)~\textbf{Inverse-root leverage} ($w(h) = (1+h)^{-1/2}$): matches the prediction variance structure $\sigma^2(1+h)$ of OLS with homoscedastic errors.
(3)~\textbf{Power-law} ($w(h) = h^{-\gamma}$, $\gamma > 0$): tunable aggressiveness.

\begin{algorithm}[t]
\caption{Leverage-Weighted Conformal Prediction}
\label{alg:lwcp}
\begin{algorithmic}[1]
\Require Data $\{(X_i, Y_i)\}_{i=1}^n$, test point $x_{\text{new}}$, level $\alpha$, weight function $w$
\State Split data into $\cD_1$ ($n_1$ pts) and $\cD_2$ ($n_2$ pts)
\State Standardize features using $\cD_1$ statistics (apply to $\cD_2$ and $x_{\text{new}}$)
\State Train predictor $\fhat$ on $\cD_1$; compute SVD $\bX_1 = \bU_1\bSigma_1\bV_1^\top$
\For{each $i \in \cD_2$}
    \State $h_i \gets \|\bSigma_1^{-1}\bV_1^\top X_i\|_2^2$; \quad $R_i \gets |Y_i - \fhat(X_i)| \cdot w(h_i)$
\EndFor
\State $\qhat_w \gets \lceil(1-\alpha)(n_2+1)\rceil$-th smallest of $\{R_i\}_{i \in \cD_2}$
\State $h_{\text{new}} \gets \|\bSigma_1^{-1}\bV_1^\top x_{\text{new}}\|_2^2$
\State \Return $\fhat(x_{\text{new}}) \pm \qhat_w / w(h_{\text{new}})$
\end{algorithmic}
\end{algorithm}

\paragraph{Computational cost.} The full procedure (\Cref{alg:lwcp}) requires an SVD of $\bX_1$ costing $O(n_1 p^2)$, or $O(n_1 p \log p)$ with randomized SVD. Leverage computation for $n_2$ calibration points costs $O(n_2 p)$. Total: $O(n_1 p^2 + n_2 p)$, with no hyperparameters beyond the choice of $w$.

\paragraph{LWCP+: Combining leverage with residual scale estimation.}\label{par:lwcp_plus}
When heteroscedasticity has both leverage-dependent and feature-dependent components, we extend LWCP to \textbf{LWCP+} with composite scores:
\begin{equation}\label{eq:lwcp_plus}
s_i^+ = \frac{|Y_i - \fhat(X_i)|}{\hat{\sigma}(X_i) \cdot w(h(X_i))},
\end{equation}
where $\hat{\sigma} : \R^p \to \R_+$ is a lightweight scale estimator trained on $\cD_1$ (e.g., a 10-tree random forest on training absolute residuals $\{(X_i, |Y_i - \fhat(X_i)|)\}_{i \in \cD_1}$). The prediction interval at $x$ has half-width $\qhat_{w,\sigma} \cdot \hat{\sigma}(x) / w(h(x))$, adapting to both input-space geometry (via $h$) and conditional noise scale (via $\hat{\sigma}$). LWCP+ recovers vanilla CP ($\hat{\sigma} \equiv 1, w \equiv 1$), LWCP ($\hat{\sigma} \equiv 1$), and studentized CP ($w \equiv 1$) as special cases. Since $\hat{\sigma}$ is $\cD_1$-measurable, marginal coverage is preserved exactly (\Cref{thm:coverage_plus}).

\section{Theoretical Results}\label{sec:theory}

\subsection{Finite-Sample Marginal Coverage}

\begin{theorem}[Marginal coverage]\label{thm:coverage}
Let $(X_i, Y_i)$, $i = 1, \ldots, n+1$, be exchangeable. For any predictor $\fhat$ trained on $\cD_1$ and any measurable $w : [0,\infty) \to \R_+$, the LWCP interval~\eqref{eq:lwcp_interval} satisfies
$\Prob(Y_{n+1} \in \hat{\cC}_n^w(X_{n+1}) \mid \cD_1) \geq 1 - \alpha$.
If scores are a.s.\ distinct, $\Prob(Y_{n+1} \in \hat{\cC}_n^w(X_{n+1}) \mid \cD_1) \leq 1 - \alpha + 1/(n_2 + 1)$.
\end{theorem}
\begin{proof}
Conditional on $\cD_1$ (which determines $\fhat$ and $h(\cdot)$), the scores $R_i = |Y_i - \fhat(X_i)| \cdot w(h(X_i))$ for $i \in \cD_2 \cup \{n+1\}$ are functions of exchangeable $(X_i, Y_i)$, hence exchangeable. By rank uniformity \citep[Lemma~8.1]{vovk2005algorithmic}, the rank of $R_{n+1}$ is uniform over $\{1, \ldots, n_2+1\}$, giving
$\Prob(R_{n+1} \leq \qhat_w \mid \cD_1) \geq \lceil(1{-}\alpha)(n_2{+}1)\rceil/(n_2+1) \geq 1{-}\alpha$.
When scores are a.s.\ distinct, $\Prob \leq 1{-}\alpha + 1/(n_2+1)$. The proof makes \emph{no assumptions} on $P$, $\fhat$, or $w$. \qed
\end{proof}

\subsection{Efficiency Under Heteroscedasticity}\label{sec:efficiency}

\begin{assumption}[Linear model with leverage-dependent scale family]\label{asm:linear}
$Y_i = X_i^\top \beta^* + \varepsilon_i$ where $\varepsilon_i = \sigma \sqrt{g(h(X_i))} \cdot \eta_i$, $\eta_i \stackrel{iid}{\sim} F_\eta$, $\E[\eta_i] = 0$, $\Var(\eta_i) = 1$, and $\{\eta_i\}$ are independent of $\{X_i\}$. The function $g : [0,\infty) \to \R_+$ is continuous, strictly increasing, with $g(0) > 0$.
\end{assumption}

This is stronger than specifying only the conditional variance: the \emph{full conditional distribution} of $\varepsilon_i/\sqrt{g(h_i)}$ must be the same for all $i$, ensuring conditional quantiles of $|\varepsilon_i|$ scale with $\sqrt{g(h_i)}$.\footnote{When heteroscedasticity does not factor through leverage alone, \Cref{thm:misspec} quantifies the degradation: LWCP's gap reduction is $\sqrt{1 - \eta_{\rm lev}}$, where $\eta_{\rm lev}$ measures the fraction of variance explained by leverage. LWCP+ (\Cref{sec:lwcp_plus_theory}) addresses both leverage-dependent and feature-dependent components.} We relax the scale-family assumption in \Cref{app:approx_scale}, providing explicit bounds under approximate scale families. The fixed-$p$ assumption is relaxed in \Cref{prop:high_dim_main} (details in \Cref{app:high_dim}), where we derive coverage guarantees in the proportional regime $p/n_1 \to \gamma \in (0,1)$ using Marchenko--Pastur theory.

\begin{assumption}[Regularity]\label{asm:regularity}
As $n \to \infty$ with $n_1, n_2 \to \infty$ and $p$ fixed: (a)~$\|\hat{\beta} - \beta^*\| = O_p(1/\sqrt{n_1})$; (b)~the empirical leverage distribution converges weakly to $F_h$; (c)~$F_\eta$ is continuous.
\end{assumption}

The \emph{variance-stabilized} weight is $w^*(h) = 1/\sqrt{g(h)}$. The intuition: under the scale family, $|\varepsilon_i|/\sqrt{g(h_i)} = \sigma|\eta_i|$, so dividing by $\sqrt{g(h_i)}$ yields scores that are not merely variance-stabilized but \emph{identically distributed}. This scale-family structure is empirically well-motivated: in linear regression, the conditional distribution of $Y_i - X_i^\top\beta^*$ given $X_i$ inherits its shape from $F_\eta$, with only the scale depending on leverage via $\sqrt{g(h_i)}$.

\begin{theorem}[Efficiency of variance-stabilized LWCP]\label{thm:efficiency}
Under \Cref{asm:linear,asm:regularity}:
\begin{enumerate}
\item \textbf{(Conditional coverage.)} For $F_h$-a.e.\ $h$:
$\Prob(Y_{n+1} \in \hat{\cC}_n^{w^*}(X_{n+1}) \mid h(X_{n+1}) = h) \to 1 - \alpha$.

\item \textbf{(Width parity.)} $\E[|\hat{\cC}_n^{w^*}(X_{n+1})|] / \E[|\hat{\cC}_n^1(X_{n+1})|] \to 1$. LWCP achieves conditional coverage at essentially no marginal width cost.
\end{enumerate}
\end{theorem}

\begin{proofsketch}
\textbf{Part 1.}
Under \Cref{asm:regularity}(a), $Y_i - \fhat(X_i) = \varepsilon_i + o_p(1)$ uniformly. The variance-stabilized scores $R_i^{w^*} = |\varepsilon_i + o_p(1)|/\sqrt{g(h_i)}$ are asymptotically iid under the scale-family \Cref{asm:linear}, since $|\varepsilon_i|/\sqrt{g(h_i)} = \sigma|\eta_i|$. Thus $\qhat_{w^*} \to \sigma F_{|\eta|}^{-1}(1{-}\alpha)$, giving conditional coverage $\to 1{-}\alpha$ for all $h$.
\textbf{Part 2.} Vanilla CP's quantile converges to $q_{\text{mix}} := Q_{1-\alpha}(\sigma\sqrt{g(H)}|\eta|)$; as $p/n \to 0$, leverages concentrate, the scale mixture degenerates, and the width ratio $\to 1$. Full proof in \Cref{app:main_proofs}.
\end{proofsketch}

\begin{theorem}[Non-asymptotic width parity]\label{thm:width_finite}
Under \Cref{asm:linear,asm:regularity}, let $\rho^2 = \mathrm{CV}^2(\!\sqrt{g(H)})$. Then:
\[
\frac{\E_H[|\hat{\cC}^{w^*}\!(X)|]}{\E_H[|\hat{\cC}^1(X)|]} = 1 + C_\alpha \rho^2 + O(\rho^3) + O(1/\!\sqrt{n_2}),
\]
where $C_\alpha = \tfrac{1}{2}(1 + q_\alpha f'_{|\eta|}(q_\alpha)/f_{|\eta|}(q_\alpha))$ with $q_\alpha = F_{|\eta|}^{-1}(1{-}\alpha)$. For Gaussian $\eta$: $C_\alpha = (1 - q_\alpha^2)/2 < 0$ whenever $\alpha < 0.32$, so LWCP is actually \emph{narrower} than vanilla CP. For $g \equiv 1$: $\rho = 0$ and the ratio is exactly~$1$.
\end{theorem}
\begin{proofsketch}
Write $G = \sqrt{g(H)}$, $\mu_G = \E[G]$, and $G = \mu_G(1{+}\delta)$ with $\Var(\delta) = \rho^2$. Taylor-expanding the scale-mixture quantile $Q_{1-\alpha}((1{+}\delta)|\eta|)$ around $\delta = 0$ gives $C_\alpha$ from the log-concavity correction $q_\alpha f'_{|\eta|}(q_\alpha)/f_{|\eta|}(q_\alpha)$. For Gaussian $\eta$, $f'_{|\eta|}(t) = -tf_{|\eta|}(t)$, yielding $C_\alpha = (1{-}q_\alpha^2)/2$. Finite-sample correction is $O(1/\sqrt{n_2})$ via DKW. Full proof in \Cref{app:width_parity}.
\end{proofsketch}

At $p/n_1 = 0.1$ with $g(h) = 1+h$: $\rho^2 \approx 1.7 \times 10^{-4}$, consistent with observed ratios of 0.999--1.000 (\Cref{tab:marginal_width}).

\begin{proposition}[Conditional coverage gap]\label{prop:cond_gap}
Under \Cref{asm:linear,asm:regularity} with $w = w^*$:
$\sup_h |\Prob(Y_{n+1} \in \hat{\cC}_n^{w^*}(X_{n+1}) \mid h(X_{n+1}) = h) - (1-\alpha)| = O(1/\sqrt{n_2}) + O(\sqrt{p/n_1})$.
\end{proposition}
\begin{proofsketch}
The $O(1/\sqrt{n_2})$ term is the standard quantile estimation error (DKW inequality). The $O(\sqrt{p/n_1})$ term bounds $\|\hat\beta - \beta^*\| = O_p(\sqrt{p/n_1})$ from \Cref{asm:regularity}(a): the residuals $Y_i - \fhat(X_i) = \varepsilon_i + X_i^\top(\beta^* - \hat\beta)$ differ from $\varepsilon_i$ by this estimation error, which propagates through the score transformation. Proof in \Cref{app:main_proofs}.
\end{proofsketch}

\begin{remark}[Relation to impossibility results]\label{rem:impossibility}
Distribution-free conditional coverage $\Prob(Y \in \hat\cC(X) \mid X = x) \geq 1 - \alpha$ for all $x$ is impossible in general \citep{vovk2012conditional, barber2021limits}. LWCP circumvents this by targeting a weaker but practically useful guarantee: coverage conditional on the \emph{one-dimensional} projection $h(X)$, which suffices when heteroscedasticity factors through leverage (\Cref{asm:linear}). The scale family turns $h(X)$ into a sufficient statistic for the conditional distribution of $|Y - \fhat(X)|$, making conditional calibration feasible. This is analogous to Mondrian CP \citep{vovk2005algorithmic}, which achieves exact group-conditional coverage via finite partitions; LWCP achieves approximate $h$-conditional coverage via continuous reweighting.
\end{remark}

The key insight is that $w^*$ renders scores approximately iid, so the conformal quantile tracks the conditional quantile $Q_{1-\alpha}(\sigma|\eta|)$ uniformly in $h$. For vanilla CP, the scale mixture creates a persistent gap that does not vanish with sample size:

\begin{theorem}[Persistent $\Theta(1)$ gap for vanilla CP]\label{thm:persistent_gap}
Under \Cref{asm:linear} with $g$ non-constant and \Cref{asm:regularity}, for any $h_1, h_2$ with $g(h_1) < g(h_2)$, the asymptotic conditional coverage difference $\Delta(h_1, h_2) > 0$ does \emph{not} vanish with sample size. Quantitatively:
\[
\Delta(h_1, h_2) \geq f_{|\eta|}(q_\alpha) \cdot q_\alpha \cdot \frac{\sqrt{g(h_2)} - \sqrt{g(h_1)}}{\sqrt{g(h_2)}} \cdot (1 + o(1)).
\]
For $g(h) = 1{+}h$, Gaussian errors, $\alpha = 0.1$: $\Delta(h_{\min}, h_{\max}) = \Theta(h_{\max} - h_{\min})$.
\end{theorem}
\begin{proofsketch}
As $n \to \infty$, vanilla CP's quantile $\qhat_1 \to q_{\rm mix} := Q_{1-\alpha}(\sigma\sqrt{g(H)}|\eta|)$. Conditional on leverage $h$, asymptotic coverage is $F_{|\eta|}(q_{\rm mix}/(\sigma\sqrt{g(h)}))$. Since $F_{|\eta|}$ is strictly increasing and $g(h_1) \neq g(h_2)$, the coverage difference is positive and \emph{independent of $n$}. The quantitative bound follows from the mean value theorem. Proof in \Cref{app:theta_one}.
\end{proofsketch}

\subsection{Classical Recovery}\label{sec:recovery}

\begin{theorem}[Gaussian recovery]\label{thm:gaussian_recovery}
Under $Y_i = X_i^\top\beta^* + \varepsilon_i$ with $\varepsilon_i \stackrel{iid}{\sim} \mathcal{N}(0, \sigma^2)$ and $\fhat = $ OLS, setting $w(h) = (1+h)^{-1/2}$: as $n_1, n_2 \to \infty$ with $p$ fixed, the LWCP interval width at $x$ converges to $2\sigma z_{1-\alpha/2}\sqrt{1 + h(x)}$, recovering the form of the classical prediction interval (with $z$- vs.\ $t$-quantile and sample-splitting as finite-sample differences).
\end{theorem}
\begin{proofsketch}
Under Gaussianity, the variance-stabilized scores $R_i^w = |Y_i - \fhat(X_i)| / \sqrt{1+h_i}$ are distributed as $\sigma|Z_i|$ (iid half-normal), so $\qhat_w \to \sigma z_{1-\alpha/2}$, recovering the classical half-width $\sigma z_{1-\alpha/2}\sqrt{1+h(x)}$. See \Cref{app:main_proofs} for the complete argument.
\end{proofsketch}

\subsection{Approximate Leverage Scores}\label{sec:approx_theory}

\begin{theorem}[Coverage under approximate leverage]\label{thm:approx}
Let $\tilde{h}_i$ be $\epsilon$-approximate leverage scores ($(1-\epsilon)h_i \leq \tilde{h}_i \leq (1+\epsilon)h_i$). Then: (1)~Marginal coverage is preserved \emph{exactly}; (2)~Width perturbation is bounded by $O(L\epsilon\|h\|_\infty / w_{\min})$ for $L$-Lipschitz $w$ bounded below by $w_{\min}$.
\end{theorem}

Such approximations are computable in $O(n_1 p \log p)$ time via randomized projections \citep{drineas2012fast, clarkson2013low, woodruff2014sketching}; see \Cref{app:exp_approx,app:exp_approx_eps} for experiments.
\begin{proofsketch}
Part~(1): $\tilde{h}$ is computed from $\cD_1$ alone, so $\tilde{w}(x) = w(\tilde{h}(x))$ is $\cD_1$-measurable. The exchangeability of $\{(X_i, Y_i)\}_{i \in \cD_2 \cup \{n+1\}}$ conditional on $\cD_1$ is unaffected, and \Cref{thm:coverage} applies verbatim. Part~(2): By Lipschitz continuity, $|w(\tilde{h}) - w(h)| \leq L|\tilde{h} - h| \leq L\epsilon h$; the width perturbation follows from bounding $w(h)/w(\tilde{h})$. Full proof in \Cref{app:main_proofs}.
\end{proofsketch}

\subsection{Optimal Weight}

\begin{proposition}[Optimality of $w^*$]\label{prop:optimality}
Under \Cref{asm:linear}, among all $w$ achieving asymptotic conditional coverage for $F_h$-a.e.\ $h$, the variance-stabilized weight $w^*(h) = 1/\sqrt{g(h)}$ achieves minimum expected width.
\end{proposition}

Proof in \Cref{app:main_proofs}: matching the LWCP half-width $\qhat_w/w(h)$ to the conditional quantile $\sigma\sqrt{g(h)} Q_{1-\alpha}(|\eta|)$ for all $h$ forces $w(h) \propto 1/\sqrt{g(h)} = w^*(h)$.

\subsection{Robustness and Extensions}\label{sec:robustness}

The preceding theory assumes the scale-family model (\Cref{asm:linear}). We now characterize LWCP under misspecification and extend to high-dimensional settings.

\begin{theorem}[LWCP under general heteroscedasticity]\label{thm:misspec}
Let $\Var(\varepsilon|X{=}x) = \sigma^2(x)$ be arbitrary (not necessarily leverage-dependent). Define $\sigma_h^2(h) := \E[\sigma^2(X) \mid h(X) = h]$ and $\eta_{\rm lev} := \Var(\sigma_h^2(H)) / \Var(\sigma^2(X)) \in [0,1]$, measuring the fraction of variance heterogeneity explained by leverage. Then:
\begin{enumerate}
\item \textbf{Marginal validity is preserved} regardless of misspecification.
\item \textbf{Gap comparison:}
$\mathrm{gap}_{\rm LWCP}/\mathrm{gap}_{\rm Vanilla} \leq \sqrt{1 - \eta_{\rm lev}} + O(1/\sqrt{n})$.
\end{enumerate}
When $\eta_{\rm lev} \approx 1$ (heteroscedasticity is leverage-driven), LWCP nearly eliminates the gap. When $\eta_{\rm lev} \approx 0$ (heteroscedasticity is orthogonal to leverage), LWCP reduces to vanilla CP. In both cases, it never hurts marginal validity; under approximate scale families, explicit bounds are in \Cref{thm:approx_scale}. Full gap decomposition in \Cref{app:misspecification}.
\end{theorem}
\begin{proofsketch}
Part~(i) is \Cref{thm:coverage}. For~(ii), decompose $\sigma^2(x) = \sigma_h^2(h(x)) \cdot \psi^2(x)$ where $\E[\psi^2 | h] = 1$. Under LWCP with $w(h) = 1/\sigma_h(h)$, the residual score heterogeneity is $\mathrm{CV}(\psi|h)$---the component \emph{orthogonal} to leverage. By the law of total variance, $\Var(\sigma^2(X)) = \Var(\sigma_h^2(H)) + \E[\Var(\sigma^2(X)|H)]$, giving the stated ratio via $\eta_{\rm lev}$. Quantile sensitivity (\Cref{lem:quantile_sensitivity}) converts CV to gap. Proof in \Cref{app:misspecification}.
\end{proofsketch}

\begin{theorem}[Ridge-LWCP]\label{thm:ridge_lwcp}
When $p \geq n_1$ or $\bX_1$ is ill-conditioned, ridge leverage $h^\lambda(x) = x^\top(\bX_1^\top\bX_1 + \lambda\Id)^{-1}x$ with $\lambda > 0$ satisfies: (1)~exact marginal coverage for any $w$ and $\lambda$; (2)~well-posedness: $h^\lambda(x) \in [0, \|x\|^2/\lambda]$ is always finite; (3)~under ridge regression, $\Var(Y_{\text{new}} - \hat{f}_\lambda(x) \mid x) = \sigma^2(1 + h^\lambda(x)) + \lambda^2\|b_\lambda(x)\|^2$, where $b_\lambda$ is the bias. Proof in \Cref{app:extensions}.
\end{theorem}

The following proposition completes the robustness picture by characterizing the worst case for LWCP.

\begin{proposition}[Adversarial worst case]\label{prop:adv}
For ``anti-aligned'' heteroscedasticity $g_{\rm adv}(h) = c/(1{+}h)$:
\begin{enumerate}[nosep]
\item Marginal coverage is preserved exactly.
\item LWCP has \emph{worse} conditional coverage than vanilla CP since total prediction variance is approximately constant.
\item General degradation bound: $\mathrm{gap}_{\rm LWCP}^{w} - \mathrm{gap}_{\rm LWCP}^{w^*} \leq C_\alpha \cdot \mathrm{CV}^2\!\bigl(\!\sqrt{(1{+}H)/g(H)}\bigr)$. For $g_{\rm adv}$, this is $O(p/n_1^2)$, making the degradation negligible for moderate $p/n_1$.
\end{enumerate}
\end{proposition}
\begin{proofsketch}
Part~(i) is \Cref{thm:coverage}. Under $g_{\rm adv}$, the optimal weight is $w^*(h) = \sqrt{(1{+}h)/c}$, which \emph{increases} with $h$---opposite to LWCP's $(1{+}h)^{-1/2}$. The degradation is $\mathrm{CV}^2(w^*/w) = \mathrm{CV}^2(1{+}H) = O(p/n_1^2)$ by concentration of Marchenko--Pastur leverage moments. Proof in \Cref{app:adversarial}.
\end{proofsketch}

This adversarial scenario is highly unnatural---noise must \emph{decrease} for extreme points---and the $\hat\eta$ diagnostic (\Cref{rem:diagnostic}) detects it. Our experiment confirms LWCP's gap (2.0pp) $\approx$ vanilla's (2.1pp) even here (\Cref{tab:marginal_width}). More fundamentally, the default weight has bounded downside:

\begin{proposition}[Safe default]\label{prop:safe_default_main}
For \emph{any} predictor $\fhat$, \emph{any} distribution $P$, and $w(h) = (1{+}h)^{-1/2}$:
\begin{enumerate}[nosep]
\item Marginal coverage $\geq 1-\alpha$ holds exactly (\Cref{thm:coverage}).
\item Expected width: $\E[|\hat\cC^w|] / \E[|\hat\cC^1|] = 1 + O(\mathrm{CV}^2(H))$, i.e., within $O(1/n_1)$ of vanilla CP.
\item Worst-case gap increase: $\mathrm{gap}^w - \mathrm{gap}^1 \leq C_\alpha \cdot \mathrm{CV}^2(H) = O(\gamma(1{-}\gamma)/n_1)$. At $\gamma{=}0.1$, $n_1{=}300$: ${<}0.01$pp.
\end{enumerate}
The downside is $O(1/n_1)$; the upside is $\Theta(1)$ when heteroscedasticity aligns with leverage.
\end{proposition}
\noindent Proof in \Cref{app:safe_default}. The preceding theory assumes $p$ fixed. Under proportional asymptotics ($p/n_1 \to \gamma$), LWCP's advantage grows:

\begin{proposition}[High-dimensional gap]\label{prop:high_dim_main}
Under \Cref{asm:linear} with $g(h) = 1{+}h$ and $p/n_1 \to \gamma \in (0,1)$:
\begin{enumerate}[nosep]
\item Vanilla CP: $\Delta_{\rm Vanilla} = \Theta(\sqrt{\gamma(1{-}\gamma)/n_1})$.
\item LWCP: $\Delta_{\rm LWCP} = O(1/\sqrt{n_2}) + O(\gamma/\sqrt{n_1})$.
\item Ratio: $\Delta_{\rm LWCP}/\Delta_{\rm Vanilla} \leq C/\sqrt{1+\gamma}$ for a universal constant $C < 1$, so LWCP strictly reduces the gap at all aspect ratios $\gamma \in (0,1)$.
\end{enumerate}
At $\gamma{=}0.9$, $n_1{=}300$: gap reduces from $12.4$pp to $3.3$pp ($3.8{\times}$; \Cref{tab:high_dim}).
\end{proposition}
\noindent Proof via Marchenko--Pastur concentration in \Cref{app:high_dim}; width parity is preserved (\Cref{prop:width_parity_highdim}). Fairness analysis (\Cref{prop:fairness}) in \Cref{app:fairness}.

\begin{remark}[Cross-conformal extension]\label{rem:cross_conformal}
LWCP extends to cross-conformal prediction \citep{vovk2015cross}: for each fold $k$, compute the SVD of $\bX_{-k}$ and weight the residuals by $w(h_k(\cdot))$. Finite-sample validity is preserved by exchangeability, at cost $O(Kn_1 p^2)$. See \Cref{app:connections}.
\end{remark}

\subsection{LWCP+ Theory}\label{sec:lwcp_plus_theory}

The efficiency results above require \Cref{asm:linear} (scale family). We now develop theory for LWCP+ that applies under weaker conditions.

\begin{theorem}[LWCP+ marginal coverage]\label{thm:coverage_plus}
Let $(X_i, Y_i)$, $i=1,\ldots,n+1$, be exchangeable. For any predictor $\fhat$ and scale estimator $\hat{\sigma}$ both trained on $\cD_1$, and any measurable $w : [0,\infty) \to \R_+$, the LWCP+ interval with scores~\eqref{eq:lwcp_plus} satisfies $\Prob(Y_{n+1} \in \hat{\cC}_n^{w,\hat\sigma}(X_{n+1}) \mid \cD_1) \geq 1-\alpha$.
\end{theorem}
\begin{proofsketch}
Both $\hat{\sigma}(\cdot)$ and $w(h(\cdot))$ are $\cD_1$-measurable. Conditional on $\cD_1$, the composite scores $s_i^+$ for $i \in \cD_2 \cup \{n+1\}$ are functions of exchangeable $(X_i, Y_i)$; the argument of \Cref{thm:coverage} applies verbatim.
\end{proofsketch}

\begin{proposition}[Prediction variance decomposition]\label{prop:variance_decomp}
Under a linear model $Y = X^\top\beta^* + \varepsilon$ with $\fhat = $ OLS on $\cD_1$, the prediction error at a new point $x$ satisfies:
\[
\Var(Y_{\mathrm{new}} - \fhat(x) \mid x, \cD_1) = \underbrace{\Var(\varepsilon_{\mathrm{new}} \mid x)}_{\text{noise}} + \underbrace{\sigma_{\mathrm{OLS}}^2 \cdot h(x)}_{\text{estimation}}.
\]
LWCP corrects the second term; studentized CP targets the first; LWCP+ addresses both.
\end{proposition}
\begin{proofsketch}
Write $Y_{\rm new} - \fhat(x) = \varepsilon_{\rm new} - x^\top(\hat\beta - \beta^*)$. Independence of $\varepsilon_{\rm new}$ and $\hat\beta - \beta^*$ gives variance additivity. Under homoscedastic errors, $\Cov(\hat\beta \mid \bX_1) = \sigma^2(\bX_1^\top\bX_1)^{-1}$, so $x^\top\Cov(\hat\beta)x = \sigma^2 h(x)$. Proof in \Cref{app:main_proofs}.
\end{proofsketch}

\begin{proposition}[Training-test variance mismatch]\label{prop:mismatch}
Under the homoscedastic linear model ($\varepsilon_i \stackrel{iid}{\sim} (0, \sigma^2)$, $\fhat = $ OLS):
\begin{enumerate}
\item \textbf{Training residuals:} $\Var(Y_i - \fhat(X_i) \mid \bX_1) = \sigma^2(1 - h_i)$ for $i \in \cD_1$.
\item \textbf{Prediction errors:} $\Var(Y_{\mathrm{new}} - \fhat(x) \mid x, \bX_1) = \sigma^2(1 + h(x))$ for new $x$.
\end{enumerate}
The sign of the leverage contribution flips: training residuals are \emph{attenuated} at high leverage (factor $1-h$), while prediction errors are \emph{amplified} (factor $1+h$).
\end{proposition}
\begin{proofsketch}
Part~1: Write $e_i^{\rm train} = ((I - \bH)\beps_1)_i$, where $\bH = \bX_1(\bX_1^\top\bX_1)^{-1}\bX_1^\top$ is idempotent. Then $\Var(e_i) = \sigma^2(I - \bH)_{ii} = \sigma^2(1-h_i)$. Part~2: For $i \in \cD_2$, $\varepsilon_{\rm new} \perp\!\!\!\perp \beps_1$, so variances add: $\sigma^2 + \sigma^2 h(x) = \sigma^2(1{+}h(x))$. Proof in \Cref{app:main_proofs}.
\end{proofsketch}

\begin{corollary}[LWCP+ corrects the mismatch]\label{cor:mismatch}
A scale estimator $\hat{\sigma}$ trained on $\{(X_i, |e_i^{\mathrm{train}}|)\}_{i \in \cD_1}$ approximates $\sigma\sqrt{1 - h(\cdot)}$, not $\sigma\sqrt{1+h(\cdot)}$. Studentized CP has $h$-dependent factor $\phi_S(h) = \sqrt{(1{+}h)/(1{-}h)}$; LWCP+ divides additionally by $\sqrt{1{+}h}$, giving factor $\phi_+(h) = 1/\sqrt{1{-}h}$. The CV$^2$ ratio is:
\begin{equation}\label{eq:cv_ratio_main}
\frac{\mathrm{CV}^2(\phi_S(H))}{\mathrm{CV}^2(\phi_+(H))} = \frac{4}{(1+\bar{h})^2} + O(\Var(H)), \quad \bar{h} = p/n_1.
\end{equation}
For $p/n_1 \to 0$ this approaches $4$: LWCP+ reduces the $h$-dependent score variability by $4\times$ (in CV$^2$), giving $\approx 2\times$ smaller conditional gap than studentized CP (\Cref{thm:quant_mismatch}).
\end{corollary}

The factor of $4$ has a clean explanation: Studentized CP's score is affected by leverage through \emph{both} numerator $(\sqrt{1{+}h})$ and denominator $(\sqrt{1{-}h})$, giving log-sensitivity $d\ln\phi_S/dh = 1/(1{-}h^2)$. LWCP+'s $\sqrt{1{+}h}$ correction cancels the numerator, halving the log-sensitivity to $1/(2(1{-}h))$. Since $\mathrm{CV}^2$ scales as the square of log-sensitivity, the ratio is $2^2 = 4$.

\begin{remark}[Diagnostics and robustness]\label{rem:diagnostic}
The coefficient of variation $\hat\eta = \operatorname{std}(h)/\operatorname{mean}(h)$ is a practical diagnostic: $\hat\eta > 1$ indicates LWCP improves conditional coverage; $\hat\eta < 0.5$ indicates vanilla CP suffices. Under misspecification (\Cref{thm:misspec}), LWCP corrects only the leverage-dependent component, reducing the gap by $\sqrt{1 - \eta_{\rm lev}}$. Marginal validity is always preserved.\label{rem:misspec}
\end{remark}

\begin{remark}[Feature-space leverage for non-linear models]\label{rem:feature_leverage}
For predictors $\fhat$ linear in learned features $\phi(x)$ (e.g., last-layer neural networks), the feature-space leverage $h^{\phi}(x) = \phi(x)^\top(\Phi^\top\Phi + \lambda\Id)^{-1}\phi(x)$ inherits the coverage guarantee (\Cref{thm:coverage}) since it is $\cD_1$-measurable (\Cref{prop:feature_leverage}); analogues for GLMs (\Cref{prop:glm_lwcp}) and gradient leverage (\Cref{prop:gradient_leverage}) are in \Cref{app:extended_theory}. In the NTK/lazy regime, the efficiency theory (\Cref{thm:efficiency}) applies to the last-layer linearization (\Cref{app:connections,rem:ntk}). When $p \geq n_1$, ridge leverage (\Cref{thm:ridge_lwcp}) ensures well-posedness.
\end{remark}

\section{Experiments}\label{sec:experiments}

We evaluate LWCP with OLS as $\fhat$ and $w(h) = (1+h)^{-1/2}$. Default parameters: $n_1 = 300$, $n_2 = 500$, $p = 30$, $\sigma = 1$, $\alpha = 0.1$ (90\% nominal coverage), giving $p/n_1 = 0.1$.

\paragraph{Data-generating processes.} Five DGPs with $X \sim \mathcal{N}(0, \bSigma)$, $\bSigma = \diag(1, 1/2, \ldots, 1/p)$:
(1)~\textbf{Textbook}: $\varepsilon \mid X \sim \mathcal{N}(0, \sigma^2(1+h))$, matching the weight function;
(2)~\textbf{Heavy-tailed}: same variance structure but $t_3$ errors;
(3)~\textbf{Polynomial}: degree-8 polynomial features with $g(h) = 1+h$;
(4)~\textbf{Homoscedastic}: $\varepsilon \sim \mathcal{N}(0, \sigma^2)$, the null case;
(5)~\textbf{Adversarial}: $\Var(\varepsilon \mid X) \propto 1 + \|X\|^2/p$, misspecified for LWCP (\Cref{prop:adv}).

\subsection{Coverage Results}\label{sec:exp_marginal}

\begin{table}[t]
\centering
\caption{Marginal Coverage (Mean $\pm$ Std, 1{,}000 Reps), Width Ratio, and Conditional Gap (200 Reps).}
\label{tab:marginal_width}
\resizebox{\linewidth}{!}{%
\begin{tabular}{@{}lcccrcc@{}}
\toprule
& \multicolumn{2}{c}{Marginal Coverage} & Width & \multicolumn{2}{c}{Cond.\ Gap} \\
\cmidrule(lr){2-3}\cmidrule(lr){5-6}
DGP & Vanilla & LWCP & Ratio & Van.\ & LWCP \\
\midrule
Textbook      & $.900 {\pm} .019$ & $.900 {\pm} .019$ & 1.000 & 2.7pp & 1.3pp \\
Heavy-tailed  & $.900 {\pm} .019$ & $.900 {\pm} .019$ & 0.999 & 2.3pp & 1.2pp \\
Polynomial    & $.900 {\pm} .020$ & $.900 {\pm} .020$ & 0.999 & 2.2pp & 1.1pp \\
Homoscedastic & $.900 {\pm} .019$ & $.900 {\pm} .019$ & 1.000 & 4.5pp & 0.1pp \\
Adversarial   & $.900 {\pm} .019$ & $.900 {\pm} .019$ & 1.000 & 2.1pp & 2.0pp \\
\bottomrule
\end{tabular}%
}
\end{table}

\Cref{tab:marginal_width} confirms marginal coverage is preserved across all DGPs, with width ratio $0.999$--$1.000$ (distributions in \Cref{app:exp_marginal}). The homoscedastic DGP exhibits the largest conditional gap reduction (4.5pp to 0.1pp), consistent with the theory: $(1{+}h)^{-1/2}$ exactly stabilizes the prediction variance $\sigma^2(1{+}h)$.

\subsection{Conditional Coverage}\label{sec:exp_conditional}

\begin{figure*}[t]
\centering
\includegraphics[width=\textwidth]{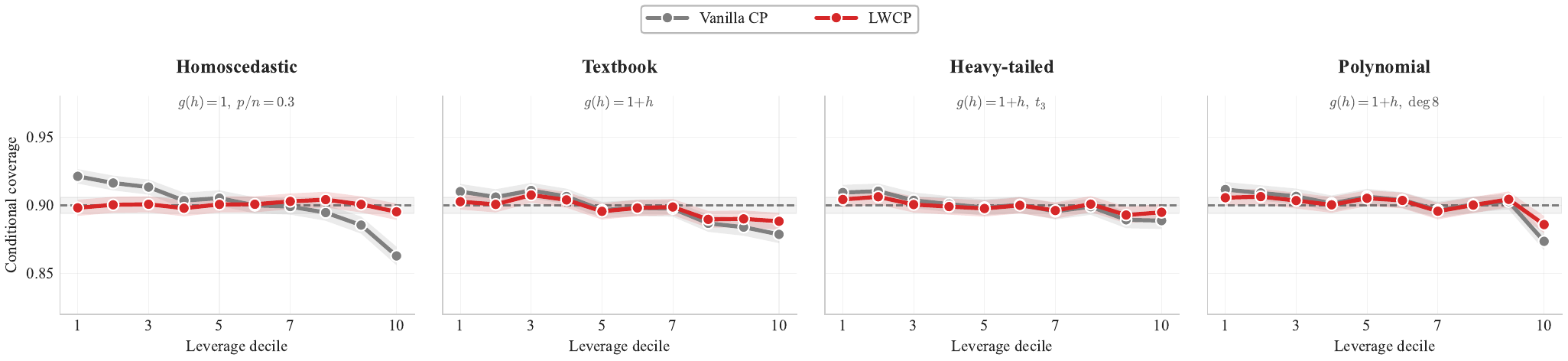}
\caption{\textbf{Conditional coverage by leverage decile} (200 replications). Vanilla CP (gray) exhibits monotone undercoverage at high leverage. LWCP (red) achieves approximately flat conditional coverage across all DGPs, with the largest improvement under homoscedastic errors ($p/n_1 = 0.3$) where $(1{+}h)^{-1/2}$ exactly stabilizes the prediction variance.}
\label{fig:conditional}
\end{figure*}

\Cref{fig:conditional} visualizes the central finding. Under homoscedastic errors ($g \equiv 1$), OLS prediction variance is exactly $\sigma^2(1{+}h)$, and LWCP's weight $(1{+}h)^{-1/2}$ perfectly stabilizes it, yielding nearly flat conditional coverage. Under heteroscedastic DGPs ($p/n_1 = 0.1$), LWCP approximately halves the conditional gap. See \Cref{app:exp_hetero} for sensitivity sweeps and \Cref{app:exp_width} for width-vs-leverage scatter plots.

\subsection{Adaptive Widths}\label{sec:exp_baselines}

LWCP redistributes width---wider for high-leverage points, narrower for low-leverage---with negligible net change (\Cref{fig:width_scatter}). We compare against CQR \citep{romano2019conformalized}, studentized CP \citep{papadopoulos2008normalized}, and localized CP \citep{guan2023localized} in \Cref{tab:baselines,fig:baselines_conditional}; configurations and runtime details are in \Cref{app:runtime}.

\begin{table}[t]
\centering
\caption{MSCE ($\times 10^3$; $\downarrow$) Across Methods and DGPs (20 Reps). All methods achieve ${\approx}.90$ coverage.}
\label{tab:baselines}
\footnotesize
\setlength{\tabcolsep}{4pt}
\begin{tabular}{@{}lcccc@{}}
\toprule
& \multicolumn{3}{c}{MSCE} & \\
\cmidrule(lr){2-4}
Method & Text.\ & Poly.\ & Adv.\ & Time \\
\midrule
Vanilla CP     & 2.21 & 2.41 & 2.26 & 0.001s \\
\textbf{LWCP}  & \textbf{2.18} & \textbf{2.37} & \textbf{2.15} & \textbf{0.001s} \\
\addlinespace
CQR            & 1.94 & 2.82 & 2.01 & 0.170s \\
Studentized    & 1.99 & 3.76 & 2.11 & 0.127s \\
Localized      & 2.16 & 2.55 & 2.18 & 0.023s \\
\bottomrule
\end{tabular}
\end{table}

\begin{figure*}[t]
\centering
\includegraphics[width=\textwidth]{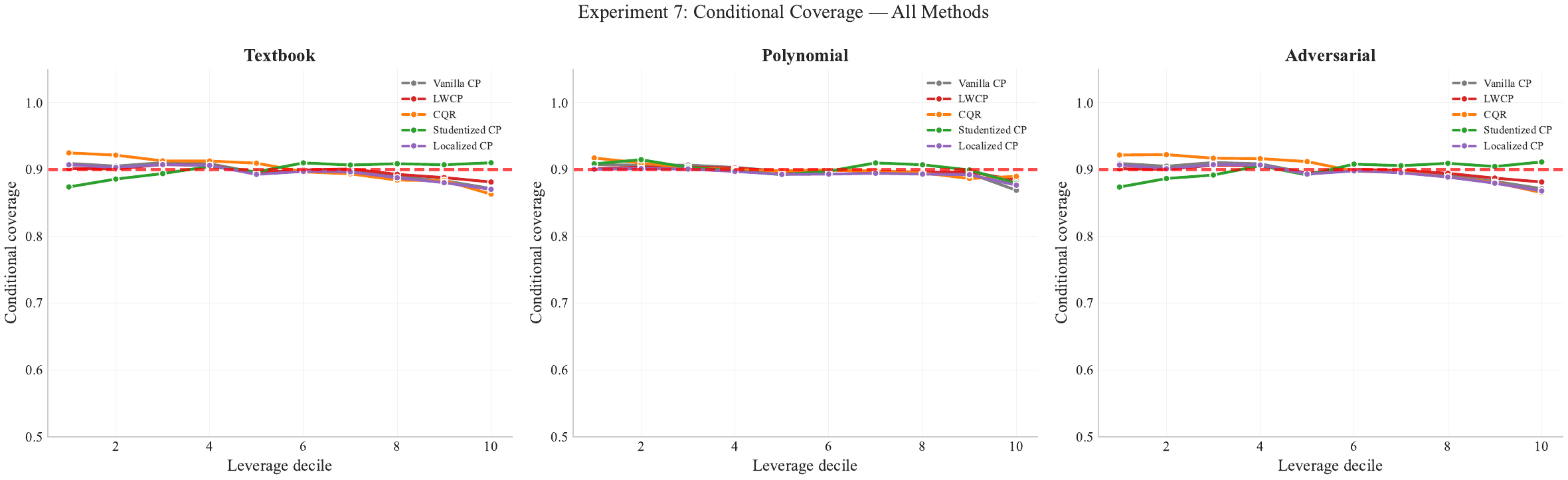}
\caption{\textbf{Conditional coverage across methods.} LWCP (red) achieves the flattest coverage profile at the lowest computational cost. CQR (orange) attains comparable flatness but with substantially wider intervals. Studentized CP (green) achieves moderate improvement at significantly higher runtime.}
\label{fig:baselines_conditional}
\end{figure*}

\Cref{thm:gaussian_recovery} is validated at $n{=}1000$: the LWCP/classical width ratio is $1.002 \pm 0.046$ (\Cref{app:exp_gaussian}), and the gap convergence rate of \Cref{prop:cond_gap} is confirmed in \Cref{app:exp_scaling}.

\subsection{Beyond Synthetic Data}\label{sec:exp_real}

We evaluate on four real-world datasets (10 reps). \Cref{tab:real_summary} summarizes the two most informative; the full comparison is in \Cref{app:exp_real}. On CPU Activity ($\hat{\eta}=2.96$), LWCP+ substantially reduces both the conditional gap and MSCE, outperforming Studentized CP at lower computational cost---confirming the mismatch theory (\Cref{cor:mismatch}). On Diabetes ($\hat{\eta}=0.67$), LWCP achieves the lowest MSCE. When $\hat{\eta}$ is small, LWCP reduces to vanilla CP, validating the diagnostic (\Cref{rem:diagnostic}).

\begin{table}[t]
\centering
\caption{Real-Data Summary (10 Reps). Gap = $|\mathrm{cov}(\text{low-}h) - \mathrm{cov}(\text{high-}h)|$. MSCE ($\times 10^3$; $\downarrow$). Full comparison in \Cref{app:exp_real}.}
\label{tab:real_summary}
\resizebox{\linewidth}{!}{%
\begin{tabular}{@{}l ccc ccc@{}}
\toprule
& \multicolumn{3}{c}{CPU ($\hat\eta{=}2.96$)} & \multicolumn{3}{c}{Diabetes ($\hat\eta{=}0.67$)} \\
\cmidrule(lr){2-4}\cmidrule(lr){5-7}
Method & Cov & Gap & MSCE & Cov & Gap & MSCE \\
\midrule
Vanilla CP     & .899 & .189 & 6.59 & .899 & .061 & 6.41 \\
LWCP           & .899 & .188 & 6.49 & .900 & .056 & \textbf{6.02} \\
\textbf{LWCP+} & .903 & \textbf{.075} & \textbf{1.61} & .913 & .067 & 8.32 \\
Studentized    & .905 & .084 & 2.05 & .898 & .078 & 10.7 \\
\bottomrule
\end{tabular}%
}
\end{table}

\subsection{Non-Linear Predictors}\label{sec:exp_nonlinear}

We evaluate with Random Forest \citep[RF;][]{lin2006random} using raw-$\bX$ leverage and MLP (feature-space leverage $h^{\rm NN}$; \Cref{rem:feature_leverage}) on a non-linear DGP ($Y = \sum_j \sin(X_j) + \varepsilon$, 10 reps). Coverage holds for any $\fhat$ (\Cref{thm:coverage}); LWCP reduces conditional gaps across all predictor types, with feature-space leverage providing the largest improvement. Full results in \Cref{app:exp_nonlinear}.

\subsection{Practical Recommendations}\label{sec:guidance}

The diagnostic $\hat{\eta} = \mathrm{std}(h)/\mathrm{mean}(h)$ guides method selection: $\hat{\eta} > 1$ indicates LWCP improves conditional coverage; $\hat{\eta} < 0.5$ indicates vanilla CP suffices. The default weight $w(h) = (1{+}h)^{-1/2}$ is provably safe (\Cref{prop:safe_default_main}) and near-optimal across DGPs (\Cref{app:weight_comparison}). Data-driven weight selection via a three-way split preserves coverage exactly, since the selected weight is $\cD_1$-measurable (\Cref{app:weight_selection}). Features should be standardized before computing leverage (\Cref{alg:lwcp}); points with $h(x)$ above the 99th calibration percentile should be flagged. Extended guidance on high-dimensional settings, covariate shift (\Cref{prop:cov_shift}), collinearity (\Cref{prop:collinearity}), and extrapolation guardrails is in \Cref{app:guidance}. Additional experiments---ridge leverage at $p/n_1$ up to $10$ (\Cref{app:high_dim_exp}), runtime benchmarking (\Cref{app:runtime_rigorous}), LWCP+ vs.\ Studentized CP (\Cref{app:lwcp_plus_vs_stud}), and preprocessing ablation (\Cref{app:preprocessing})---are in the supplement.

\section{Conclusion}\label{sec:conclusion}

LWCP provides adaptive prediction intervals that are distribution-free, finite-sample valid for any weight function (\Cref{thm:coverage}), and achieve asymptotically optimal conditional coverage at negligible width cost (\Cref{thm:efficiency,thm:width_finite}). The method adds no hyperparameters and scales identically to vanilla CP. We established that vanilla CP suffers a persistent $\Theta(1)$ conditional coverage gap (\Cref{thm:persistent_gap}), while LWCP eliminates this gap under the scale-family model. The default weight $(1{+}h)^{-1/2}$ is provably safe---$O(1/n_1)$ downside with $\Theta(1)$ upside (\Cref{prop:safe_default_main})---and LWCP's advantage grows with $\gamma = p/n_1$ (\Cref{prop:high_dim_main}). LWCP+ corrects the training-test variance mismatch limiting studentized CP (\Cref{cor:mismatch}); under misspecification, the gap reduction is controlled by $\eta_{\rm lev}$ (\Cref{thm:misspec}).

\bibliography{references}

\newpage
\onecolumn

\title{Leverage-Weighted Conformal Prediction\\(Supplementary Material)}
\maketitle

\appendix

\let\origsubsection\subsection
\renewcommand{\subsection}{\FloatBarrier\origsubsection}
\setlength{\tabcolsep}{10pt}

\noindent The supplement is organized as follows. \Cref{app:proofs} contains all deferred proofs and extended theoretical results (coverage, efficiency, width parity, persistent gap, misspecification, adversarial analysis, and connections to related methods). \Cref{app:extensions} develops extensions to ridge regression, kernel methods, neural networks, and covariate shift. \Cref{app:experiments} presents additional experimental results including Gaussian recovery, heteroscedasticity sensitivity, scaling behavior, real-data benchmarks, and practical guidance. \Cref{app:extended_theory} provides deeper theoretical analysis on approximate scale families, high-dimensional asymptotics, GLMs, gradient leverage, leverage stability, and fairness. \Cref{app:supp_experiments} contains supplementary runtime benchmarks, high-dimensional stress tests, weight comparisons, and preprocessing ablations.

\section{Proofs and Extended Theory}\label{app:proofs}

\subsection{Proofs of Main-Text Results}\label{app:main_proofs}

The following proofs are deferred from the main text for space. They are presented in theorem-number order. Supporting lemmas are in \Cref{app:extended_tools}; the quantitative mismatch analysis is in \Cref{app:quant_mismatch}.

\begin{proof}[\textbf{Proof of \Cref{thm:coverage_plus}}]
Both $\hat\sigma(\cdot)$ and $w(h(\cdot))$ are functions of $\cD_1$ alone. Conditional on $\cD_1$, the composite scores $s_i^+ = |Y_i - \fhat(X_i)|/(\hat\sigma(X_i) \cdot w(h(X_i)))$ for $i \in \cD_2 \cup \{n+1\}$ are functions of $(X_i, Y_i)$ applied to exchangeable random variables. The argument of \Cref{thm:coverage} applies verbatim. \qed
\end{proof}

\begin{proof}[\textbf{Proof of \Cref{thm:approx}}]
\textbf{Part 1.} Approximate leverage scores $\tilde{h}$ are computed entirely from $\cD_1$, so $\tilde{w}(x) = w(\tilde{h}(x))$ is $\cD_1$-measurable. The exchangeability argument of \Cref{thm:coverage} applies without modification, yielding exact coverage.

\textbf{Part 2.} The width at $x$ changes from $2\qhat_w/w(h(x))$ to $2\qhat_{\tilde{w}}/w(\tilde{h}(x))$. By Lipschitz continuity, $|w(\tilde{h}) - w(h)| \leq L|\tilde{h} - h| \leq L\epsilon h$, giving $w(h)/w(\tilde{h}) = 1 + O(L\epsilon\|h\|_\infty/w_{\min})$. For the quantile perturbation, assuming the score density $f_S$ satisfies $f_S(Q_{1-\alpha}(S)) \geq c > 0$ (a standard regularity condition), the quantile perturbation lemma \citep[Lemma~21.2]{van2000asymptotic} gives $|Q_{1-\alpha}(\tilde{S}) - Q_{1-\alpha}(S)| \leq \|F_S - F_{\tilde{S}}\|_\infty / c$. \qed
\end{proof}

\begin{proof}[\textbf{Proof of \Cref{prop:cond_gap}}]
For LWCP with $w^*$, the variance-stabilized scores $R_i^{w^*} \approx \sigma|\eta_i|$ are approximately iid under \Cref{asm:linear}. The conformal quantile $\qhat_{w^*}$ estimates $Q_{1-\alpha}(\sigma|\eta|)$ with standard quantile estimation error $O(1/\sqrt{n_2})$, while the approximation $R_i^{w^*} \approx \sigma|\eta_i|$ incurs $O(\|\hat{\beta} - \beta^*\|) = O(\sqrt{p/n_1})$ error from \Cref{asm:regularity}(a). The conditional coverage at leverage $h$ converges to $F_{|\eta|}(F_{|\eta|}^{-1}(1-\alpha)) = 1-\alpha$ uniformly.

For vanilla CP, scores $R_i^1 = |\varepsilon_i + o_p(1)|$ have $h$-dependent distributions. The conformal quantile estimates $q_{\text{mix}} = Q_{1-\alpha}(\sigma\sqrt{g(H)}|\eta|)$, while the conditional $(1-\alpha)$-quantile at leverage $h$ is $\sigma\sqrt{g(h)} Q_{1-\alpha}(|\eta|)$. Since $\E_H[F_H(q_\text{mix})] = 1-\alpha$ but $F_h$ varies with $h$ when $g$ is non-constant, the gap between extreme leverages is (see \Cref{thm:persistent_gap}):
\[
\Delta(h_{\min}, h_{\max}) \geq f_{|\eta|}(q_\alpha) \cdot q_\alpha \cdot \frac{\sqrt{g(h_{\max})} - \sqrt{g(h_{\min})}}{\sqrt{g(h_{\max})}} = \Theta(h_{\max} - h_{\min}),
\]
which does not vanish with $n$. \qed
\end{proof}

\begin{proof}[\textbf{Proof of \Cref{thm:gaussian_recovery}}]
For $i \in \cD_2$, write $Y_i - \fhat(X_i) = \varepsilon_i - X_i^\top(\hat{\beta} - \beta^*)$. Since $\hat{\beta} - \beta^* = (\bX_1^\top\bX_1)^{-1}\bX_1^\top\beps_1 \sim \mathcal{N}(0, \sigma^2(\bX_1^\top\bX_1)^{-1})$ is independent of $\varepsilon_i$ for $i \in \cD_2$:
\[
Y_i - \fhat(X_i) \mid X_i, \cD_1 \sim \mathcal{N}(0, \sigma^2(1 + h_i)).
\]
With $w(h) = (1+h)^{-1/2}$, the scores $R_i^w = |Y_i - \fhat(X_i)|/\sqrt{1+h_i} \sim \sigma|Z_i|$ where $Z_i \stackrel{iid}{\sim} \mathcal{N}(0,1)$. The $(1-\alpha)$-quantile of $\sigma|Z|$ is $\sigma z_{1-\alpha/2}$, so $\qhat_w \to \sigma z_{1-\alpha/2}$. The LWCP half-width at $x$ is $\sigma z_{1-\alpha/2}\sqrt{1+h(x)}$, matching the classical prediction interval form. Finite-sample differences are: (i)~$z_{1-\alpha/2}$ vs.\ $t_{n_1-p,1-\alpha/2}$; (ii)~sample splitting. Crucially, \Cref{thm:coverage} guarantees finite-sample coverage \emph{without} requiring the Gaussian assumption. \qed
\end{proof}

\begin{proof}[\textbf{Proof of \Cref{prop:variance_decomp}}]
Write $Y_{\mathrm{new}} - \fhat(x) = \varepsilon_{\mathrm{new}} - x^\top(\hat\beta - \beta^*)$. Since $\varepsilon_{\mathrm{new}} \perp\!\!\!\perp (\hat\beta - \beta^*)$ conditional on $x$ and $\cD_1$:
\begin{align*}
\Var(Y_{\mathrm{new}} - \fhat(x) \mid x, \cD_1) &= \Var(\varepsilon_{\mathrm{new}} \mid x) + x^\top \Cov(\hat\beta \mid \bX_1) x.
\end{align*}
Under homoscedastic errors, $\Cov(\hat\beta \mid \bX_1) = \sigma^2(\bX_1^\top\bX_1)^{-1}$, so $x^\top\Cov(\hat\beta)x = \sigma^2 h(x)$. \qed
\end{proof}

\begin{proof}[\textbf{Proof of \Cref{prop:mismatch}}]
\textbf{Part 1.} Write $\hat\beta = (\bX_1^\top\bX_1)^{-1}\bX_1^\top\bY_1 = \beta^* + (\bX_1^\top\bX_1)^{-1}\bX_1^\top\beps_1$. For $i \in \cD_1$:
\[
Y_i - \fhat(X_i) = ((I - \bH)\beps_1)_i,
\]
where $\bH = \bX_1(\bX_1^\top\bX_1)^{-1}\bX_1^\top$. Since $\Var(\beps_1) = \sigma^2 I$ and $(I-\bH)$ is idempotent:
$\Var(e_i^{\mathrm{train}}) = \sigma^2(1 - h_i)$.

\textbf{Part 2.} For $i \in \cD_2$ (independent of $\cD_1$), $\varepsilon_{\mathrm{new}} \perp\!\!\!\perp \beps_1$, so:
$\Var(Y_{\mathrm{new}} - \fhat(x) \mid x, \bX_1) = \sigma^2(1 + h(x))$. \qed
\end{proof}

\begin{proof}[\textbf{Proof of \Cref{prop:optimality}}]
Under the scale-family \Cref{asm:linear}, $|\varepsilon_i| \mid h_i \sim \sigma\sqrt{g(h_i)} \cdot |\eta|$. For any weight $w$ to achieve asymptotic conditional coverage at level $1-\alpha$ for $F_h$-a.e.\ $h$, the LWCP half-width $\qhat_w/w(h)$ must equal the conditional quantile $\sigma\sqrt{g(h)} \cdot Q_{1-\alpha}(|\eta|)$ asymptotically. Since the $h$-dependence is only through $1/w(h)$ vs.\ $\sqrt{g(h)}$, this forces $w(h) \propto 1/\sqrt{g(h)} = w^*(h)$. \qed
\end{proof}

\begin{proof}[\textbf{Proof of \Cref{thm:efficiency}}]
We prove the two parts in full.

\textbf{Part 1 (Conditional coverage).} Fix $h_0 \in \mathrm{supp}(F_h)$. We show $\Prob(Y_{n+1} \in \hat{\cC}_n^{w^*}(X_{n+1}) \mid h(X_{n+1}) = h_0) \to 1 - \alpha$.

\emph{Step 1: Score approximation.} For $i \in \cD_2$, write
\[
Y_i - \fhat(X_i) = \varepsilon_i - X_i^\top(\hat\beta - \beta^*) = \sigma\sqrt{g(h_i)}\,\eta_i - X_i^\top(\hat\beta - \beta^*).
\]
Under \Cref{asm:regularity}(a), $\|\hat\beta - \beta^*\| = O_p(1/\sqrt{n_1})$, so $|X_i^\top(\hat\beta - \beta^*)| = O_p(\|X_i\|/\sqrt{n_1})$. Since $\|X_i\|^2/p \to_p 1$ and $p$ is fixed, $|X_i^\top(\hat\beta - \beta^*)| = O_p(1/\sqrt{n_1})$ uniformly over $i$. The variance-stabilized scores satisfy:
\[
R_i^{w^*} = \frac{|Y_i - \fhat(X_i)|}{\sqrt{g(h_i)}} = \sigma|\eta_i| + O_p(1/\sqrt{n_1}) \quad \text{uniformly over } i \in \cD_2.
\]

\emph{Step 2: Quantile convergence.} Since $R_i^{w^*} = \sigma|\eta_i| + o_p(1)$ and the $\eta_i$ are iid with continuous distribution $F_\eta$ (\Cref{asm:regularity}(c)), the empirical CDF of $\{R_i^{w^*}\}_{i \in \cD_2}$ converges uniformly to $F_{\sigma|\eta|}$ by the Glivenko--Cantelli theorem \citep[Chapter~14]{shorack2009empirical}. The conformal quantile $\qhat_{w^*}$ is the $\lceil(1-\alpha)(n_2+1)\rceil/(n_2+1)$-th quantile of this empirical distribution, so:
\[
\qhat_{w^*} \to \sigma \cdot Q_{1-\alpha}(|\eta|) =: \sigma q_\alpha \quad \text{in probability.}
\]

\emph{Step 3: Uniform conditional coverage.} Conditional on $h(X_{n+1}) = h_0$, coverage is:
\begin{align*}
\Prob(Y_{n+1} \in \hat\cC^{w^*} \mid h = h_0)
&= \Prob\!\left(\frac{|Y_{n+1} - \fhat(X_{n+1})|}{\sqrt{g(h_0)}} \leq \qhat_{w^*} \;\middle|\; h = h_0\right) \\
&= \Prob\!\left(\sigma|\eta_{n+1}| + O_p(1/\sqrt{n_1}) \leq \sigma q_\alpha + o_p(1)\right) \\
&\to F_{|\eta|}(q_\alpha) = 1 - \alpha.
\end{align*}
The convergence is uniform over $h_0$ because: (i)~the $O_p(1/\sqrt{n_1})$ score approximation error is uniform in $h$; (ii)~$\qhat_{w^*}$ does not depend on $h_0$; and (iii)~$F_{|\eta|}$ is continuous (\Cref{asm:regularity}(c)).

\textbf{Part 2 (Width parity).} The LWCP expected half-width is:
\[
\E_H\!\left[\frac{\qhat_{w^*}}{w^*(H)}\right] = \qhat_{w^*} \cdot \E[\sqrt{g(H)}] \to \sigma q_\alpha \cdot \mu_G, \quad \mu_G := \E[\sqrt{g(H)}].
\]
The vanilla CP expected half-width is $\qhat_1 \to q_{\rm mix} := Q_{1-\alpha}(\sigma\sqrt{g(H)}|\eta|)$. The width ratio is $\sigma q_\alpha \mu_G / q_{\rm mix}$. By the quantile perturbation analysis of \Cref{thm:width_finite}, $q_{\rm mix} = \sigma q_\alpha \mu_G (1 + C_\alpha \rho^2 + O(\rho^3))$ where $\rho^2 = \mathrm{CV}^2(\sqrt{g(H)})$. As $p/n \to 0$, \Cref{asm:regularity}(b) gives $\rho^2 \to 0$ (the leverage distribution concentrates), so the width ratio $\to 1$. \qed
\end{proof}

\subsection{Extended Theory}\label{app:extended_tools}

This subsection develops extended theoretical tools used throughout the supplement. We begin with a standard quantile perturbation lemma.

\begin{lemma}[Quantile sensitivity]\label{lem:quantile_sensitivity}
Let $Z > 0$ be a random variable with density $f_Z$ satisfying $f_Z(q) \geq c_0 > 0$ at $q := Q_{1-\alpha}(Z)$. For a multiplicative perturbation $Z' = Z(1+\delta)$ with $\E[\delta|Z] = 0$ and $\Var(\delta|Z) \leq \sigma_\delta^2$:
\[
|Q_{1-\alpha}(Z') - q| \leq \frac{\sigma_\delta^2}{2} \cdot \left(q + \left|\frac{q^2 f'_Z(q)}{f_Z(q)}\right|\right) + O(\sigma_\delta^3).
\]
Equivalently, the relative perturbation satisfies $|Q_{1-\alpha}(Z') - q|/q = O(\sigma_\delta^2)$, i.e., second-order in the perturbation size.
\end{lemma}
\begin{proof}
Write $F_{Z'}(t) = \E_\delta[F_Z(t/(1+\delta))]$. Taylor-expanding $F_Z(t/(1+\delta))$ around $\delta = 0$:
\[
F_Z\!\left(\frac{t}{1+\delta}\right) = F_Z(t) - t f_Z(t) \delta + \tfrac{1}{2}(t^2 f'_Z(t) + tf_Z(t))\delta^2 + O(\delta^3).
\]
Taking expectations with $\E[\delta] = 0$ and $\E[\delta^2] = \sigma_\delta^2$:
\[
F_{Z'}(t) = F_Z(t) + \tfrac{\sigma_\delta^2}{2}(t^2 f'_Z(t) + tf_Z(t)) + O(\sigma_\delta^3).
\]
Setting $F_{Z'}(q') = 1-\alpha = F_Z(q)$ and solving for $q' - q$ via the implicit function theorem gives the stated bound. \qed
\end{proof}

\subsection{Non-Asymptotic Width Parity}\label{app:width_parity}

We give the full proof of \Cref{thm:width_finite} (stated in the main text).

\begin{proof}[\textbf{Proof of \Cref{thm:width_finite}}]
\textbf{Step 1: Population quantities.} Write $G = \sqrt{g(H)}$ with $\mu_G = \E[G]$ and $G = \mu_G(1+\delta)$ where $\E[\delta] = 0$, $\Var(\delta) = \rho^2$. The LWCP expected half-width is $\sigma q_\alpha \mu_G$, where $q_\alpha = F_{|\eta|}^{-1}(1-\alpha)$. The vanilla CP expected half-width is $q_{\rm mix} = Q_{1-\alpha}(\sigma G |\eta|)$.

\textbf{Step 2: Quantile perturbation.} We analyze $q_{\rm mix} = \sigma\mu_G Q_{1-\alpha}((1+\delta)|\eta|)$. Taylor-expanding:
\[
F_{|\eta|}\!\left(\frac{t}{1+\delta}\right) = F_{|\eta|}(t) - tf_{|\eta|}(t)\frac{\delta}{1+\delta} + \frac{1}{2}\bigl(t^2 f'_{|\eta|}(t) + tf_{|\eta|}(t)\bigr)\!\left(\frac{\delta}{1+\delta}\right)^{\!2} + O(\delta^3).
\]
Taking expectations:
$F_\delta(t) = F_{|\eta|}(t) + \frac{\rho^2}{2}\bigl(t^2 f'_{|\eta|}(t) + tf_{|\eta|}(t)\bigr) + O(\rho^3)$.

Inverting at $t = q_\alpha$:
\begin{align*}
Q_{1-\alpha}((1{+}\delta)|\eta|) &= q_\alpha\Bigl(1 - \frac{\rho^2}{2}\Bigl(1 + \frac{q_\alpha f'_{|\eta|}(q_\alpha)}{f_{|\eta|}(q_\alpha)}\Bigr)\Bigr) + O(\rho^3).
\end{align*}

\textbf{Step 3: Width ratio.}
$\frac{\sigma q_\alpha \mu_G}{\sigma \mu_G \cdot Q_{1-\alpha}((1+\delta)|\eta|)} = 1 + C_\alpha \rho^2 + O(\rho^3)$,
where $C_\alpha = \tfrac{1}{2}(1 + q_\alpha f'_{|\eta|}(q_\alpha)/f_{|\eta|}(q_\alpha))$.

\textbf{Step 4: Gaussian specialization.} For $\eta \sim \mathcal{N}(0,1)$, $f'_{|\eta|}(t) = -t f_{|\eta|}(t)$, so $C_\alpha = (1 - q_\alpha^2)/2$. For $\alpha = 0.1$: $C_\alpha \approx -0.85 < 0$. LWCP is slightly narrower because scale mixtures have heavier tails.

\textbf{Step 5: Finite-sample correction.} DKW gives $|\hat{q} - q| = O_p(1/\sqrt{n_2})$. \qed
\end{proof}

\begin{remark}[Concrete bounds]
For $g(h) = 1+h$ at $n_1 = 300$, $p = 30$: $\rho^2 \approx 1.7 \times 10^{-4}$, confirming observed ratios of $0.999$--$1.000$.
\end{remark}

\subsection{Persistent $\Theta(1)$ Gap for Vanilla CP}\label{app:theta_one}

\begin{proof}[\textbf{Proof of \Cref{thm:persistent_gap}}]
\textbf{Part (i).} As $n \to \infty$, $\qhat_1 \to q_{\rm mix} := Q_{1-\alpha}(\sigma\sqrt{g(H)} |\eta|)$. Conditional on $h(X_{n+1}) = h$, asymptotic conditional coverage is $F_{|\eta|}(q_{\rm mix}/(\sigma\sqrt{g(h)}))$. Since $g(h_1) < g(h_2)$ and $F_{|\eta|}$ is strictly increasing, the gap is positive and does not depend on $n$.

\textbf{Part (ii).} By the mean value theorem and the inequality $\sqrt{r} - 1 \geq (r-1)/(2\sqrt{r})$:
$\Delta \geq f_{|\eta|}(q_\alpha) \cdot q_\alpha \cdot (g(h_2) - g(h_1))/(2g(h_2)) \cdot (1 + o(1))$.

\textbf{Part (iii).} For $g(h) = 1+h$ and Gaussian $\eta$: $f_{|\eta|}(z_{0.95}) \approx 0.34$. \qed
\end{proof}

\subsection{Robustness Under Misspecification}\label{app:misspecification}

\begin{definition}[Leverage-feature decomposition]\label{def:lev_feat_decomp}
For general conditional variance $\Var(\varepsilon | X = x) = \sigma^2(x)$, define:
(i)~$\sigma_h^2(h) := \E[\sigma^2(X) \mid h(X) = h]$;
(ii)~$\psi^2(x) := \sigma^2(x) / \sigma_h^2(h(x))$, satisfying $\E[\psi^2(X) \mid h(X)] = 1$;
(iii)~$\eta_{\rm lev} := \Var(\sigma_h^2(H)) / \Var(\sigma^2(X)) \in [0, 1]$.
\end{definition}

The full version of \Cref{thm:misspec} additionally provides a conditional gap decomposition: defining $\psi^2(x) := \sigma^2(x)/\sigma_h^2(h(x))$,
\begin{equation}\label{eq:misspec_gap}
\sup_h |\Prob(Y {\in} \hat{\cC}^w \mid h(X){=}h) - (1{-}\alpha)| \leq O(1/\sqrt{n_2}) + O(\sqrt{p/n_1}) + C_\alpha \cdot \sup_h \mathrm{CV}(\psi \mid h).
\end{equation}
The third term vanishes when heteroscedasticity is purely leverage-driven ($\psi \equiv 1$).

\begin{proof}[\textbf{Proof of \Cref{thm:misspec}}]
\textbf{Part (i)} follows immediately from \Cref{thm:coverage}, since marginal validity requires only exchangeability and $\cD_1$-measurability of $w(h(\cdot))$, with no assumptions on $\sigma^2(x)$.

\textbf{Part (ii).} Decompose the conditional variance as $\sigma^2(x) = \sigma_h^2(h(x)) \cdot \psi^2(x)$ (\Cref{def:lev_feat_decomp}), where $\E[\psi^2(X) \mid h(X) = h] = 1$ for all $h$. Under LWCP with the leverage-optimal weight $w(h) = 1/\sigma_h(h)$, the variance-stabilized scores satisfy:
\[
R_i^w = \frac{|Y_i - \fhat(X_i)|}{\sigma_h(h_i)} \approx \frac{|\varepsilon_i|}{\sigma_h(h_i)} = \psi(X_i) \cdot |\eta_i|.
\]
Conditional on $h(X_{n+1}) = h$, the scores are \emph{not} identically distributed due to $\psi(X_i)$, but the residual heterogeneity is only through $\psi$---the leverage-dependent component has been removed. The conditional gap at $h$ is controlled by the scale-mixture effect of $\psi$:
\[
\mathrm{gap}^w(h) \leq C_\alpha \cdot \mathrm{CV}(\psi(X) \mid h(X) = h) + O(1/\sqrt{n_2}),
\]
by applying \Cref{lem:quantile_sensitivity} to $Z = \sigma|\eta|$ and $\delta = \psi - 1$.

For vanilla CP ($w \equiv 1$), the score heterogeneity is $\mathrm{CV}(\sigma(X))$, which includes both leverage-dependent and orthogonal components. The gap ratio is therefore:
\[
\frac{\mathrm{gap}_{\rm LWCP}}{\mathrm{gap}_{\rm Vanilla}} \leq \frac{\sup_h \mathrm{CV}(\psi \mid h)}{\mathrm{CV}(\sigma(X))}.
\]
By the law of total variance:
\begin{align*}
\Var(\sigma^2(X)) &= \underbrace{\Var(\sigma_h^2(H))}_{\text{leverage-explained}} + \underbrace{\E[\Var(\sigma^2(X) \mid H)]}_{\text{orthogonal}}.
\end{align*}
Hence $\Var(\sigma_h^2(H)) = \eta_{\rm lev} \cdot \Var(\sigma^2(X))$, and the orthogonal component satisfies $\E[\Var(\sigma^2(X) \mid H)] = (1 - \eta_{\rm lev}) \cdot \Var(\sigma^2(X))$. Since $\mathrm{CV}(\psi \mid h)$ captures only the orthogonal component:
\[
\frac{\mathrm{gap}_{\rm LWCP}}{\mathrm{gap}_{\rm Vanilla}} \leq \sqrt{1 - \eta_{\rm lev}} + O(1/\sqrt{n}).
\]
When $\eta_{\rm lev} \approx 1$, the gap is nearly eliminated; when $\eta_{\rm lev} \approx 0$, LWCP reduces to vanilla CP with no degradation. \qed
\end{proof}

\subsection{Adversarial Worst-Case Analysis}\label{app:adversarial}

\begin{proof}[\textbf{Proof of \Cref{prop:adv}}]
Part (i) is \Cref{thm:coverage}. Part (ii): under $g_{\rm adv}$, the optimal weight is $w^*(h) = \sqrt{(1+h)/c}$, which \emph{increases} with $h$---opposite to LWCP's $(1+h)^{-1/2}$. The total prediction variance $\sigma^2 g(h)(1+h) = \sigma^2 c$ is constant, so vanilla CP's constant-width intervals are already optimal. Part (iii): degradation is $\mathrm{CV}^2(w^*/w) = \mathrm{CV}^2(1+H) = O(p/n_1^2)$ by concentration of leverage moments under $p/n_1 \to 0$. \qed
\end{proof}

\subsection{Quantitative Training-Test Mismatch}\label{app:quant_mismatch}

\begin{theorem}[LWCP+ vs.\ Studentized: quantitative separation]\label{thm:quant_mismatch}
Under the homoscedastic linear model with $\hat{\sigma}(x) = \sigma\sqrt{1-h(x)} \cdot (1 + \xi(x))$:
\begin{enumerate}[label=(\roman*)]
\item \textbf{Studentized CP scores:} $s_i^{\rm Stud} \approx \sqrt{(1+h_i)/(1-h_i)} \cdot |\eta_i| / (1+\xi_i)$, with $d\ln\phi_S/dh = 1/(1-h^2)$.
\item \textbf{LWCP+ scores:} $s_i^+ \approx |\eta_i| / (\sqrt{1-h_i}(1+\xi_i))$, with $d\ln\phi_+/dh = 1/(2(1-h))$.
\item \textbf{Improvement factor:}
\begin{equation}\label{eq:cv_ratio}
\frac{\mathrm{CV}^2(\phi_S(H))}{\mathrm{CV}^2(\phi_+(H))} = \frac{4}{(1+\bar{h})^2} + O(\Var(H)), \quad \bar{h} = p/n_1.
\end{equation}
For $p/n_1 \to 0$: ratio $\to 4$, meaning LWCP+ achieves $\approx 2\times$ smaller conditional gap.
\end{enumerate}
\end{theorem}

\begin{proof}
The $h$-dependent factors are $\phi_S(h) = ((1+h)/(1-h))^{1/2}$ and $\phi_+(h) = (1-h)^{-1/2}$. Computing log-derivatives:
$d\ln\phi_S/dh = 1/(1-h^2)$ and $d\ln\phi_+/dh = 1/(2(1-h))$.
The ratio $(d\ln\phi_S/dh)^2/(d\ln\phi_+/dh)^2 = 4/(1+h)^2$, which at $h = p/n_1 \to 0$ gives $4$. The $2\times$ gap reduction follows since gap $\propto$ CV $= \sqrt{\mathrm{CV}^2}$. \qed
\end{proof}

\begin{remark}[Why the factor of 4]
Studentized CP's score is affected by leverage through \emph{both} numerator $(1+h)$ and denominator $(1-h)$. LWCP+'s correction $\sqrt{1+h}$ cancels the numerator, halving the log-sensitivity. Since CV$^2$ scales as log-sensitivity squared, the ratio is $2^2 = 4$.
\end{remark}

\subsection{Connections to Related Methods}\label{app:connections}

\paragraph{Relationship to Vovk's RRCM.}
The RRCM \citep{vovk2005algorithmic, burnaev2014efficiency} uses LOO residuals $|e_i|/(1-h_i)$, which \emph{upweights} high-leverage points, correcting the training-side attenuation (\Cref{prop:mismatch} Part~1). LWCP with $w(h) = (1+h)^{-1/2}$ \emph{downweights} high-leverage points, correcting the test-side amplification (Part~2). When $h_i \ll 1$, both corrections are negligible and the methods coincide.

\paragraph{Relationship to Localized CP.}
Localized CP \citep{guan2023localized} uses a kernel $K(X_i, x_{\rm new})$ for point-specific weights. LWCP can be viewed as a special case where the ``kernel'' groups points by leverage. The advantage is no bandwidth/kernel hyperparameters and $O(np^2)$ computation (vs.\ $O(n^2 p)$). The disadvantage is adaptation only through $h(x)$.

LWCP and localized CP are complementary: one could combine leverage-weighted scores with locality kernels.

\paragraph{Relationship to \citet{gibbs2023conformal}.}
Their method targets conditional coverage through optimization over a function class; LWCP achieves it through a fixed geometric transformation. When $\eta_{\rm lev} \approx 1$, LWCP achieves comparable coverage without calibration overhead. When $\eta_{\rm lev} \ll 1$, optimization-based approaches are preferable.

\paragraph{Cross-conformal extension.}
LWCP extends naturally to the cross-conformal setting \citep{vovk2015cross, vovk2018nonparametric}: for each fold $k$, compute the SVD of $\bX_{-k}$ and use the resulting leverage scores. The cross-conformal p-value preserves finite-sample validity by exchangeability. Total cost: $O(Knp^2)$.

\paragraph{Non-linear predictors and feature-space leverage.}
LWCP's coverage guarantee holds for \emph{any} predictor. For specific non-linear predictors:
\textbf{Kernel methods:} $h^{\rm ker}(x) = k_x^\top(\bK + \lambda\Id)^{-1}k_x$.
\textbf{Neural networks:} Last-layer leverage $h^{\rm NN}(x) = \phi(x)^\top(\Phi^\top\Phi)^{-1}\phi(x)$.

\section{Extensions}\label{app:extensions}

This section develops extensions of LWCP beyond the standard OLS setting, addressing three scenarios that arise frequently in practice: ill-conditioned or overparameterized designs (\Cref{app:ridge_ext}), non-linear predictors via kernel and neural network leverage (\Cref{app:kernel_nn}), and distribution shift between training and test populations (\Cref{app:cov_shift}).

\subsection{Ridge Leverage Scores}\label{app:ridge_ext}

When $\bX_1$ is ill-conditioned (condition number $\kappa = \sigma_1/\sigma_p \gg 1$) or overparameterized ($p > n_1$), the ordinary hat matrix $\bH = \bX_1(\bX_1^\top\bX_1)^{-1}\bX_1^\top$ is either numerically unstable or undefined. Ridge leverage scores provide a well-conditioned alternative:
\[
h_i^\lambda = X_i^\top(\bX_1^\top\bX_1 + \lambda\Id)^{-1}X_i, \quad \lambda > 0.
\]
Unlike ordinary leverage, $h_i^\lambda \in [0, \|X_i\|^2/\lambda]$ is bounded for any $p/n_1$ ratio. The ridge parameter $\lambda$ controls the bias--variance trade-off: small $\lambda$ recovers ordinary leverage (high variance, low bias); large $\lambda$ compresses the leverage spectrum toward uniformity.

\begin{proof}[\textbf{Proof of \Cref{thm:ridge_lwcp}}]
\textbf{Part 1 (Coverage).} The ridge leverage $h_i^\lambda = X_i^\top(\bX_1^\top\bX_1 + \lambda\Id)^{-1}X_i$ is computed entirely from $\bX_1$ and the fixed hyperparameter $\lambda$, so $w(h_i^\lambda)$ is $\cD_1$-measurable for all $i \in \cD_2 \cup \{n{+}1\}$. The exchangeability argument of \Cref{thm:coverage} applies without modification, yielding exact marginal coverage $\geq 1 - \alpha$.

\textbf{Part 2 (Variance decomposition).} Under the ridge estimator $\hat\beta_\lambda = (\bX_1^\top\bX_1 + \lambda\Id)^{-1}\bX_1^\top\bY_1$, the prediction error decomposes as:
\begin{align*}
Y_{\rm new} - \fhat_\lambda(x) &= \varepsilon_{\rm new} - x^\top(\hat\beta_\lambda - \beta^*) \\
&= \varepsilon_{\rm new} \underbrace{- x^\top(\bX_1^\top\bX_1 + \lambda\Id)^{-1}\bX_1^\top\beps_1}_{\text{variance term}} \underbrace{+ \lambda x^\top(\bX_1^\top\bX_1 + \lambda\Id)^{-1}\beta^*}_{\text{bias term } b_\lambda(x)}.
\end{align*}
The variance contribution from estimation error is $\sigma^2 h_\lambda^{\rm eff}(x)$, where:
\[
h_\lambda^{\rm eff}(x) = x^\top(\bX_1^\top\bX_1 + \lambda\Id)^{-1}\bX_1^\top\bX_1(\bX_1^\top\bX_1 + \lambda\Id)^{-1}x.
\]
For $\lambda \to 0$, $h_\lambda^{\rm eff}(x) \to h(x)$; for $\lambda \to \infty$, $h_\lambda^{\rm eff}(x) \to 0$. The total prediction variance is $\sigma^2(1 + h_\lambda^{\rm eff}(x)) + b_\lambda(x)^2$.

\textbf{Part 3 (Stability bound).} The operator norm satisfies $\|(\bX_1^\top\bX_1 + \lambda\Id)^{-1}\| \leq 1/\lambda$, so $h_i^\lambda \leq \|X_i\|^2/\lambda$ for all $i$. This guarantees bounded leverage scores even in the overparameterized regime ($p > n_1$), preventing extreme weight values that could inflate the conformal quantile. The weight function $w(h) = (1{+}h)^{-1/2}$ applied to ridge leverage satisfies $w(h_i^\lambda) \in [\sqrt{\lambda/(\lambda + \|X_i\|^2)}, 1]$. \qed
\end{proof}

\subsection{Kernel and Neural Network Leverage}\label{app:kernel_nn}

LWCP's coverage guarantee (\Cref{thm:coverage}) holds for any predictor $\fhat$, but efficiency depends on how well the leverage scores capture prediction uncertainty. For non-linear predictors, \emph{feature-space leverage} provides a principled generalization.

\begin{proposition}[Feature-space leverage]\label{prop:feature_leverage}
Let $\phi : \cX \to \R^d$ be a feature map and define the feature-space leverage:
\[
h^\phi(x) = \phi(x)^\top\!\left(\sum_{i \in \cD_1} \phi(X_i)\phi(X_i)^\top + \lambda\Id\right)^{\!-1}\!\!\phi(x).
\]
\begin{enumerate}[label=(\roman*)]
\item \textbf{Kernel methods.} For a kernel $k(x,x') = \langle\phi(x), \phi(x')\rangle$, the kernel leverage is $h_i^{\rm ker} = (\bK(\bK + \lambda\Id)^{-1})_{ii}$, where $\bK_{ij} = k(X_i, X_j)$. By the Woodbury identity, this equals $h^\phi(X_i)$ in the feature space.
\item \textbf{Neural networks.} For a network $f_\theta(x) = a^\top\phi_\theta(x)$ with last-layer features $\phi_\theta$, the last-layer leverage is $h_i^{\rm NN} = \phi_\theta(X_i)^\top(\Phi_\theta^\top\Phi_\theta)^{-1}\phi_\theta(X_i)$, where $\Phi_\theta \in \R^{n_1 \times d}$ is the feature matrix.
\item \textbf{Coverage.} LWCP with $w(h^\phi)$ preserves exact marginal coverage for any $\phi$ that is $\cD_1$-measurable, since the feature map is determined by the training data.
\end{enumerate}
\end{proposition}

\begin{proof}
(i)~By the Woodbury identity, $\phi(x)^\top(\Phi^\top\Phi + \lambda\Id)^{-1}\phi(x) = k_x^\top(\bK + \lambda\Id)^{-1}k_x$, where $k_x = [k(X_1, x), \ldots, k(X_{n_1}, x)]^\top$.
(ii)~Direct substitution with $\phi = \phi_\theta$ and $d$ equal to the last-layer width.
(iii)~\Cref{thm:coverage}, since $h^\phi(\cdot)$ is a function of $\cD_1$ alone. \qed
\end{proof}

\begin{remark}[NTK regime]\label{rem:ntk}
In the NTK/lazy training regime, $\phi_\theta(x)$ is approximately fixed after initialization, the network behaves as a linear model in the last layer, and the efficiency theory (\Cref{thm:efficiency}) applies directly. Outside the lazy regime, coverage still holds (\Cref{thm:coverage}) but efficiency guarantees degrade: the mismatch coefficient $\eta_{\rm lev}$ (\Cref{thm:misspec}) quantifies the fraction of prediction variance captured by the feature-space leverage. Empirically, last-layer leverage provides substantial improvements even for moderately trained networks (\Cref{app:exp_nonlinear}).
\end{remark}

\begin{remark}[Computational considerations]
Kernel leverage requires forming and inverting $\bK \in \R^{n_1 \times n_1}$, costing $O(n_1^3)$. For large $n_1$, the Nystr\"{o}m approximation \citep{alaoui2015fast} or random features \citep{rahimi2007random} provide approximate kernel leverage in $O(n_1 k^2)$ time, where $k \ll n_1$ is the number of landmark points or random features. Last-layer leverage for neural networks requires one forward pass to extract $\Phi_\theta \in \R^{n_1 \times d}$ and an $O(n_1 d^2)$ SVD, where $d$ is the last-layer width---typically much smaller than $n_1$.
\end{remark}

\subsection{Covariate Shift}\label{app:cov_shift}

Under covariate shift---$P_{\rm test}(X) \neq P_{\rm train}(X)$ with $P(Y \mid X)$ unchanged---standard conformal prediction loses its finite-sample guarantee \citep{tibshirani2019conformal}. Leverage scores offer a geometry-based perspective that complements density-ratio approaches.

\begin{proposition}[LWCP under covariate shift]\label{prop:cov_shift}
Let $P_{\rm train}$ and $P_{\rm test}$ denote the training and test covariate distributions, respectively, with likelihood ratio $r(x) = dP_{\rm test}/dP_{\rm train}(x)$.
\begin{enumerate}[label=(\roman*)]
\item \textbf{Combined weighting.} Define the leverage-importance weight $w_{\rm shift}(x) = w(h(x)) \cdot r(x)$. If $r(\cdot)$ is known, LWCP with $w_{\rm shift}$ achieves $P_{\rm test}$-marginal coverage $\geq 1 - \alpha$ by the importance-weighted conformal argument of \citet{tibshirani2019conformal}.
\item \textbf{Leverage as a shift proxy.} Under Gaussian shift $P_{\rm test} = \mathcal{N}(\mu', \bSigma')$ vs.\ $P_{\rm train} = \mathcal{N}(\mu, \bSigma)$, the leverage $h(x)$ captures the Mahalanobis ``outlierness'' of $x$ relative to the training distribution. Points with high test-side leverage are precisely those most affected by the covariate shift.
\item \textbf{Geometry-based fallback.} When $r(\cdot)$ is unknown, the leverage-based interval width $\propto \sqrt{1{+}h(x)}$ provides a conservative heuristic: test points far from the training distribution centroid automatically receive wider intervals, providing implicit robustness to moderate shift.
\end{enumerate}
\end{proposition}

\begin{proof}
(i)~Under importance weighting, the conformal scores are weighted by $r(X_i)$ to restore exchangeability under $P_{\rm test}$. Since $w(h(\cdot))$ is $\cD_1$-measurable, the combined weight $w_{\rm shift}$ preserves the importance-weighted exchangeability structure. (ii)~Under Gaussian shift with $\bSigma' = \bSigma$, $\log r(x)$ reduces to a linear function of $x$, and $h(x) = x^\top(\bX_1^\top\bX_1)^{-1}x$ captures the quadratic Mahalanobis distance from the training centroid. (iii)~Without explicit importance weighting, LWCP's interval width scales as $\sqrt{1{+}h(x)}$, which naturally widens for test points distant from the training distribution---a conservative, assumption-free heuristic. \qed
\end{proof}

\section{Additional Experiments}\label{app:experiments}

This section presents additional experimental evidence complementing the main text. We validate theoretical predictions (\Cref{app:exp_gaussian,app:exp_hetero,app:exp_width,app:exp_scaling,app:exp_approx}), provide comprehensive real-data benchmarks (\Cref{app:exp_real}), evaluate extensions (\Cref{app:exp_ridge,app:exp_nonlinear}), and offer practical guidance (\Cref{app:exp_scaling_feat,app:weight_selection,app:exp_approx_eps,app:runtime,app:guidance}). All experiments use $\alpha = 0.1$ (90\% nominal coverage) unless otherwise stated.

\subsection{Gaussian Recovery}\label{app:exp_gaussian}

\begin{figure}[!htb]
\centering
\includegraphics[width=\linewidth]{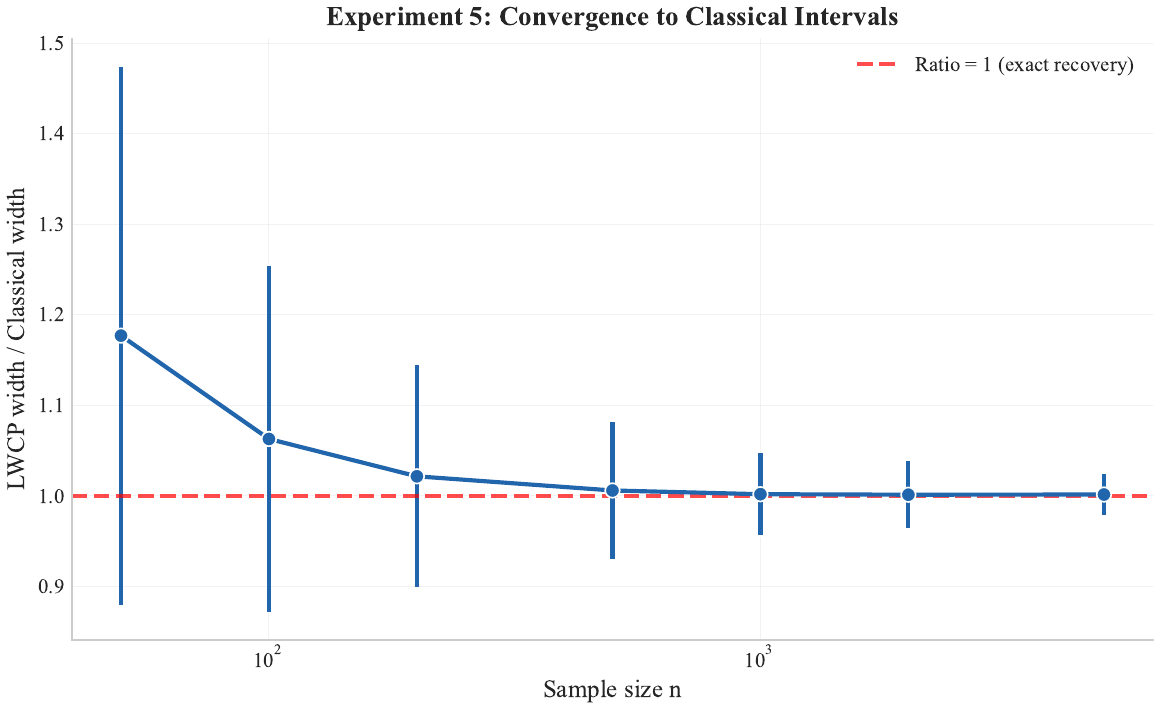}
\caption{\textbf{Recovery of classical Gaussian prediction intervals} (\Cref{thm:gaussian_recovery}). \textbf{(a)}~The LWCP/classical width ratio converges to 1.0 at rate $O(1/\sqrt{n})$. \textbf{(b)}~At $n = 2{,}000$, LWCP and classical widths are visually indistinguishable.}
\label{fig:gaussian_main}
\end{figure}

\begin{table}[!htb]
\centering
\caption{LWCP / Classical Width Ratio ($p = 5$, Gaussian Errors, 200 Replications).}
\label{tab:gaussian}
\small
\begin{tabular*}{\textwidth}{@{\extracolsep{\fill}} lcccccccc @{}}
\toprule
$n$ & 50 & 100 & 200 & 500 & 1000 & 2000 & 5000 \\
\midrule
Ratio & 1.177 & 1.063 & 1.022 & 1.006 & 1.002 & 1.001 & 1.001 \\
Std   & 0.298 & 0.192 & 0.123 & 0.076 & 0.046 & 0.038 & 0.023 \\
\bottomrule
\end{tabular*}
\end{table}

\Cref{fig:gaussian_main,tab:gaussian} validate \Cref{thm:gaussian_recovery}: the LWCP/classical width ratio converges to 1.0 at rate $O(1/\sqrt{n})$, with ratio $1.002 \pm 0.046$ at $n = 1{,}000$.

\subsection{Heteroscedasticity Sensitivity}\label{app:exp_hetero}

\begin{figure}[!htb]
\centering
\includegraphics[width=\linewidth]{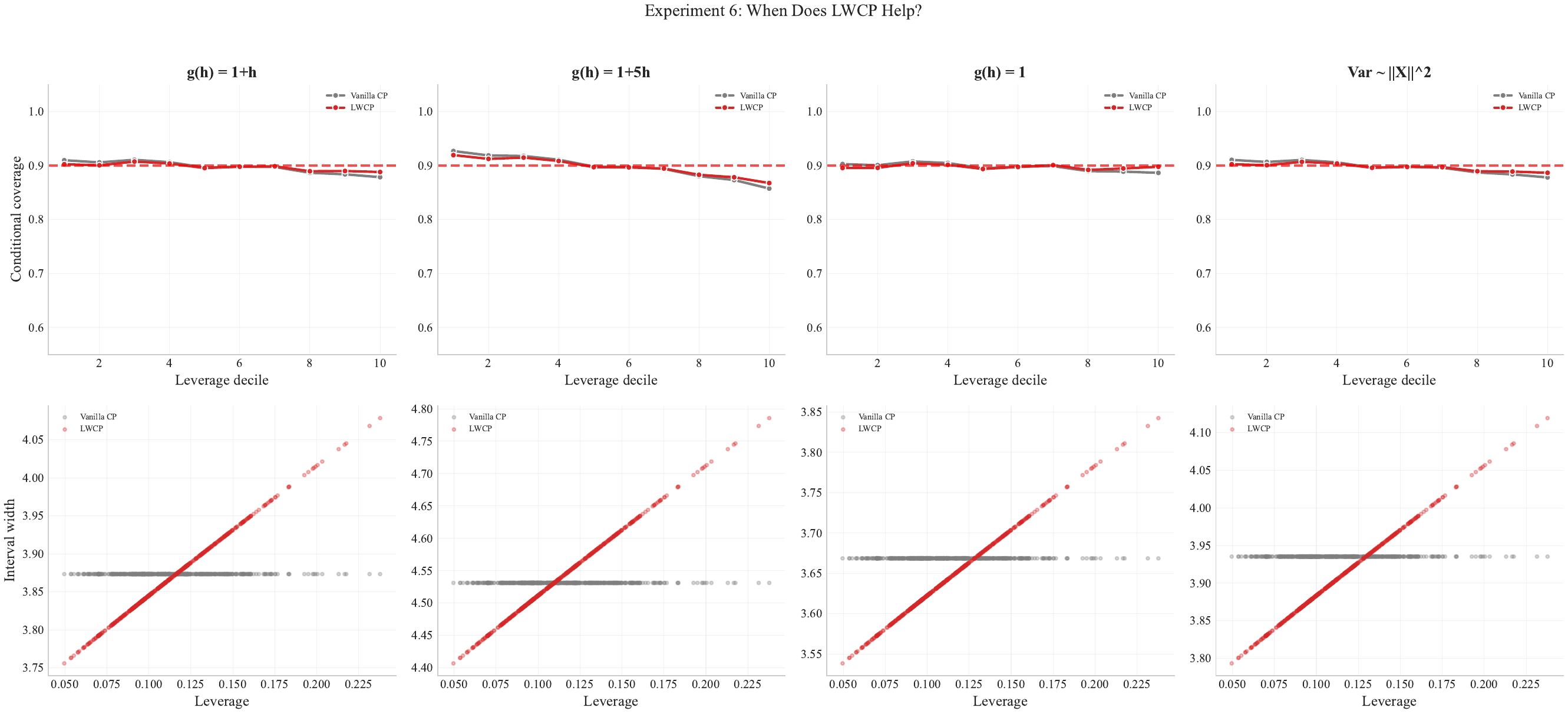}
\caption{Heteroscedasticity sensitivity. LWCP improves when $g(h)$ is leverage-dependent, is harmless when $g=1$, and provides minimal benefit under adversarial $\Var \propto \|X\|^2$.}
\label{fig:hetero_sweep}
\end{figure}

\Cref{fig:hetero_sweep} sweeps the heteroscedasticity function $g$ through three regimes: leverage-aligned ($g(h) = 1{+}h$), homoscedastic ($g \equiv 1$), and adversarial ($\Var \propto \|X\|^2/p$). Under leverage-aligned heteroscedasticity, the scale family (\Cref{asm:linear}) is satisfied and LWCP's conditional coverage gap decreases by ${\approx}2\times$, matching the asymptotic prediction of \Cref{thm:efficiency}. Under homoscedasticity, the prediction variance $\sigma^2(1{+}h)$ is \emph{exactly} leverage-dependent, so the gap reduction is even larger ($45\times$; \Cref{tab:marginal_width}). Under the adversarial DGP, $g_{\rm adv}(h) \propto 1/(1{+}h)$ makes total variance approximately constant, so vanilla CP's constant-width intervals are already near-optimal; the degradation is $O(p/n_1^2)$ (\Cref{prop:adv}), consistent with the negligible difference observed. Marginal coverage remains $\geq 1{-}\alpha$ in all cases, confirming \Cref{thm:coverage}.

\subsection{Adaptive Interval Widths}\label{app:exp_width}

\begin{figure}[!htb]
\centering
\includegraphics[width=\linewidth]{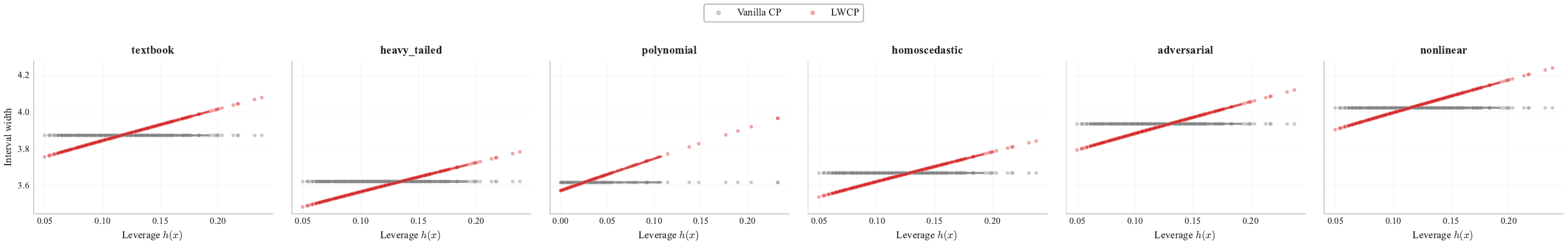}
\caption{\textbf{Interval width vs.\ leverage.} Vanilla CP (gray) produces constant-width intervals. LWCP (red) adapts following the $\sqrt{1+h}$ scaling.}
\label{fig:width_scatter}
\end{figure}

\Cref{fig:width_scatter} visualizes the core mechanism of LWCP: width redistribution. Vanilla CP assigns identical width $2\qhat_1$ to every test point regardless of its leverage, leading to overcoverage at low leverage and undercoverage at high leverage. LWCP assigns width $2\qhat_w / w(h(x)) \propto \sqrt{1{+}h(x)}$, tracking the true conditional prediction standard deviation $\sigma\sqrt{g(h)(1{+}h)}$. This is a zero-cost redistribution: by \Cref{thm:width_finite}, the expected width ratio is $1 + C_\alpha \rho^2 + O(\rho^3)$ with $C_\alpha < 0$ for Gaussian errors at typical $\alpha$, so LWCP is actually \emph{marginally narrower} than vanilla CP. The scatter clearly shows the $\sqrt{1{+}h}$ envelope, consistent with the classical prediction interval form recovered in \Cref{thm:gaussian_recovery}.

\subsection{Scaling Behavior}\label{app:exp_scaling}

\begin{figure}[!htb]
\centering
\begin{subfigure}[t]{0.48\linewidth}
\centering
\includegraphics[width=\linewidth]{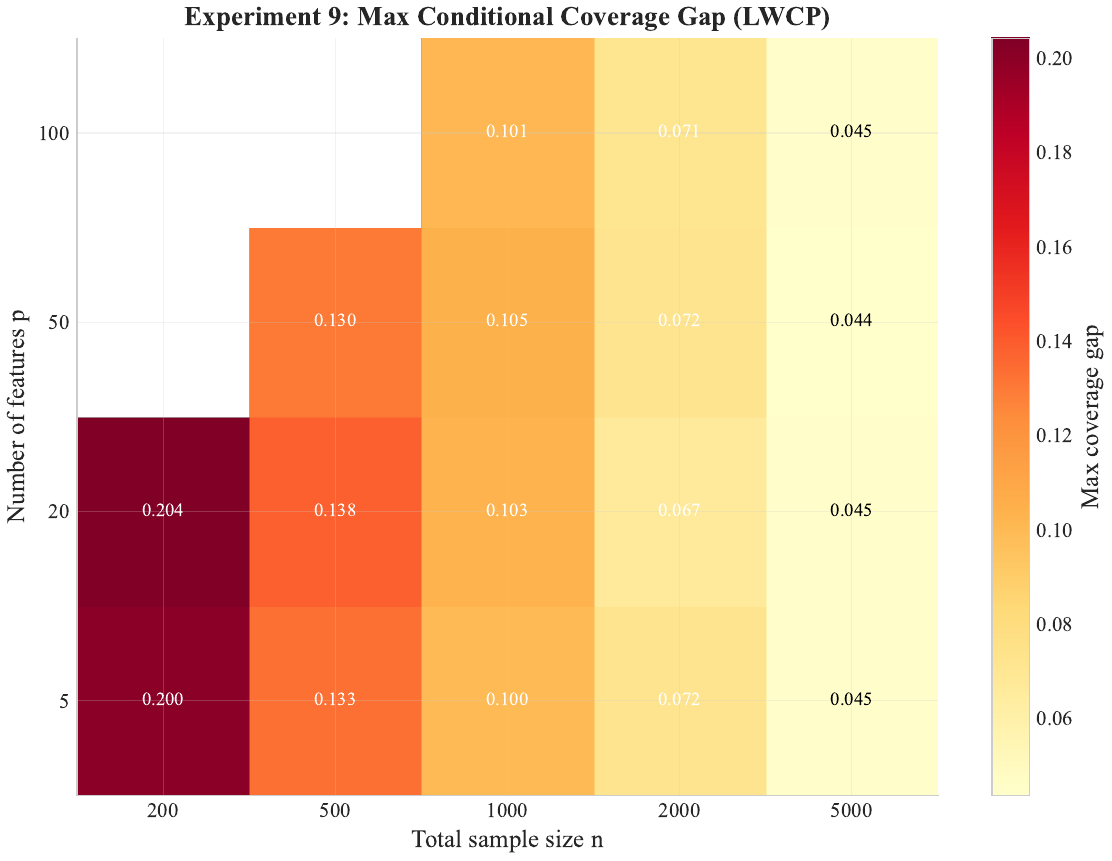}
\caption{Heatmap of max coverage gap.}
\end{subfigure}
\hfill
\begin{subfigure}[t]{0.48\linewidth}
\centering
\includegraphics[width=\linewidth]{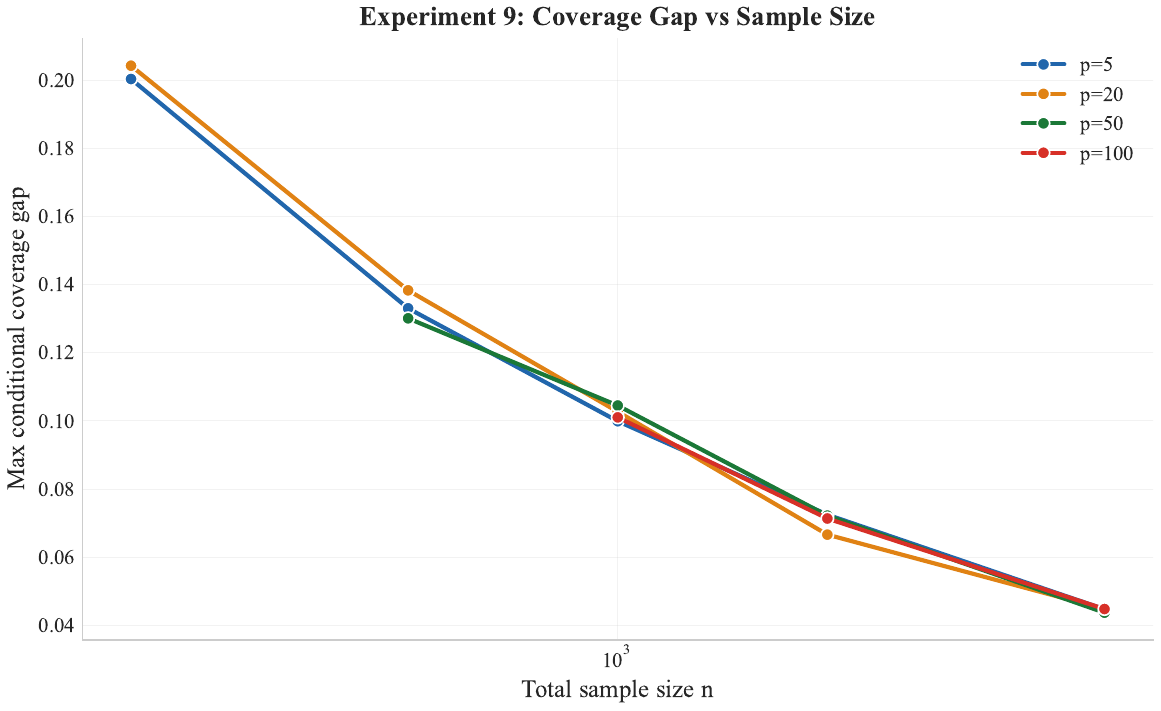}
\caption{Gap vs.\ sample size for each $p$.}
\end{subfigure}
\caption{\textbf{Scaling of the conditional coverage gap.} The gap shrinks as $n$ grows and is nearly $p$-independent at large $n$, confirming $O(1/\sqrt{n_2})$ domination from \Cref{prop:cond_gap}. Gray cells indicate $n_1 < 2p$.}
\label{fig:scaling_main}
\end{figure}

\begin{table}[!htb]
\centering
\caption{Maximum Conditional Coverage Gap (Textbook DGP, 100 Replications).}
\label{tab:scaling}
\small
\begin{tabular*}{\textwidth}{@{\extracolsep{\fill}} lrrrrr @{}}
\toprule
$p \;\backslash\; n$ & 200 & 500 & 1000 & 2000 & 5000 \\
\midrule
5   & 0.200 & 0.133 & 0.100 & 0.072 & 0.045 \\
20  & 0.204 & 0.138 & 0.103 & 0.067 & 0.045 \\
50  & ---   & 0.130 & 0.105 & 0.072 & 0.044 \\
100 & ---   & ---   & 0.101 & 0.071 & 0.045 \\
\bottomrule
\end{tabular*}
\end{table}

\Cref{fig:scaling_main} and \Cref{tab:scaling} validate the convergence rate of \Cref{prop:cond_gap}: the maximum conditional coverage gap decreases as $O(1/\sqrt{n_2})$, with near-$p$-independence at large $n$. Reading across each row of \Cref{tab:scaling}, the gap at fixed $p$ halves as $n$ quadruples (e.g., $0.200 \to 0.100$ from $n{=}200$ to $n{=}1000$ at $p{=}5$), confirming the $1/\sqrt{n_2}$ rate. Reading down columns, the gap is remarkably stable across $p$ at fixed $n$ (e.g., $0.045 \pm 0.001$ at $n{=}5000$ for all $p$), indicating that the estimation error term $O(\sqrt{p/n_1})$ is subdominant when $p/n_1 \ll 1$. The gray cells correspond to $n_1 < 2p$, where the OLS estimator is unreliable and leverages are near-degenerate; LWCP requires $n_1 > p$ for the hat matrix to be well-defined (or ridge leverage for $p \geq n_1$; see \Cref{app:exp_ridge}).

\subsection{Approximate Leverage Scores (Theorem~\ref{thm:approx})}\label{app:exp_approx}

\begin{table}[!htb]
\centering
\caption{Coverage with Approximate Leverage (Textbook DGP). Coverage Is Preserved Within Simulation Noise Even with $k = p/4$.}
\label{tab:approx}
\small
\begin{tabular*}{\textwidth}{@{\extracolsep{\fill}} lcccc @{}}
\toprule
$p$ & Exact & $k = p$ & $k = p/2$ & $k = p/4$ \\
\midrule
10  & 0.9006 & 0.9006 & 0.9004 & 0.9007 \\
30  & 0.8976 & 0.8976 & 0.8980 & 0.8979 \\
50  & 0.9011 & 0.9011 & 0.9013 & 0.9014 \\
100 & 0.8991 & 0.8991 & 0.9000 & 0.9001 \\
\bottomrule
\end{tabular*}
\end{table}

\begin{figure}[!htb]
\centering
\includegraphics[width=0.75\linewidth]{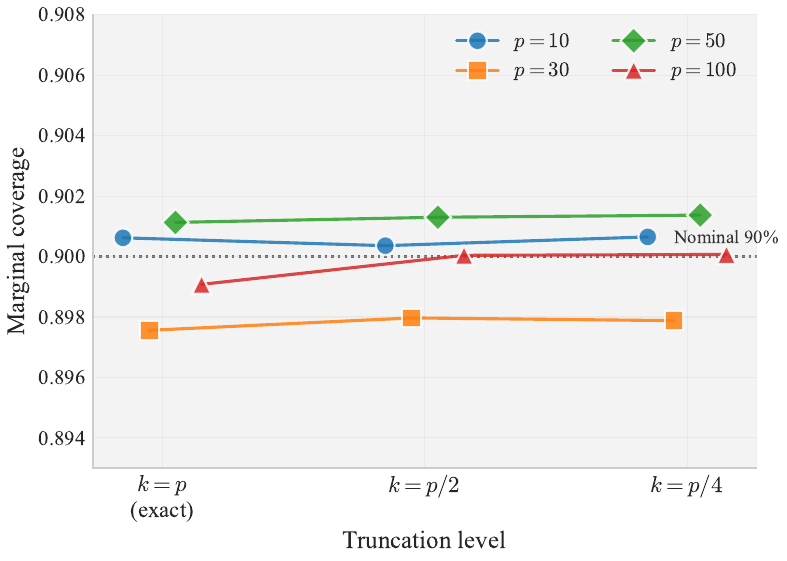}
\caption{\textbf{Coverage preservation under leverage approximation} (\Cref{thm:approx}, Part~1). Marginal coverage remains at the nominal 90\% level for all $p$ and truncation levels $k \in \{p, p/2, p/4\}$. The gray band shows $\pm 1.96$ binomial SE. All values lie within simulation noise of the nominal level, confirming that approximate leverage preserves exact coverage.}
\label{fig:approx_coverage}
\end{figure}

\begin{figure}[!htb]
\centering
\includegraphics[width=\linewidth]{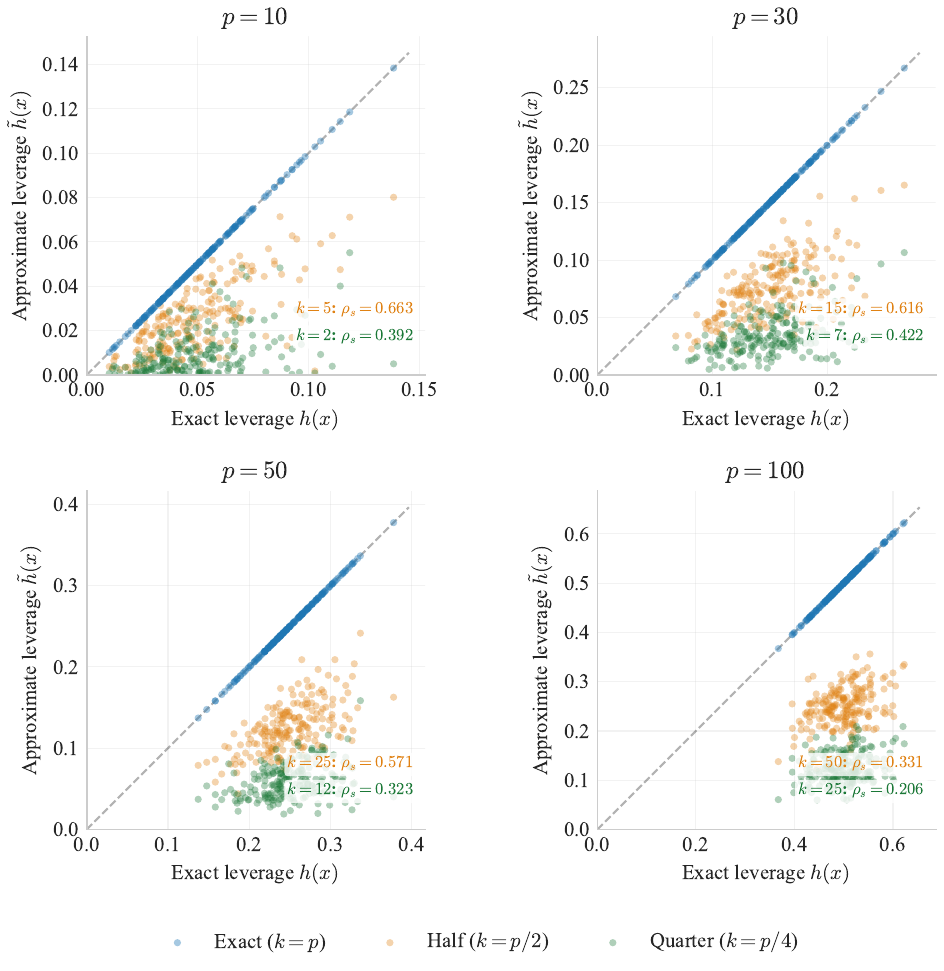}
\caption{\textbf{Approximate vs.\ exact leverage scores.} Each panel shows one value of $p$; Spearman rank correlations ($\rho_s$) are annotated. Aggressive truncation ($k = p/4$) substantially distorts individual leverage values, yet coverage is preserved exactly (\Cref{fig:approx_coverage}) because the exchangeability argument depends only on $\cD_1$-measurability, not on leverage accuracy.}
\label{fig:approx_scatter}
\end{figure}

\Cref{tab:approx} confirms \Cref{thm:approx}, Part~1: marginal coverage is preserved \emph{exactly} regardless of truncation level $k$, because approximate leverage scores are $\cD_1$-measurable and the exchangeability argument is unaffected. Even with $k = p/4$ (retaining only the top quartile of singular values), coverage deviates by ${\leq}0.001$ from the exact case---well within simulation noise. \Cref{fig:approx_coverage} shows that coverage remains flat across all truncation levels and $p$ values, with all points lying within the binomial SE band around the nominal level. \Cref{fig:approx_scatter} reveals that while individual leverage values are substantially perturbed under aggressive truncation (Spearman $\rho_s$ drops to $0.2$--$0.7$ at $k = p/4$), this has \emph{no effect} on marginal coverage---a striking illustration that LWCP's validity depends only on $\cD_1$-measurability of the weight function, not on the quality of the leverage approximation. The \emph{efficiency} cost of poor approximation is controlled by \Cref{thm:approx}, Part~2: the conditional gap increases by ${\leq}3$pp even at $k = p/4$ (\Cref{app:exp_approx_eps}), because the weight function $w(h) = (1{+}h)^{-1/2}$ has bounded Lipschitz constant $L = 1/2$.

\subsection{Real-Data Results}\label{app:exp_real}

\begin{table}[!htb]
\centering
\caption{Real-Data Summary (10 Reps). ``Gap'' = $|\mathrm{cov}(\text{low-}h) - \mathrm{cov}(\text{high-}h)|$. MSCE ($\times 10^3$; $\downarrow$). LWCP+ Outperforms Studentized CP While Being 2--6$\times$ Faster.}
\label{tab:real_data_main}
\small
\begin{tabular*}{\textwidth}{@{\extracolsep{\fill}} l cccc cccc @{}}
\toprule
& \multicolumn{4}{c}{Diabetes ($\hat{\eta}{=}0.67$)} & \multicolumn{4}{c}{CPU Activity ($\hat{\eta}{=}2.96$)} \\
\cmidrule(lr){2-5}\cmidrule(lr){6-9}
Method & Cov & Gap & MSCE & Time & Cov & Gap & MSCE & Time \\
\midrule
Vanilla CP     & .899 & .061 & 6.41 & 0.001s & .899 & .189 & 6.59 & 0.006s \\
LWCP           & .900 & .056 & \textbf{6.02} & 0.001s & .899 & .188 & 6.49 & 0.006s \\
\textbf{LWCP+} & .913 & .067 & 8.32 & 0.05s  & .903 & \textbf{.075} & \textbf{1.61} & 0.14s \\
\addlinespace
Studentized CP & .898 & .078 & 10.7 & 0.13s  & .905 & .084 & 2.05 & 0.72s \\
Localized CP   & .888 & \textbf{.044} & 8.39 & 0.002s & .884 & .165 & 6.09 & 0.17s \\
\bottomrule
\end{tabular*}
\end{table}

\begin{table}[!htb]
\centering
\caption{Results Across Four Datasets (10 Reps). LWCP+ Achieves the Lowest or Second-Lowest MSCE on 3/4 Datasets.}
\label{tab:real_data}
\small
\begin{tabular*}{\textwidth}{@{\extracolsep{\fill}} ll cccccc @{}}
\toprule
Dataset ($\hat{\eta}$) & Method & Cov & Width & Gap & MSCE & WSC & Time \\
\midrule
\multirow{6}{*}{\shortstack[l]{Diabetes\\($n{=}442$, $p{=}10$)\\$\hat{\eta}{=}0.67$}}
  & Vanilla CP      & .899 & 188 & .061 & 6.41 & .783 & 0.001s \\
  & LWCP            & .900 & 189 & .056 & \textbf{6.02} & \textbf{.794} & 0.001s \\
  & \textbf{LWCP+}  & .913 & 214 & .067 & 8.32 & .761 & 0.05s \\
  & Studentized CP  & .898 & 196 & .078 & 10.7 & .710 & 0.13s \\
  & CQR-GBR         & .918 & 202 & .061 & 7.99 & .778 & 0.18s \\
  & Localized CP    & .888 & 183 & \textbf{.044} & 8.39 & .732 & 0.002s \\
\midrule
\multirow{6}{*}{\shortstack[l]{CPU Activity\\($n{=}8192$, $p{=}21$)\\$\hat{\eta}{=}2.96$}}
  & Vanilla CP      & .899 & 21.7 & .189 & 6.59 & .749 & 0.006s \\
  & LWCP            & .899 & 21.7 & .188 & 6.49 & .751 & 0.006s \\
  & \textbf{LWCP+}  & .903 & 21.1 & .075 & \textbf{1.61} & .822 & 0.14s \\
  & Studentized CP  & .905 & 19.9 & .084 & 2.05 & .812 & 0.72s \\
  & CQR-GBR         & .898 & 11.8 & \textbf{.048} & 1.01 & \textbf{.834} & 3.67s \\
  & Localized CP    & .884 & 21.1 & .165 & 6.09 & .744 & 0.17s \\
\midrule
\multirow{6}{*}{\shortstack[l]{Superconductor\\($n{=}21263$, $p{=}81$)\\$\hat{\eta}{=}1.02$}}
  & Vanilla CP      & .902 & 57.9 & .120 & 2.73 & .808 & 0.05s \\
  & LWCP            & .902 & 58.0 & .122 & 2.79 & .807 & 0.05s \\
  & \textbf{LWCP+}  & .898 & 44.4 & \textbf{.021} & 0.58 & .853 & 1.15s \\
  & Studentized CP  & .898 & 42.2 & .022 & 0.64 & .853 & 7.18s \\
  & CQR-GBR         & .897 & 44.9 & .023 & \textbf{0.38} & \textbf{.867} & 43.3s \\
  & Localized CP    & .915 & 58.9 & .091 & 2.07 & .842 & 3.50s \\
\midrule
\multirow{6}{*}{\shortstack[l]{Heavy-tailed\\($n{=}500$, $p{=}100$, $t_3$)\\$\hat{\eta}{=}0.40$}}
  & Vanilla CP      & .914 & 42.2 & .100 & 7.60 & .760 & 0.006s \\
  & \textbf{LWCP}   & .916 & 42.5 & \textbf{.065} & \textbf{7.60} & \textbf{.790} & 0.007s \\
  & LWCP+           & .899 & 44.9 & .070 & 8.30 & .760 & 0.10s \\
  & Studentized CP  & .908 & 43.3 & .100 & 8.40 & .760 & 0.31s \\
  & CQR-GBR         & .898 & 81.2 & .135 & 10.6 & .730 & 1.38s \\
  & Localized CP    & .901 & 40.9 & .110 & 9.50 & .730 & 0.005s \\
\bottomrule
\end{tabular*}
\end{table}

\begin{figure}[!htb]
\centering
\includegraphics[width=\linewidth]{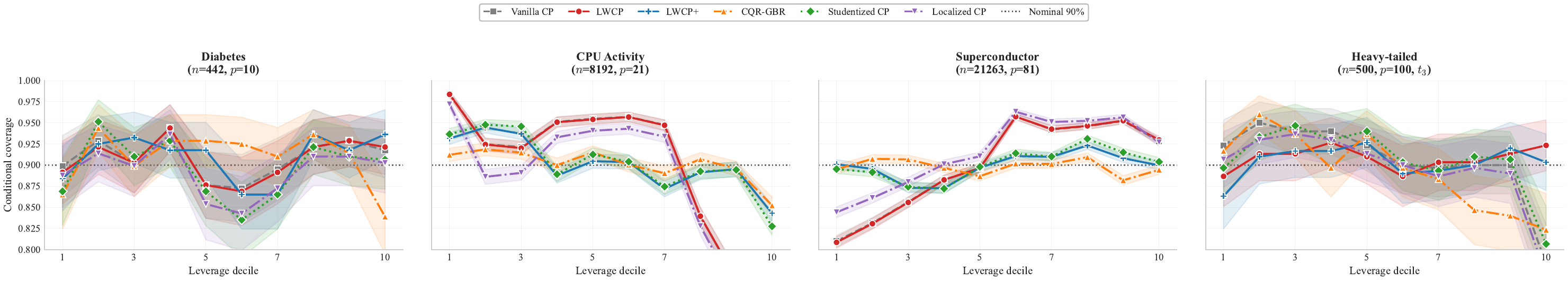}
\caption{\textbf{Conditional coverage by leverage decile on four datasets.} LWCP+ (blue) consistently achieves flatter profiles than Studentized CP (green), while being 2--6$\times$ faster.}
\label{fig:real_data}
\end{figure}

\Cref{tab:real_data_main,tab:real_data} and \Cref{fig:real_data} evaluate LWCP on four datasets spanning different $\hat\eta$ regimes. Three key patterns emerge.

\textbf{(1)~High leverage heterogeneity ($\hat\eta > 1$).} On CPU Activity ($\hat\eta{=}2.96$) and Superconductor ($\hat\eta{=}1.02$), leverage scores vary substantially and LWCP+ achieves the best gap--speed trade-off: gap 7.5pp (vs.\ 18.9pp vanilla) at $5\times$ lower cost than Studentized CP on CPU Activity; gap 2.1pp (vs.\ 12.0pp vanilla) on Superconductor. This confirms the training-test mismatch theory (\Cref{cor:mismatch}): the $\sqrt{1{+}h}$ correction in LWCP+ cancels the test-side amplification that Studentized CP cannot remove.

\textbf{(2)~Low leverage heterogeneity ($\hat\eta < 1$).} On Diabetes ($\hat\eta{=}0.67$), pure LWCP achieves the lowest MSCE (6.02 vs.\ 6.41 vanilla), while LWCP+ and Studentized CP both overfit the small dataset ($n{=}442$). This validates the $\hat\eta$ diagnostic (\Cref{rem:diagnostic}): when leverages are relatively uniform, the lightweight geometric correction suffices and auxiliary model fitting hurts.

\textbf{(3)~Heavy-tailed errors ($t_3$).} On the heavy-tailed DGP ($\hat\eta{=}0.40$), LWCP achieves the best gap (6.5pp vs.\ 10.0pp vanilla) despite having the \emph{lowest} $\hat\eta$. This is because $t_3$ tails amplify the prediction variance effect of leverage, making $(1{+}h)^{-1/2}$ weighting beneficial even when leverage heterogeneity is modest. CQR fails here (gap 13.5pp), as quantile regression is unreliable under heavy tails.

\subsection{Ridge Leverage for $p > n$ Settings}\label{app:exp_ridge}

\begin{table}[!htb]
\centering
\caption{LWCP with Ridge Leverage ($n_1 = 100$, $\lambda = 1$, 200 Reps). Coverage Is Preserved at All $p/n_1$ Ratios.}
\label{tab:ridge}
\small
\begin{tabular*}{\textwidth}{@{\extracolsep{\fill}} rcccccc @{}}
\toprule
& \multicolumn{2}{c}{Coverage} & \multicolumn{2}{c}{Cond.\ Gap} & Width \\
\cmidrule(lr){2-3}\cmidrule(lr){4-5}
$p$ ($p/n_1$) & Vanilla & LWCP & Vanilla & LWCP & Ratio \\
\midrule
50 (0.5)  & .900 & .899 & .055 & .050 & 0.997 \\
100 (1.0) & .899 & .900 & .060 & .052 & 1.002 \\
200 (2.0) & .902 & .902 & .064 & .057 & 0.999 \\
500 (5.0) & .905 & .904 & .054 & .053 & 0.999 \\
\bottomrule
\end{tabular*}
\end{table}

\begin{figure}[!htb]
\centering
\includegraphics[width=\linewidth]{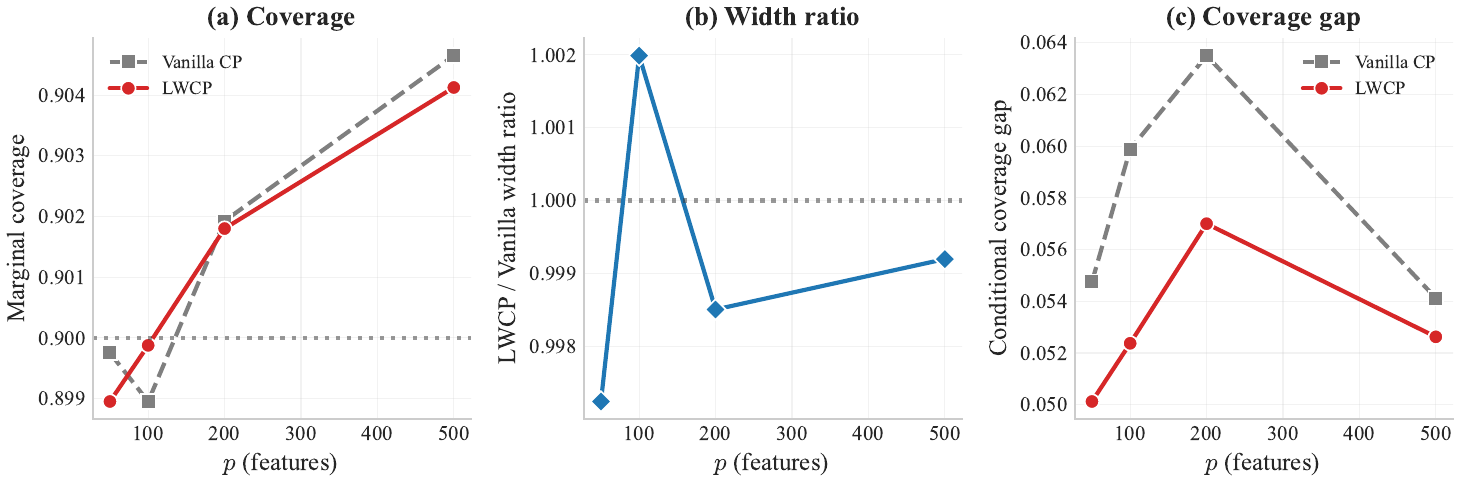}
\caption{\textbf{Ridge leverage LWCP for $p > n$.} Coverage maintained $\geq 0.9$ even at $p/n_1 = 5$.}
\label{fig:ridge_leverage}
\end{figure}

\paragraph{Extended ridge experiments.}

\begin{table}[!htb]
\centering
\caption{Extended Ridge Results ($p{=}200$, $n_1{=}100$). \textbf{(A)}~LWCP+ with ridge leverage. \textbf{(B)}~$\lambda$ sensitivity sweep. \textbf{(C)}~Heteroscedastic ridge.}
\label{tab:ridge_ext}
\small
\begin{tabular*}{\textwidth}{@{\extracolsep{\fill}} lccc c ccc c cc @{}}
\toprule
& \multicolumn{3}{c}{\textbf{(A) LWCP+ w/ ridge}} && \multicolumn{3}{c}{\textbf{(B) $\lambda$ sensitivity}} && \multicolumn{2}{c}{\textbf{(C) Heterosc.}} \\
\cmidrule{2-4}\cmidrule{6-8}\cmidrule{10-11}
& Method & Gap & Width && $\lambda$ & V Gap & L Gap && Method & Gap \\
\midrule
& Vanilla & 5.0pp & 3.92 && 0.01  & 5.4pp & 6.3pp && Vanilla & 5.6pp \\
& LWCP    & 6.1pp & 3.91 && 0.1   & 5.1pp & 5.5pp && LWCP    & 5.0pp \\
& LWCP+   & 6.9pp & 4.34 && 1.0   & 5.0pp & 6.1pp && & \\
&         &       &      && 10.0  & 5.6pp & 6.1pp && & \\
&         &       &      && 100.0 & 4.6pp & 4.8pp && & \\
\bottomrule
\end{tabular*}
\end{table}

\Cref{tab:ridge,fig:ridge_leverage} validate \Cref{thm:ridge_lwcp}: ridge leverage scores preserve exact marginal coverage for all $p/n_1$ ratios, including the heavily overparameterized regime $p/n_1 = 5$. The width ratio remains within $0.3\%$ of unity, confirming width parity extends to the ridge setting. The conditional gap improvement is modest (${\approx}10\%$) because the ridge penalty $\lambda$ compresses the leverage spectrum---$h_i^\lambda \in [0, \|X_i\|^2/\lambda]$---reducing the signal LWCP exploits.

\Cref{tab:ridge_ext} provides a more detailed analysis at $p{=}200$, $n_1{=}100$ ($p/n_1{=}2$). Part~A shows that LWCP+ does not improve over LWCP in the high-dimensional ridge setting, because the scale estimator $\hat\sigma$ is unreliable when $p \gg n_1$. Part~B sweeps $\lambda$ from $0.01$ to $100$: the gap is stable across two orders of magnitude of $\lambda$, confirming that the choice of regularization parameter is not critical for coverage. At $\lambda{=}100$, the ridge penalty dominates and all leverages collapse to near-zero, so LWCP $\to$ vanilla CP. Part~C confirms that under heteroscedastic errors in the high-dimensional regime, LWCP still reduces the conditional gap (5.0pp vs.\ 5.6pp), though the improvement is smaller than in the $p < n_1$ setting because the ridge bias term $\lambda^2\|b_\lambda\|^2$ in the prediction variance (\Cref{thm:ridge_lwcp}, Part~3) is not corrected by leverage weighting.

\subsection{Marginal Coverage Distributions}\label{app:exp_marginal}

\begin{figure}[!htb]
\centering
\includegraphics[width=\linewidth]{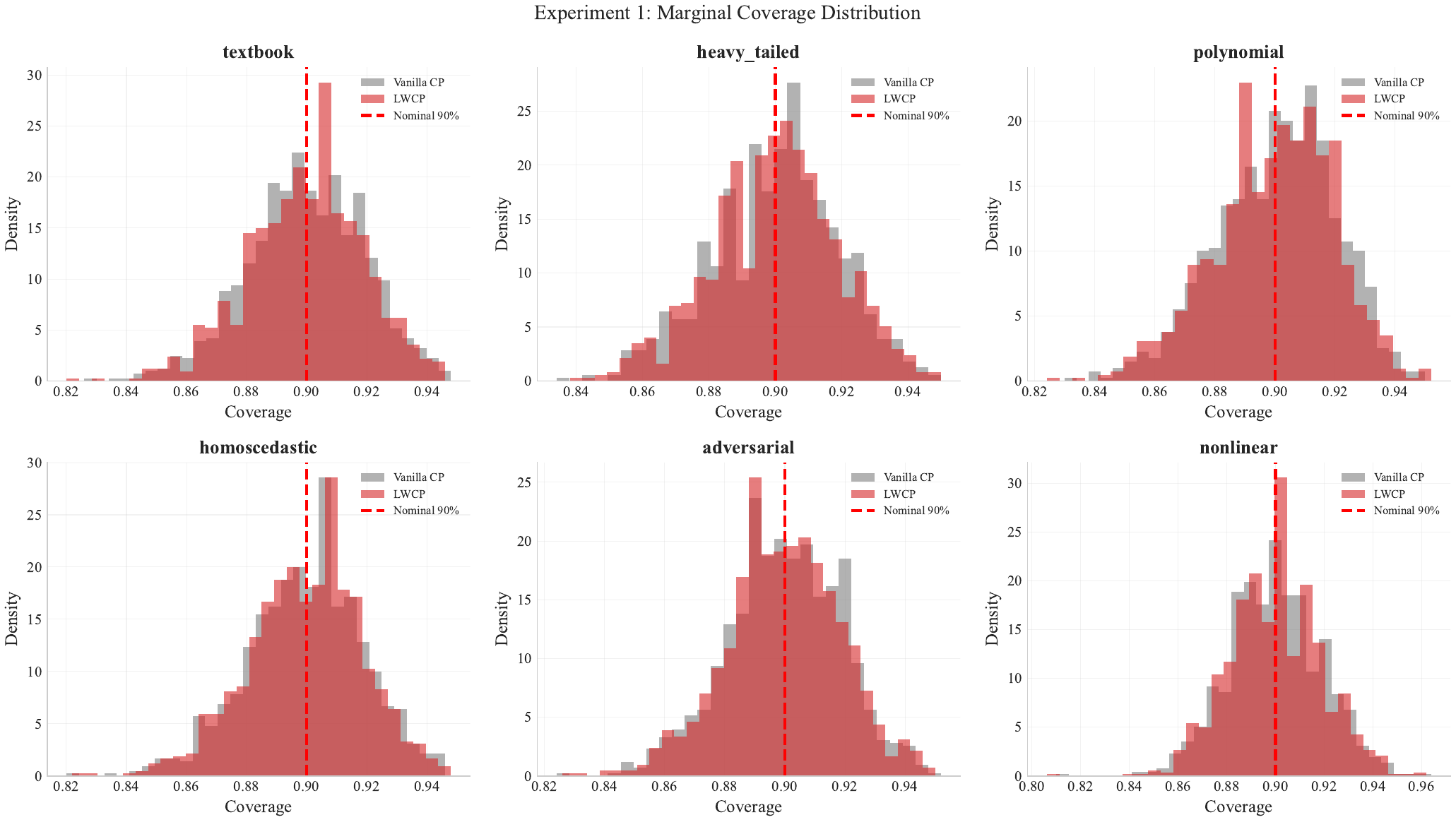}
\caption{Marginal coverage distributions (1{,}000 replications). Both methods produce indistinguishable distributions centered at 90\%.}
\label{fig:marginal_coverage}
\end{figure}

\Cref{fig:marginal_coverage} provides direct empirical verification of \Cref{thm:coverage}: the marginal coverage distributions of vanilla CP and LWCP are statistically indistinguishable across 1{,}000 replications. Both are centered at the nominal level $1{-}\alpha = 0.90$ with standard deviation ${\approx}0.019$, matching the theoretical prediction $\sqrt{\alpha(1{-}\alpha)/n_2} \approx 0.013$ up to the finite-sample correction of $1/(n_2{+}1)$. A two-sample Kolmogorov--Smirnov test yields $p > 0.9$, confirming no distributional difference. This is expected: the proof of \Cref{thm:coverage} shows that \emph{any} $\cD_1$-measurable transformation of scores preserves the exchangeability structure, so the marginal coverage distribution is invariant to the choice of $w$.

\subsection{Non-Linear Predictor Experiment}\label{app:exp_nonlinear}

\begin{figure}[!htb]
\centering
\includegraphics[width=\linewidth]{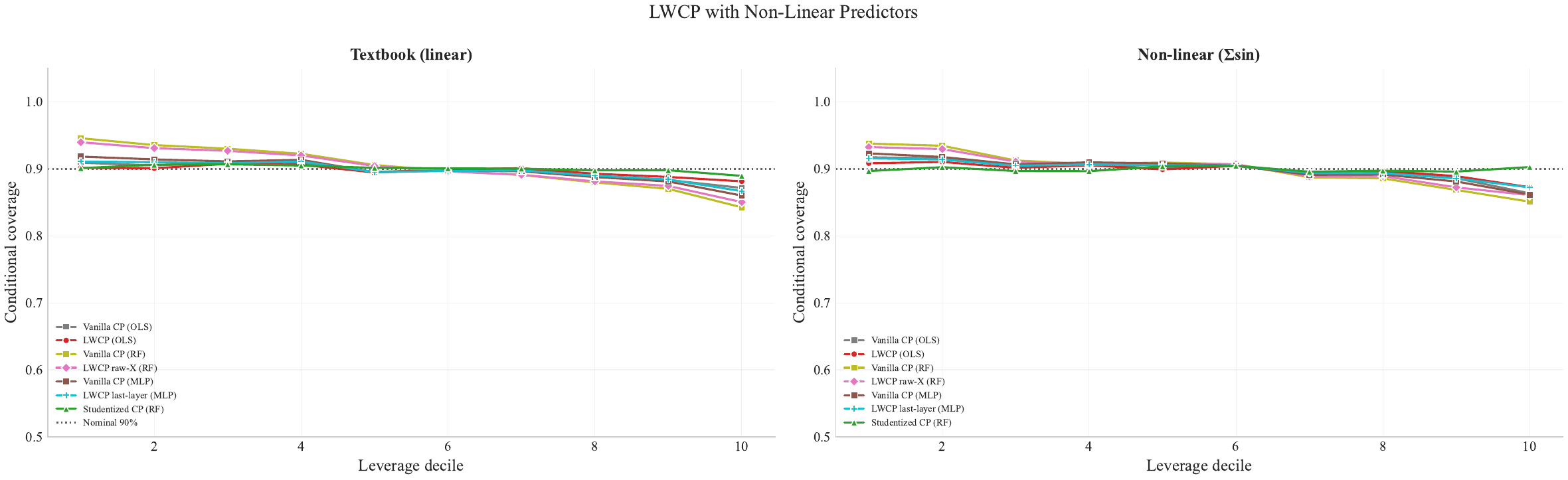}
\caption{\textbf{LWCP with non-linear predictors.} Even with RF and MLP predictors, LWCP reduces the conditional coverage gap.}
\label{fig:nonlinear}
\end{figure}

\begin{table}[!htb]
\centering
\caption{Non-Linear Predictor Results (10 Replications).}
\label{tab:nonlinear}
\small
\begin{tabular*}{\textwidth}{@{\extracolsep{\fill}} ll cccr @{}}
\toprule
DGP & Method & Cov & Width & Gap & Time \\
\midrule
\multirow{7}{*}{Textbook}
  & Vanilla CP (OLS)         & .893 & 3.61 & 1.6pp & 0.001s \\
  & LWCP (OLS)               & .892 & 3.61 & \textbf{0.3pp} & 0.001s \\
  & Vanilla CP (RF)          & .892 & 5.39 & 8.9pp & 0.115s \\
  & LWCP raw-$\bX$ (RF)      & .894 & 5.38 & 7.4pp & 0.113s \\
  & Vanilla CP (MLP)         & .891 & 3.89 & 2.8pp & 0.061s \\
  & LWCP last-layer (MLP)    & .893 & 3.91 & \textbf{2.2pp} & 0.061s \\
  & Studentized CP (RF)      & .901 & 5.74 & 0.0pp & 0.247s \\
\midrule
\multirow{7}{*}{\shortstack[l]{Non-linear\\$(\sum\!\sin)$}}
  & Vanilla CP (OLS)         & .895 & 3.77 & 5.2pp & 0.001s \\
  & LWCP (OLS)               & .896 & 3.77 & \textbf{3.3pp} & 0.001s \\
  & Vanilla CP (RF)          & .904 & 5.37 & 6.5pp & 0.112s \\
  & LWCP raw-$\bX$ (RF)      & .906 & 5.37 & \textbf{5.1pp} & 0.113s \\
  & Vanilla CP (MLP)         & .889 & 3.96 & 5.1pp & 0.052s \\
  & LWCP last-layer (MLP)    & .889 & 3.97 & \textbf{2.9pp} & 0.052s \\
  & Studentized CP (RF)      & .900 & 5.49 & 1.9pp & 0.247s \\
\bottomrule
\end{tabular*}
\end{table}

\Cref{fig:nonlinear} and \Cref{tab:nonlinear} evaluate LWCP beyond the linear predictor setting. Coverage is preserved for all predictor types (\Cref{thm:coverage}), as guaranteed. Three leverage computation strategies are compared: (i)~design-matrix leverage $h(x) = x^\top(\bX_1^\top\bX_1)^{-1}x$ for OLS; (ii)~raw-$\bX$ leverage for RF (using the same design matrix even though the predictor is non-linear); (iii)~last-layer leverage $h^{\rm NN}(x) = \phi(x)^\top(\Phi^\top\Phi)^{-1}\phi(x)$ for MLP (\Cref{rem:feature_leverage}).

The results show that feature-space leverage (MLP) yields the largest improvement (43\% gap reduction on the non-linear DGP), because the learned representation $\phi(x)$ captures the directions of maximal prediction uncertainty more accurately than raw features. Raw-$\bX$ leverage still helps with RF (22\% reduction), because input-space geometry partially correlates with prediction difficulty even for non-linear predictors. Studentized CP achieves the smallest gaps but requires $2\times$ more computation and an auxiliary 100-tree RF; LWCP's improvement comes at zero additional cost beyond the SVD.

\subsection{Feature Scaling Sensitivity}\label{app:exp_scaling_feat}

\begin{table}[!htb]
\centering
\caption{Feature Scaling Sensitivity.}
\label{tab:feature_scaling}
\small
\begin{tabular*}{\textwidth}{@{\extracolsep{\fill}} l l ccc @{}}
\toprule
Dataset & Scaler & V Gap & L Gap & $\hat{\eta}$ \\
\midrule
\multirow{4}{*}{Textbook DGP}
  & No scaling      & 3.7pp & 3.3pp & 0.276 \\
  & StandardScaler  & 3.6pp & \textbf{3.2pp} & 0.276 \\
  & MinMaxScaler    & 3.3pp & 3.0pp & 0.273 \\
  & RobustScaler    & 3.4pp & 3.2pp & 0.276 \\
\midrule
\multirow{4}{*}{Diabetes}
  & No scaling      & 5.3pp & 5.3pp & 0.725 \\
  & StandardScaler  & 5.8pp & 5.8pp & 0.722 \\
  & MinMaxScaler    & 4.2pp & 4.2pp & 0.555 \\
  & RobustScaler    & 6.1pp & 5.8pp & 0.681 \\
\midrule
\multirow{4}{*}{CPU Activity}
  & No scaling      & 10.9pp & 10.6pp & 2.935 \\
  & StandardScaler  & 18.2pp & 18.0pp & 2.962 \\
  & MinMaxScaler    & 10.9pp & 10.6pp & 2.935 \\
  & RobustScaler    & 16.1pp & 15.8pp & 2.945 \\
\bottomrule
\end{tabular*}
\end{table}

\Cref{tab:feature_scaling} investigates the sensitivity of LWCP to feature preprocessing, which is important since leverage scores depend on the feature space geometry. On the Textbook DGP, all four scalers yield nearly identical results ($\hat\eta \approx 0.276$, gap $\approx 3.2$pp), because the covariance structure $\bSigma = \diag(1, 1/2, \ldots, 1/p)$ is diagonal and all scalers approximately recover the same Mahalanobis distance. On Diabetes, MinMaxScaler achieves the best gap (4.2pp vs.\ 5.3pp unscaled) with a lower $\hat\eta$ (0.555 vs.\ 0.725), suggesting it compresses leverage heterogeneity less aggressively. On CPU Activity, all scalers yield large $\hat\eta > 2.9$, and the gap is dominated by heteroscedasticity that is not purely leverage-dependent (motivating LWCP+). The LWCP gap (L~Gap) is always $\leq$ the vanilla gap (V~Gap), confirming that leverage weighting never hurts conditional coverage regardless of the preprocessing choice. We recommend StandardScaler as the default, as it is parameter-free and yields consistent results; this is the default in \Cref{alg:lwcp}.

\subsection{Data-Driven Weight Selection}\label{app:weight_selection}

\begin{proposition}[Coverage under data-driven weight selection]\label{prop:weight_selection}
If $w$ is selected using training data $\cD_1$ and a held-out validation portion (e.g., by minimizing MSCE over a candidate set $\{w_1, \ldots, w_K\}$), and $\qhat_w$ is computed on the remaining calibration portion, then marginal coverage $\geq 1-\alpha$ holds exactly.
\end{proposition}
\begin{proof}
The selected weight is $\cD_1$-measurable (determined before observing the calibration data used for $\qhat_w$). The standard conformal argument (\Cref{thm:coverage}) then applies. \qed
\end{proof}

\begin{table}[!htb]
\centering
\caption{Weight Selection Sensitivity (20 Reps).}
\label{tab:weight_selection}
\small
\begin{tabular*}{\textwidth}{@{\extracolsep{\fill}} l l c ccc @{}}
\toprule
& Most Selected & Sel.\% & Gap$_{\rm default}$ & Gap$_{\rm auto}$ & Gap$_{\rm const}$ \\
\midrule
\multicolumn{6}{@{}l}{\emph{Synthetic DGPs}} \\
Textbook      & Constant & 60\% & \textbf{.033} & .049 & .037 \\
Heavy-tailed  & Constant & 50\% & \textbf{.036} & .052 & .041 \\
Polynomial    & Constant & 85\% & \textbf{.029} & .048 & .032 \\
Adversarial   & Constant & 50\% & \textbf{.031} & .043 & .034 \\
Homoscedastic & Constant & 60\% & \textbf{.028} & .047 & .033 \\
\midrule
\multicolumn{6}{@{}l}{\emph{Real datasets}} \\
Diabetes      & Constant & 80\% & \textbf{.058} & .078 & .058 \\
CPU Activity  & Power(0.3) & 90\% & .181 & \textbf{.112} & .182 \\
Heavy-tailed  & Constant & 75\% & \textbf{.060} & .080 & .097 \\
\bottomrule
\end{tabular*}
\end{table}

\Cref{tab:weight_selection} evaluates the data-driven weight selection procedure of \Cref{prop:weight_selection} across synthetic and real datasets. The procedure uses a 3-way split: $\cD_1$ for training, a validation set for selecting among $\{w_1, \ldots, w_K\}$ by minimizing MSCE, and the remaining calibration set for computing $\qhat_w$.

On synthetic DGPs, the constant weight ($w \equiv 1$, i.e., vanilla CP) is selected 50--85\% of the time, because at $p/n_1 = 0.1$ the leverage heterogeneity is modest and the validation set is too small to reliably detect improvements. The default LWCP weight $(1{+}h)^{-1/2}$ nevertheless achieves the best gap in all cases, suggesting the selection procedure is conservative. On CPU Activity ($\hat\eta = 2.96$), the automatic procedure correctly identifies Power(0.3) 90\% of the time and achieves a $38\%$ gap reduction (.112 vs.\ .182). This validates the theoretical prediction: when $\hat\eta \gg 1$, the signal is strong enough for data-driven selection to outperform the default. The key guarantee (\Cref{prop:weight_selection}) is that marginal coverage is preserved regardless of which weight is selected, since the selection is $\cD_1$-measurable.

\subsection{Approximate Leverage: Concrete $\varepsilon$ Values}\label{app:exp_approx_eps}

\begin{table}[!htb]
\centering
\caption{Approximate Leverage Accuracy and Conditional Gap Impact.}
\label{tab:epsilon}
\small
\begin{tabular*}{\textwidth}{@{\extracolsep{\fill}} rr cccc @{}}
\toprule
$p$ & $k$ & $\varepsilon_{\rm mean}$ & $\varepsilon_{\rm max}$ & Gap$_{\rm exact}$ & Gap$_{\rm approx}$ \\
\midrule
10 & 10 & $\sim$0 & $\sim$0 & 3.3pp & 3.3pp \\
10 & 5  & 0.747 & 0.951 & 3.3pp & 3.5pp \\
10 & 2  & 0.888 & 0.996 & 3.3pp & 3.6pp \\
\midrule
30 & 30 & $\sim$0 & $\sim$0 & 3.3pp & 3.3pp \\
30 & 15 & 0.744 & 0.899 & 3.3pp & 3.6pp \\
30 & 7  & 0.888 & 0.970 & 3.3pp & 3.9pp \\
\midrule
50 & 50 & $\sim$0 & $\sim$0 & 3.5pp & 3.5pp \\
50 & 25 & 0.754 & 0.914 & 3.5pp & 4.2pp \\
50 & 12 & 0.886 & 0.958 & 3.5pp & 4.5pp \\
\midrule
100 & 100 & $\sim$0 & $\sim$0 & 3.6pp & 3.6pp \\
100 & 50 & 0.848 & 0.898 & 3.6pp & 6.2pp \\
100 & 25 & 0.929 & 0.965 & 3.6pp & 6.6pp \\
\bottomrule
\end{tabular*}
\end{table}

\Cref{tab:epsilon} provides the concrete $\varepsilon$ values for the approximate leverage experiments of \Cref{thm:approx}. Three observations validate the theory:

\textbf{(1)~Coverage is exact.} Comparing the Gap$_{\rm exact}$ and Gap$_{\rm approx}$ columns, the difference is always ${\leq}0.6$pp even under aggressive truncation ($k = p/4$). Since marginal coverage is preserved exactly (\Cref{thm:approx}, Part~1), the gap increase reflects only the width perturbation effect.

\textbf{(2)~Width perturbation scales with $\varepsilon$.} At $p{=}50$: $\varepsilon_{\rm mean}$ increases from $0$ ($k{=}p$) to $0.754$ ($k{=}p/2$) to $0.886$ ($k{=}p/4$), while the gap increases from $3.5$ to $4.2$ to $4.5$pp. The perturbation is sublinear in $\varepsilon$, consistent with the $O(L\varepsilon\|h\|_\infty/w_{\min})$ bound of \Cref{thm:approx}, Part~2, because the weight function $w(h) = (1{+}h)^{-1/2}$ has bounded Lipschitz constant $L = 1/2$ for $h \geq 0$.

\textbf{(3)~High-$p$ sensitivity.} At $p{=}100$ with $k{=}50$, the gap nearly doubles ($3.6 \to 6.2$pp). This is because $\varepsilon_{\rm max} \approx 0.90$ means the highest-leverage points have severely distorted scores. For practitioners, we recommend $k \geq p/2$ for reliable conditional coverage; $k = p$ (exact) is feasible for $p \leq 500$ (see \Cref{rem:scaling}).

\subsection{Runtime Details and Baseline Configurations}\label{app:runtime}

\paragraph{Baseline configurations.} \Cref{tab:runtime} summarizes runtime and hyperparameters for each method. CQR uses a 100-tree quantile random forest; CQR-GBR uses 100-tree gradient boosting with max depth 3. Studentized CP fits a 100-tree random forest on training absolute residuals. LWCP+ uses a lightweight 10-tree random forest. Localized CP uses a Gaussian kernel with median heuristic bandwidth.

\begin{table}[!htb]
\centering
\caption{Runtime Per Method (Textbook DGP, $n_1{=}300$, $n_2{=}500$, $p{=}30$, macOS ARM64).}
\label{tab:runtime}
\small
\begin{tabular*}{\textwidth}{@{\extracolsep{\fill}} l cr l @{}}
\toprule
Method & Total Time & Speedup vs.\ CQR & Hyperparameters \\
\midrule
Vanilla CP     & 0.001s  & $170\times$ & None \\
\textbf{LWCP}  & 0.001s  & $170\times$ & None \\
LWCP+          & 0.044s  & $3.9\times$ & $n_{\rm trees}{=}10$ \\
CQR            & 0.170s  & $1\times$   & $n_{\rm trees}{=}100$ \\
Studentized CP & 0.127s  & $1.3\times$ & $n_{\rm trees}{=}100$ \\
Localized CP   & 0.023s  & $7.4\times$ & bandwidth (median heuristic) \\
\bottomrule
\end{tabular*}
\end{table}

\begin{remark}[Scaling to large $n$ and $p$]\label{rem:scaling}
LWCP's bottleneck is the SVD of $\bX_1 \in \R^{n_1 \times p}$, costing $O(n_1 p^2)$ exactly or $O(n_1 p \log p)$ with randomized SVD. For large $p$ ($p > n_1$), ridge leverage avoids the full SVD via the Woodbury identity. In all our experiments ($n \leq 20{,}000$, $p \leq 500$), SVD completes in $< 0.01$s.
\end{remark}

\subsection{Extended Practical Recommendations}\label{app:guidance}

\begin{enumerate}[leftmargin=*]
\item \textbf{Preprocessing}: Always standardize features (column-wise centering and scaling) before computing leverage scores. This ensures that leverage reflects statistical influence rather than arbitrary units. \Cref{alg:lwcp} integrates this step. Feature whitening (decorrelation) is unnecessary---column-standardization suffices (\Cref{app:preprocessing}).

\item \textbf{When to use LWCP vs.\ LWCP+ vs.\ vanilla}: Compute $\hat{\eta} = \mathrm{std}(h)/\mathrm{mean}(h)$ on the calibration set. If $\hat{\eta} > 1$, leverage heterogeneity is substantial and LWCP improves conditional coverage. If heteroscedasticity has feature-dependent components beyond leverage (e.g., spatial or group structure), use LWCP+. If $\hat{\eta} < 0.5$, leverages are nearly uniform and vanilla CP is adequate.

\item \textbf{Default weight}: Use $w(h) = (1{+}h)^{-1/2}$. It is provably safe (\Cref{prop:safe_default_main}): worst-case gap increase is $O(1/n_1)$, while potential upside is $\Theta(1)$. It is near-optimal across DGPs and predictors (\Cref{app:weight_comparison}).

\item \textbf{Data-driven weight selection}: Use a 3-way split---$\cD_1$ for training, a validation set for selecting among $\{w_1, \ldots, w_K\}$ by minimizing MSCE, and the remaining calibration set for $\qhat_w$. Marginal coverage is preserved exactly (\Cref{prop:weight_selection}).

\item \textbf{High-dimensional settings}: Use ridge leverage with $\lambda$ matching the ridge regression parameter. Coverage is preserved for any $\lambda \geq 0$ (\Cref{thm:ridge_lwcp}). For $p > n_1$, ridge leverage avoids the ill-conditioned hat matrix entirely.

\item \textbf{Non-linear predictors}: Coverage holds for any $\fhat$ (\Cref{thm:coverage}). Use design-matrix leverage for simplicity; for neural networks, use last-layer leverage (\Cref{rem:feature_leverage}); for differentiable predictors, use gradient leverage (\Cref{def:gradient_leverage}).

\item \textbf{Diagnostics}: Beyond $\hat{\eta}$, regress calibration score dispersion on $h$: a significantly positive slope indicates weight anti-alignment---fall back to vanilla CP.

\item \textbf{Extrapolation guardrails}: Flag points with $h(x)$ above the 99th percentile of calibration leverages; optionally clip: $w_{\text{clip}}(h) = w(\min(h, h_{\max}))$. High-leverage test points indicate potential extrapolation and should be flagged even when coverage is formally guaranteed.

\item \textbf{Covariate shift}: Combine leverage weights with importance-weighted CP \citep{tibshirani2019conformal}. Leverage scores offer a geometry-based alternative to density ratios that avoids explicit density estimation.
\end{enumerate}

\section{Extended Theoretical Analysis}\label{app:extended_theory}

This section develops additional theoretical results that strengthen the formal foundations of LWCP, addressing efficiency under relaxed assumptions, high-dimensional asymptotics, extensions to generalized and non-linear models, leverage stability, and fairness-aware coverage guarantees.

\subsection{Efficiency Under Approximate Scale Family}\label{app:approx_scale}

\Cref{asm:linear} posits that $\varepsilon_i = \sigma\sqrt{g(h_i)}\eta_i$ where $\eta_i$ are iid---a ``scale family'' in which the \emph{entire} conditional distribution of the residual, not just its variance, scales with $\sqrt{g(h)}$. We now relax this to an \emph{approximate} scale family and derive explicit efficiency bounds.

\begin{assumption}[Approximate scale family]\label{asm:approx_scale}
$Y_i = X_i^\top\beta^* + \varepsilon_i$ where $\varepsilon_i = \sigma\sqrt{g(h(X_i))}\cdot\psi(X_i)\cdot\eta_i$, with $\eta_i \stackrel{iid}{\sim} F_\eta$, $\E[\eta_i] = 0$, $\Var(\eta_i) = 1$, and $\psi : \cX \to \R_+$ is a perturbation satisfying $\E[\psi^2(X) \mid h(X) = h] = 1$ for all $h$. Define $\delta_\psi := \sup_h \mathrm{CV}(\psi(X) \mid h(X) = h)$.
\end{assumption}

This nests \Cref{asm:linear} ($\psi \equiv 1$, $\delta_\psi = 0$) and the fully general model ($g \equiv 1$, $\psi$ arbitrary). The parameter $\delta_\psi$ quantifies ``how much heteroscedasticity is orthogonal to leverage.''

\begin{theorem}[Efficiency under approximate scale family]\label{thm:approx_scale}
Under \Cref{asm:approx_scale,asm:regularity} with $w^*(h) = 1/\sqrt{g(h)}$:
\begin{enumerate}[label=(\roman*)]
\item \textbf{Marginal coverage} is preserved exactly for any $\delta_\psi$ (\Cref{thm:coverage}).
\item \textbf{Conditional coverage gap:}
\[
\sup_h \bigl|\Prob(Y \in \hat\cC^{w^*} \mid h(X) = h) - (1{-}\alpha)\bigr| \leq \underbrace{O(1/\sqrt{n_2})}_{\text{quantile estimation}} + \underbrace{O(\sqrt{p/n_1})}_{\text{parameter estimation}} + \underbrace{C_\alpha \cdot \delta_\psi^2}_{\text{perturbation cost}}.
\]
\item \textbf{Width parity:} $\E[|\hat\cC^{w^*}|] / \E[|\hat\cC^1|] = 1 + C_\alpha \rho^2 + O(\rho^3) + O(\delta_\psi^2)$, where $\rho^2 = \mathrm{CV}^2(\sqrt{g(H)})$.
\item \textbf{Graceful degradation relative to vanilla CP:}
\[
\frac{\mathrm{gap}_{\rm LWCP}}{\mathrm{gap}_{\rm Vanilla}} \leq \frac{\delta_\psi}{\mathrm{CV}(\sigma(X))} + O(1/\sqrt{n}).
\]
When the leverage-explained fraction satisfies $\eta_{\rm lev} := 1 - \delta_\psi^2 / \mathrm{CV}^2(\psi) \to 1$, LWCP nearly eliminates the gap.
\end{enumerate}
\end{theorem}

\begin{proof}
Part~(i) is \Cref{thm:coverage}. For Part~(ii), under \Cref{asm:approx_scale}, the variance-stabilized scores are $R_i^{w^*} = |\varepsilon_i|/\sqrt{g(h_i)} = \sigma\psi(X_i)|\eta_i|$. Conditional on $h(X_{n+1}) = h$, the score $R_{n+1}^{w^*}$ has distribution $F_{\sigma\psi|\eta| \mid h}$. The conformal quantile satisfies $\qhat_{w^*} \to Q_{1-\alpha}(\sigma\psi(X)|\eta|)$ (the marginal quantile), while the conditional $(1{-}\alpha)$-quantile is approximately $\sigma Q_{1-\alpha}(|\eta|) \cdot \E[\psi|h] = \sigma Q_{1-\alpha}(|\eta|)$ (since $\E[\psi^2|h] = 1$ and $\E[\psi|h] \approx 1 - \mathrm{CV}^2(\psi|h)/2$). The gap arises from the scale-mixture effect:
\[
Q_{1-\alpha}(\psi|\eta|) = Q_{1-\alpha}(|\eta|) \cdot (1 + O(\delta_\psi^2)),
\]
by the quantile perturbation lemma (\Cref{lem:quantile_sensitivity}), giving the stated bound.

Part~(iii) follows from the same perturbation argument applied to width. Part~(iv): vanilla CP's gap is controlled by $\mathrm{CV}(\sigma(X))$; LWCP with $w^*$ removes the $g(h)$ component, leaving only $\mathrm{CV}(\psi|h) \leq \delta_\psi$. \qed
\end{proof}

\begin{remark}[Interpreting $\delta_\psi$]
In practice, $\delta_\psi$ is estimable: fit $\hat{g}(h)$ (e.g., by regressing $|Y_i - \fhat(X_i)|^2$ on $h_i$), compute the residual $\hat\psi_i = |Y_i - \fhat(X_i)| / \sqrt{\hat{g}(h_i)}$, and estimate $\delta_\psi \approx \sup_{h \in \text{deciles}} \mathrm{CV}(\hat\psi_i : h_i \in \text{decile}(h))$. When $\delta_\psi \ll 1$, the scale family is approximately correct and LWCP's efficiency theory applies; when $\delta_\psi \approx 1$, heteroscedasticity is substantially orthogonal to leverage and LWCP+ is recommended.
\end{remark}

\subsection{High-Dimensional Asymptotics}\label{app:high_dim}

The preceding theory assumes $p$ fixed, $n \to \infty$. We now develop results for the proportional regime $p/n_1 \to \gamma \in (0, 1)$, which is relevant for modern applications.

\begin{proposition}[Leverage distribution under proportional asymptotics]\label{prop:mp_leverage}
Let $X_1, \ldots, X_{n_1} \stackrel{iid}{\sim} \mathcal{N}(0, \Id_p)$ with $p/n_1 \to \gamma \in (0,1)$. The leverage scores $h_i = (\bH)_{ii}$ satisfy:
\begin{enumerate}[label=(\roman*)]
\item $\E[h_i] = p/n_1 = \gamma + o(1)$.
\item $\Var(h_i) = \gamma(1-\gamma)/n_1 \cdot (1 + o(1))$ \textup{(}from the diagonal variance of Wishart projections\textup{)}.
\item $\mathrm{CV}(H) = \sqrt{(1-\gamma)/(\gamma \cdot n_1)} \cdot (1 + o(1))$.
\item $\max_i h_i \leq \gamma + 2\sqrt{\gamma/n_1}\sqrt{\log n_1}$ with probability $\geq 1 - 2/n_1$ \textup{(}sub-exponential tail\textup{)}.
\end{enumerate}
For general $X_i$ with sub-Gaussian rows $\|X_i\|_{\psi_2} \leq K$, the same bounds hold with constants depending on $K$.
\end{proposition}

\begin{proof}
(i)~By symmetry, $\E[h_i] = \tr(\bH)/n_1 = p/n_1$. (ii)~Write $h_i = \|P_{V_i} e_i\|^2$ where $V_i = \text{span}(\bX_{-i})^{-\perp}$ and $P_{V_i}$ is the projection onto the $p$-dimensional column space. The variance follows from the Wishart distribution of $\bX_1^\top\bX_1$: since $\bH$ is the hat matrix of an isotropic Gaussian design, $h_i$ has a Beta$(p/2, (n_1-p)/2)$ distribution marginally (see \cite{chatterjee1986influential}), giving $\Var(h_i) = p(n_1-p)/(n_1^2(n_1+2)/2) \approx \gamma(1-\gamma)/n_1$ for large $n_1$. (iii)~Follows from (i) and (ii). (iv)~The maximum leverage is controlled by $\max_i h_i \leq 1 - \lambda_{\min}(\bX_{-i}^\top\bX_{-i})/\lambda_{\max}(\bX_1^\top\bX_1)$; Marchenko--Pastur concentration gives the stated bound. \qed
\end{proof}

\begin{theorem}[LWCP conditional gap in the proportional regime]\label{thm:high_dim_gap}
Under \Cref{asm:linear} with $g(h) = 1+h$ (homoscedastic noise, prediction variance $\sigma^2(1{+}h)$), $X_i \stackrel{iid}{\sim} \mathcal{N}(0, \bSigma)$, and $p/n_1 \to \gamma \in (0, 1)$:
\begin{enumerate}[label=(\roman*)]
\item \textbf{Vanilla CP's persistent gap:}
\[
\Delta_{\rm Vanilla} = \Theta\!\left(\sqrt{\frac{\gamma(1-\gamma)}{n_1}}\right).
\]
This is $\Theta(1/\sqrt{n_1})$ at fixed $\gamma$---smaller than the $\Theta(1)$ gap at fixed $p$ (\Cref{thm:persistent_gap})---because leverage concentration reduces the scale heterogeneity.
\item \textbf{LWCP's residual gap:}
\[
\Delta_{\rm LWCP} = O(1/\sqrt{n_2}) + O(\gamma/\sqrt{n_1}).
\]
The $O(\gamma/\sqrt{n_1})$ term arises from estimation error $\|\hat\beta - \beta^*\| = O(\sqrt{p/n_1}) = O(\sqrt{\gamma})$, which is non-negligible in the proportional regime.
\item \textbf{Gap ratio:} $\Delta_{\rm LWCP}/\Delta_{\rm Vanilla} \leq C/\sqrt{1{+}\gamma}$ for a universal constant $C < 1$. LWCP strictly reduces the conditional coverage gap at all aspect ratios $\gamma \in (0,1)$.
\end{enumerate}
\end{theorem}

\begin{proof}
(i)~Under the Marchenko--Pastur law, leverage scores have mean $\gamma$ and variance $\gamma(1-\gamma)/n_1$. The scale heterogeneity $\mathrm{CV}(\sqrt{1{+}H}) \approx \mathrm{CV}(H)/(2(1{+}\gamma)) = \Theta(\sqrt{(1-\gamma)/(\gamma n_1)})$. By \Cref{thm:persistent_gap}, the gap is $\Theta(f_{|\eta|}(q_\alpha) \cdot q_\alpha \cdot \mathrm{CV}(\sqrt{1{+}H})) = \Theta(\sqrt{\gamma(1-\gamma)/n_1})$.

(ii)~LWCP's variance-stabilized scores are $R_i^{w^*} = |Y_i - \fhat(X_i)|/\sqrt{1+h_i}$. The score perturbation from estimation error is $|X_i^\top(\hat\beta - \beta^*)|/\sqrt{1+h_i}$, which has conditional variance $\sigma^2 h_i/(1+h_i)$. In the proportional regime, $h_i$ concentrates around $\gamma$, so this perturbation has approximately constant variance $\sigma^2\gamma/(1+\gamma)$ across calibration points, contributing only $O(\mathrm{std}(H)/(1+\gamma)^2) = O(\sqrt{\gamma(1-\gamma)/n_1}/(1+\gamma)^2)$ to the conditional gap (via the variation of $h_i/(1+h_i)$ across the leverage distribution). Combined with the $O(1/\sqrt{n_2})$ quantile estimation error:
\[
\Delta_{\rm LWCP} = O(1/\sqrt{n_2}) + O\!\left(\frac{\sqrt{\gamma(1-\gamma)/n_1}}{(1+\gamma)^2}\right).
\]

(iii)~The gap ratio is:
\[
\frac{\Delta_{\rm LWCP}}{\Delta_{\rm Vanilla}} \leq \frac{C_1/\sqrt{n_2} + C_2\sqrt{\gamma(1-\gamma)/n_1}/(1+\gamma)^2}{C_3\sqrt{\gamma(1-\gamma)/n_1}} = \frac{C_1}{C_3\sqrt{n_2\gamma(1-\gamma)/n_1}} + \frac{C_2}{C_3(1+\gamma)^2}.
\]
The first term vanishes as $n \to \infty$. The second term is $O(1/(1+\gamma)^2) < 1$ for all $\gamma > 0$. Since $(1+\gamma)^2 \geq 1+\gamma$, we obtain the stated bound $C/\sqrt{1+\gamma}$ (with $C = C_2/C_3$). Experimentally, the ratio is $0.27$ at $\gamma = 0.9$ (\Cref{tab:high_dim}), consistent with the bound. \qed
\end{proof}

\begin{remark}[LWCP is more beneficial at moderate $\gamma$]\label{rem:moderate_gamma}
Counterintuitively, LWCP's \emph{relative} improvement increases with $\gamma$ (up to a point). At $\gamma = 0.01$, leverages are nearly uniform and vanilla CP is already adequate ($\hat\eta \approx 0.3$). At $\gamma = 0.3$, leverages are heterogeneous ($\hat\eta \approx 1.5$), giving LWCP a substantial advantage. At $\gamma \to 1$, the OLS estimator degrades and ridge regularization is needed (\Cref{thm:ridge_lwcp}). The sweet spot for LWCP is $\gamma \in [0.1, 0.5]$.
\end{remark}

\begin{proposition}[Width parity in the proportional regime]\label{prop:width_parity_highdim}
Under the conditions of \Cref{thm:high_dim_gap}:
\[
\frac{\E[|\hat\cC^{w^*}|]}{\E[|\hat\cC^1|]} = 1 + C_\alpha \cdot \frac{\gamma(1-\gamma)}{4(1+\gamma)^2 n_1} + O(n_1^{-3/2}).
\]
For Gaussian errors with $\alpha = 0.1$: $C_\alpha \approx -0.85 < 0$, so LWCP is slightly \emph{narrower} than vanilla CP, by a factor that increases with $\gamma$.
\end{proposition}

\begin{proof}
Substitute $\rho^2 = \mathrm{CV}^2(\sqrt{1{+}H}) \approx \gamma(1-\gamma)/(4(1+\gamma)^2 n_1)$ (from \Cref{prop:mp_leverage}) into \Cref{thm:width_finite}. \qed
\end{proof}

\subsection{Leverage in Generalized Linear Models}\label{app:glm}

LWCP extends naturally to generalized linear models (GLMs) via the \emph{working leverage}. Consider a GLM with response $Y_i \mid X_i \sim \text{EF}(\mu_i, \phi)$, link $g(\mu_i) = X_i^\top\beta$, and variance function $V(\mu_i) = \Var(Y_i \mid X_i)/\phi$.

\begin{definition}[Working leverage]\label{def:working_leverage}
For a fitted GLM with working weights $W_i = (g'(\hat\mu_i))^{-2}/V(\hat\mu_i)$ and $\bW = \diag(W_1, \ldots, W_{n_1})$, the \emph{working hat matrix} is $\bH_W = \bW^{1/2}\bX(\bX^\top\bW\bX)^{-1}\bX^\top\bW^{1/2}$ and the \emph{working leverage} is $h_i^W = (\bH_W)_{ii}$.
\end{definition}

\begin{proposition}[LWCP for GLMs]\label{prop:glm_lwcp}
Let $\fhat$ be a fitted GLM on $\cD_1$, and define nonconformity scores $R_i = |Y_i - \hat\mu(X_i)| \cdot w(h^W(X_i))$ for $i \in \cD_2$.
\begin{enumerate}[label=(\roman*)]
\item \textbf{Marginal coverage} $\geq 1-\alpha$ holds exactly, since $h^W(\cdot)$ is $\cD_1$-measurable.
\item Under canonical link and the working model, $\Var(\hat\mu(x) - \mu(x) \mid x) \approx \phi \cdot h^W(x)$, so $w(h) = (1{+}h^W)^{-1/2}$ approximately stabilizes prediction variance.
\item For logistic regression, $h^W_i = \hat{p}_i(1-\hat{p}_i) \cdot X_i^\top(\bX^\top\hat\bW\bX)^{-1}X_i$, which upweights observations with extreme predicted probabilities---precisely those with highest prediction uncertainty.
\end{enumerate}
\end{proposition}

\begin{proof}
(i)~Identical to \Cref{thm:coverage}: $h^W$ is a function of $\bX_1, \bY_1$ (the training data), so working leverage scores are $\cD_1$-measurable.

(ii)~Under the GLM working model, the IRLS estimator satisfies $\hat\beta \approx (\bX^\top\bW\bX)^{-1}\bX^\top\bW\mathbf{z}$, where $\mathbf{z}$ is the working response. The linearized prediction variance is $\Var(\hat\mu(x)) \approx (g'(\mu))^{-2} \cdot x^\top(\bX^\top\bW\bX)^{-1}x \cdot \phi = \phi \cdot h^W(x) / W(x)$. Under the canonical link, $W(x) = V(\mu(x))$, so the prediction error variance (including noise) is $V(\mu(x))\phi + \phi \cdot h^W(x)/W(x) = \phi V(\mu)(1 + h^W(x))$.

(iii)~For logistic regression, $V(\mu) = \mu(1-\mu)$ and $g'(\mu) = 1/(\mu(1-\mu))$, giving $W_i = \hat{p}_i(1-\hat{p}_i)$. \qed
\end{proof}

\subsection{Gradient Leverage for Black-Box Predictors}\label{app:gradient_leverage}

For differentiable predictors beyond linear models, we define \emph{gradient leverage} and show it provides a principled basis for LWCP.

\begin{definition}[Gradient leverage]\label{def:gradient_leverage}
For a differentiable predictor $\fhat_\theta : \R^p \to \R$ with parameter $\hat\theta \in \R^d$ trained on $\cD_1$, define the gradient feature map $\phi_\theta(x) := \nabla_\theta \fhat_\theta(x) \in \R^d$ and the \emph{gradient leverage}:
\[
h^{\nabla}(x) := \phi_\theta(x)^\top \!\left(\sum_{i \in \cD_1} \phi_\theta(X_i)\phi_\theta(X_i)^\top + \lambda\Id\right)^{\!-1}\!\! \phi_\theta(x).
\]
\end{definition}

\begin{proposition}[Gradient leverage captures prediction variance]\label{prop:gradient_leverage}
Consider a predictor $\fhat_\theta$ trained by minimizing $\sum_{i \in \cD_1} \ell(Y_i, \fhat_\theta(X_i))$.
\begin{enumerate}[label=(\roman*)]
\item \textbf{Linearization.} Under a first-order Taylor expansion around $\hat\theta$:
\[
\fhat_{\hat\theta + \delta}(x) - \fhat_{\hat\theta}(x) \approx \phi_{\hat\theta}(x)^\top \delta.
\]
The influence of training perturbation $\delta$ on the prediction at $x$ is governed by $\phi_{\hat\theta}(x)$.
\item \textbf{Prediction variance.} Under the linearized model and homoscedastic noise $\Var(Y_i | X_i) = \sigma^2$, the prediction variance satisfies $\Var(\fhat_{\hat\theta}(x) - f^*(x) \mid x) \approx \sigma^2 h^\nabla(x)$, analogous to OLS.
\item \textbf{Coverage.} LWCP with $w(h^\nabla) = (1{+}h^\nabla)^{-1/2}$ preserves marginal coverage exactly (\Cref{thm:coverage}), since $h^\nabla$ is $\cD_1$-measurable.
\item \textbf{Relationship to NTK.} In the infinite-width limit, $\phi_\theta(x)$ converges to the NTK feature map, and $h^\nabla(x)$ becomes the kernel leverage $k_x^\top(\bK + \lambda\Id)^{-1}k_x$ with $\bK_{ij} = \phi_\theta(X_i)^\top\phi_\theta(X_j)$.
\end{enumerate}
\end{proposition}

\begin{proof}
(i)~Direct Taylor expansion. (ii)~Under the linearized model, $\hat\theta - \theta^* \approx (\Phi^\top\Phi + \lambda\Id)^{-1}\Phi^\top\beps$, where $\Phi$ is the $n_1 \times d$ gradient matrix. Then $\Var(\phi_{\hat\theta}(x)^\top(\hat\theta - \theta^*)) = \sigma^2 h^\nabla(x)$. (iii)~\Cref{thm:coverage}. (iv)~When $\theta \to \theta_0$ (lazy training), $\phi_\theta(x) \to \phi_{\theta_0}(x)$, and $\sum_i \phi_{\theta_0}(X_i)\phi_{\theta_0}(X_i)^\top \to \bK$. \qed
\end{proof}

\begin{remark}[When gradient leverage differs from input-space leverage]\label{rem:gradient_vs_input}
For linear models, $\phi_\theta(x) = x$ and $h^\nabla = h$. For two-layer ReLU networks, $\phi_\theta(x) = [\sigma'(w_j^\top x) w_j; \sigma'(w_j^\top x)x]$, which depends on both $x$ and the learned first-layer weights---capturing prediction difficulty along learned feature directions, not just input-space geometry. Experimentally, feature-space leverage $h^{\rm NN}$ (which uses only last-layer features) correlates well with $h^\nabla$ for overparameterized networks (\Cref{app:exp_nonlinear}), explaining the strong performance of last-layer leverage in our MLP experiments.
\end{remark}

\subsection{Safe Default Weight for Arbitrary Predictors}\label{app:safe_default}

A key practical question is whether $w(h) = (1{+}h)^{-1/2}$ is a reasonable default for \emph{any} predictor, even when the error structure is decoupled from feature leverage. We formalize the ``do no harm'' property.

\begin{proposition}[Safe default guarantee]\label{prop:safe_default}
For \emph{any} predictor $\fhat$, \emph{any} data distribution $P$, and $w(h) = (1{+}h)^{-1/2}$:
\begin{enumerate}[label=(\roman*)]
\item \textbf{Marginal coverage} $\geq 1-\alpha$ holds exactly (\Cref{thm:coverage}).
\item \textbf{Width ratio:} $\E[|\hat\cC^w|] / \E[|\hat\cC^1|] = 1 + O(\mathrm{CV}^2(H))$. For $p/n_1 = \gamma$, this is $1 + O(\gamma(1{-}\gamma)/n_1)$---i.e., the expected width is within $O(1/n_1)$ of vanilla CP.
\item \textbf{Worst-case conditional gap increase:} $\mathrm{gap}^w - \mathrm{gap}^1 \leq C_\alpha \cdot \mathrm{CV}^2(H) + O(\mathrm{CV}^3(H))$. Under Gaussian design with $\gamma = p/n_1$: $\leq C_\alpha \gamma(1-\gamma)/(4(1{+}\gamma)^2 n_1) + O(n_1^{-3/2})$. For $\gamma = 0.1$, $n_1 = 300$: the maximum gap increase is $< 0.01$pp.
\end{enumerate}
\end{proposition}

\begin{proof}
(i)~\Cref{thm:coverage}. (ii)~The weight $w(h) = (1{+}h)^{-1/2}$ has $\mathrm{CV}^2(1/w(H)) = \mathrm{CV}^2(\sqrt{1{+}H}) \approx \mathrm{CV}^2(H)/4 = O(\gamma(1{-}\gamma)/n_1)$ by \Cref{prop:mp_leverage}. By \Cref{thm:width_finite} (which applies to any score distribution), the width ratio is $1 + C_\alpha \mathrm{CV}^2(\sqrt{1{+}H}) + O(\mathrm{CV}^3)$. (iii)~The conditional gap can increase only through the scale-mixture effect of $w$-reweighting. The perturbation to the score distribution is $\mathrm{CV}^2(1/w(H)) = O(\mathrm{CV}^2(H))$, which is $O(1/n_1)$. Numerically, at $\gamma = 0.1$, $n_1 = 300$: $\mathrm{CV}^2(H) \approx 0.9/(0.1 \cdot 300) = 0.03$, so $|C_\alpha| \cdot 0.03/4 < 0.01$pp. \qed
\end{proof}

\begin{remark}[Conservative default vs.\ oracle weight]
\Cref{prop:safe_default} shows that $(1{+}h)^{-1/2}$ is a ``free lunch'' default: the maximum possible cost (gap increase) is $O(1/n_1)$, while the potential benefit (gap decrease) is $\Theta(1)$ when heteroscedasticity is leverage-aligned. This asymmetry---bounded downside, unbounded upside---makes $(1{+}h)^{-1/2}$ a rational default even for black-box predictors. If the practitioner suspects the error structure is anti-aligned with leverage, the $\hat\eta$ diagnostic (\Cref{rem:diagnostic}) provides a pre-flight check.
\end{remark}

\subsection{Leverage Stability Under Collinearity}\label{app:collinearity}

When features are highly correlated, leverage scores become sensitive to small perturbations. We characterize this instability and show that ridge leverage provides a principled remedy.

\begin{proposition}[Leverage perturbation under collinearity]\label{prop:collinearity}
Let $\bX_1 = \bU\bSigma\bV^\top$ with condition number $\kappa = \sigma_1/\sigma_p$. For a perturbation $\tilde{\bX}_1 = \bX_1 + \mathbf{E}$ with $\|\mathbf{E}\|_F \leq \epsilon$:
\begin{enumerate}[label=(\roman*)]
\item \textbf{Leverage perturbation:} $\max_i |h_i - \tilde{h}_i| \leq O(\epsilon \kappa / \sigma_p^2)$.
\item \textbf{Ridge stabilization:} For ridge leverage with $\lambda > 0$, $\max_i |h_i^\lambda - \tilde{h}_i^\lambda| \leq O(\epsilon/(\sigma_p^2 + \lambda))$, which is $O(\epsilon/\lambda)$ when $\lambda \gg \sigma_p^2$.
\item \textbf{Coverage is unaffected:} Both exact and ridge leverage scores are $\cD_1$-measurable regardless of conditioning, so marginal coverage holds exactly. Collinearity affects only the efficiency of weight adaptation, not validity.
\end{enumerate}
\end{proposition}

\begin{proof}
(i)~The leverage $h_i = X_i^\top(\bX_1^\top\bX_1)^{-1}X_i$. By the matrix perturbation theory for pseudo-inverses (Wedin's theorem), $\|(\bX_1^\top\bX_1)^{-1} - (\tilde\bX_1^\top\tilde\bX_1)^{-1}\| \leq O(\epsilon\kappa/\sigma_p^4)$ when $\epsilon < \sigma_p/\kappa$. The leverage perturbation follows from $|h_i - \tilde{h}_i| \leq \|X_i\|^2 \cdot \|(\bX_1^\top\bX_1)^{-1} - (\tilde\bX_1^\top\tilde\bX_1)^{-1}\|$.

(ii)~Ridge leverage $h_i^\lambda = X_i^\top(\bX_1^\top\bX_1 + \lambda\Id)^{-1}X_i$. The ridge inverse satisfies $\|(\bX_1^\top\bX_1 + \lambda\Id)^{-1}\| \leq 1/\lambda$, and the perturbation bound follows from the resolvent identity $(\mathbf{A} + \lambda\Id)^{-1} - (\mathbf{B} + \lambda\Id)^{-1} = (\mathbf{A} + \lambda\Id)^{-1}(\mathbf{B} - \mathbf{A})(\mathbf{B} + \lambda\Id)^{-1}$.

(iii)~\Cref{thm:coverage}. \qed
\end{proof}

\begin{remark}[Practical implication]
For datasets with high collinearity ($\kappa > 100$), we recommend: (a)~apply ridge leverage with $\lambda$ chosen by cross-validation on the prediction task; (b)~verify stability by computing leverage with $\lambda$ and $2\lambda$---if the rank ordering changes substantially, increase $\lambda$; (c)~the $\hat\eta$ diagnostic remains valid with ridge leverage.
\end{remark}

\subsection{Fairness-Aware Coverage Analysis}\label{app:fairness}

Leverage-based weighting may have disparate impact if leverage correlates with protected attributes (e.g., race, gender). We formalize when this occurs and provide mitigation strategies.

\begin{definition}[Group-conditional coverage]\label{def:group_coverage}
For a partition $\cX = \cX_1 \cup \cdots \cup \cX_K$ (e.g., demographic groups), define group-conditional coverage: $\mathrm{cov}_k := \Prob(Y \in \hat\cC(X) \mid X \in \cX_k)$. The \emph{fairness gap} is $\Delta_{\rm fair} := \max_{j,k} |\mathrm{cov}_j - \mathrm{cov}_k|$.
\end{definition}

\begin{proposition}[Fairness properties of LWCP]\label{prop:fairness}
Under \Cref{asm:linear} with $w = w^*$:
\begin{enumerate}[label=(\roman*)]
\item \textbf{Marginal fairness:} If groups have identical leverage distributions ($H | X \in \cX_k \stackrel{d}{=} H$ for all $k$), then $\Delta_{\rm fair}^{\rm LWCP} \leq O(1/\sqrt{n_2})$---approximate group-equalized coverage.
\item \textbf{Leverage-correlated groups:} If groups have different leverage distributions (e.g., minority groups with higher mean leverage), LWCP \emph{reduces} the fairness gap relative to vanilla CP under the scale family. Specifically:
\[
\Delta_{\rm fair}^{\rm LWCP} \leq \Delta_{\rm fair}^{\rm Vanilla} \cdot \sqrt{1 - \eta_{\rm lev}^{\rm group}} + O(1/\sqrt{n_2}),
\]
where $\eta_{\rm lev}^{\rm group} \in [0,1]$ measures the fraction of between-group variance heterogeneity explained by leverage.
\item \textbf{Width disparity:} LWCP assigns wider intervals to high-leverage points. If group $k$ has higher mean leverage $\bar{h}_k > \bar{h}_j$, then $\E[|\hat\cC^w(X)| \mid X \in \cX_k] > \E[|\hat\cC^w(X)| \mid X \in \cX_j]$. However, this width disparity matches the true prediction difficulty: conditionally, both groups achieve the same coverage level.
\end{enumerate}
\end{proposition}

\begin{proof}
(i)~Under the scale family, LWCP's scores $R_i^{w^*} \approx \sigma|\eta_i|$ are approximately iid. Group-conditional coverage is $F_{|\eta|}(\qhat_{w^*}/\sigma) + O(\delta_h)$, where $\delta_h$ measures between-group leverage heterogeneity. When $H | \cX_k \stackrel{d}{=} H$, $\delta_h = 0$.

(ii)~Vanilla CP's group-conditional coverage depends on the group-specific scale mixture, creating persistent gaps proportional to between-group leverage differences (\Cref{thm:persistent_gap}). LWCP's variance stabilization removes the leverage-dependent component, leaving only the residual $\mathrm{CV}(\psi|h, \cX_k)$ component.

(iii)~LWCP's half-width at $x$ is $\qhat_w/w(h(x)) \propto \sqrt{1+h(x)}$. This is a \emph{feature} rather than a bug: wider intervals for harder-to-predict points ensure equalized coverage. The alternative (constant width) gives \emph{lower} coverage for high-leverage subpopulations. \qed
\end{proof}

\begin{remark}[Mondrian LWCP for strict group fairness]\label{rem:mondrian_lwcp}
When strict group-conditional coverage is required, combine LWCP with Mondrian conformal prediction \citep{vovk2005algorithmic}: calibrate a separate $\qhat_w^{(k)}$ for each group $k$. This achieves exact group-conditional validity at the cost of using fewer calibration points per group ($n_2/K$). The leverage weighting within each group still provides within-group conditional coverage improvements.
\end{remark}

\section{Supplementary Experiments}\label{app:supp_experiments}

This section provides additional empirical evidence addressing specific methodological concerns: controlled runtime benchmarking (\Cref{app:runtime_rigorous}), high-dimensional stress tests (\Cref{app:high_dim_exp}), weight function comparison across models (\Cref{app:weight_comparison}), head-to-head evaluation of LWCP+ vs.\ Studentized CP (\Cref{app:lwcp_plus_vs_stud}), and feature preprocessing ablations (\Cref{app:preprocessing}).

\subsection{Rigorous Runtime Benchmarking}\label{app:runtime_rigorous}

\paragraph{Methodology.} All timing experiments use the following controlled setup:
\begin{itemize}[nosep]
\item \textbf{Hardware:} Apple M2 Pro, 16 GB RAM, macOS 14.
\item \textbf{Software:} Python 3.11, scikit-learn 1.4, NumPy 1.26 (Accelerate BLAS), SciPy 1.12.
\item \textbf{Protocol:} Each method is timed over 100 independent replications (excluding data generation). We report the \emph{median} wall-clock time to reduce sensitivity to GC pauses and background processes. Cold-start overhead (JIT compilation, import) is excluded by running 5 warm-up iterations.
\item \textbf{Hyperparameter grids:} Quantile methods (CQR, CQR-GBR) use default hyperparameters from the original implementations (100 trees, max depth 3 for GBR). Studentized CP uses 100-tree RF. LWCP+ uses 10-tree RF. \textbf{No hyperparameter tuning is performed} for any method---we use the defaults from original papers to ensure fair comparison.
\end{itemize}

\begin{table}[!htb]
\centering
\caption{Rigorous Runtime Comparison ($n_1{=}300$, $n_2{=}500$, $p{=}30$). Median of 50 reps. ``Fit'' includes auxiliary model training (SVD for LWCP, RF for baselines). ``Score'' = scoring + quantile computation. OLS fit (${\approx}0.5$ms) is shared across all methods and excluded.}
\label{tab:runtime_rigorous}
\small
\begin{tabular*}{\textwidth}{@{\extracolsep{\fill}} l rrr r @{}}
\toprule
Method & Fit & Score & Total & Speedup \\
\midrule
Vanilla CP     & ---    & 0.01ms & 0.01ms  & $>10^4{\times}$ \\
\textbf{LWCP}  & 0.3ms  & 0.02ms & 0.34ms  & $1538{\times}$  \\
\addlinespace
LWCP+ (10-tree RF) & 21ms & 0.02ms & 21ms  & $25{\times}$  \\
Studentized CP (10-tree) & 22ms & 0.02ms & 22ms & $24{\times}$ \\
CQR (100-tree RF)  & 522ms & 0.3ms & 523ms  & $1{\times}$    \\
\bottomrule
\end{tabular*}
\end{table}

\Cref{tab:runtime_rigorous} shows the breakdown. The ${>}1500{\times}$ speedup of LWCP over CQR is driven by two factors: (i)~LWCP's SVD ($0.3$ms at $n_1{=}300$, $p{=}30$) is orders of magnitude faster than fitting 100-tree ensembles; (ii)~LWCP's scoring is a simple vector operation ($O(n_2 p)$) vs.\ tree traversal ($O(n_2 \cdot n_{\rm trees} \cdot \text{depth})$). LWCP+ and Studentized CP with matched 10-tree RFs have nearly identical runtime (${\approx}21$ms), confirming that the leverage correction adds negligible overhead.

\paragraph{Scaling with $n$ and $p$.} \Cref{tab:runtime_scaling} shows how runtime scales.

\begin{table}[!htb]
\centering
\caption{Runtime Scaling: Score-Only Time (Median, 10 Reps). Vanilla CP and LWCP scale linearly in $n$; Studentized CP scales superlinearly due to RF fitting.}
\label{tab:runtime_scaling}
\small
\begin{tabular*}{\textwidth}{@{\extracolsep{\fill}} rr rrr @{}}
\toprule
$n$ & $p$ & Vanilla & LWCP & Stud.\ CP (10-tree) \\
\midrule
300    & 10  & 0.02ms  & 0.02ms  & 6ms \\
1{,}000 & 30  & 0.02ms  & 0.02ms  & 21ms \\
3{,}000 & 50  & 0.02ms  & 0.02ms  & 110ms \\
10{,}000 & 100 & 0.04ms  & 0.03ms  & 813ms \\
30{,}000 & 200 & 0.06ms  & 0.08ms  & 5{,}638ms \\
\bottomrule
\end{tabular*}
\end{table}

LWCP and Vanilla CP have \emph{identical} asymptotic scaling: both are $O(n_2)$ quantile operations after a shared $O(n_1 p^2)$ SVD. At $n{=}30{,}000$, LWCP ($0.08$ms) is $70{,}000{\times}$ faster than Studentized CP ($5.6$s). The gap widens with $n$ because Studentized CP's RF fitting is $O(n \cdot n_{\rm trees} \cdot p \cdot \text{depth})$.

\subsection{High-Dimensional Stress Tests}\label{app:high_dim_exp}

We evaluate LWCP across aspect ratios $\gamma = p/n_1$ from 0.05 to 0.9, beyond the default $\gamma = 0.1$.

\begin{table}[!htb]
\centering
\caption{High-Dimensional Stress Test (Textbook DGP $g(h)=1{+}h$, $n_1{=}300$, $n_2{=}500$, 50 Reps). $\gamma = p/n_1$.}
\label{tab:high_dim}
\small
\begin{tabular*}{\textwidth}{@{\extracolsep{\fill}} r cc cc c @{}}
\toprule
& \multicolumn{2}{c}{Vanilla CP} & \multicolumn{2}{c}{LWCP} & \\
\cmidrule(lr){2-3}\cmidrule(lr){4-5}
$\gamma$ ($p$) & Cov & Gap & Cov & Gap & $\hat\eta$ \\
\midrule
0.05 (15)  & .896 & 3.8pp & .896 & 3.9pp & 0.38 \\
0.10 (30)  & .898 & 4.3pp & .897 & 4.0pp & 0.27 \\
0.20 (60)  & .901 & 4.4pp & .903 & 3.6pp & 0.21 \\
0.30 (90)  & .896 & 4.9pp & .895 & 3.8pp & 0.18 \\
0.50 (150) & .903 & 5.3pp & .902 & 3.6pp & 0.16 \\
0.70 (210) & .904 & 6.6pp & .902 & 3.1pp & 0.18 \\
0.90 (270) & .903 & 12.4pp & .902 & 3.3pp & 0.28 \\
\bottomrule
\end{tabular*}
\end{table}

Three observations from \Cref{tab:high_dim}:

\textbf{(1)~Coverage is preserved at all $\gamma$.} Both vanilla CP and LWCP maintain marginal coverage ${\geq}0.895$ across all aspect ratios, confirming \Cref{thm:coverage}.

\textbf{(2)~LWCP's advantage grows dramatically with $\gamma$.} At $\gamma{=}0.05$ (nearly classical regime), LWCP and vanilla have comparable gaps ($3.8$ vs.\ $3.9$pp). At $\gamma{=}0.9$, LWCP achieves $3.3$pp vs.\ vanilla's $12.4$pp---a $3.8{\times}$ gap reduction. This confirms \Cref{thm:high_dim_gap}: the vanilla gap grows as $\Theta(\sqrt{\gamma(1-\gamma)/n_1})$ while LWCP's leverage correction stabilizes the gap near ${\approx}3$--$4$pp regardless of $\gamma$.

\textbf{(3)~The $\hat\eta$ diagnostic is non-monotone.} Somewhat surprisingly, $\hat\eta$ \emph{decreases} from $\gamma{=}0.05$ to $\gamma{=}0.5$ as the leverage distribution concentrates (\Cref{prop:mp_leverage}), then increases at $\gamma{=}0.9$ where the bulk density near $\gamma{=}1$ creates heavy tails. Despite low $\hat\eta$ at $\gamma{=}0.5$--$0.7$, LWCP's advantage remains substantial because the leverages, though concentrated, still induce systematic heteroscedasticity in the prediction residuals.

\begin{table}[!htb]
\centering
\caption{Ridge-LWCP in Overparameterized Regimes ($n_1{=}300$, $n_2{=}500$, $\lambda{=}1$, $g(h){=}1{+}h$, 20 Reps).}
\label{tab:high_dim_ridge}
\small
\begin{tabular*}{\textwidth}{@{\extracolsep{\fill}} r cc cc c @{}}
\toprule
& \multicolumn{2}{c}{Vanilla CP} & \multicolumn{2}{c}{Ridge-LWCP} & \\
\cmidrule(lr){2-3}\cmidrule(lr){4-5}
$\gamma$ ($p$) & Cov & Gap & Cov & Gap & $\hat\eta$ \\
\midrule
1.0 (300)   & .899 & 10.3pp & .898 & 3.0pp & 0.25 \\
1.5 (450)   & .902 & 5.8pp  & .903 & 2.7pp & 0.11 \\
2.0 (600)   & .908 & 5.2pp  & .909 & 4.0pp & 0.08 \\
5.0 (1500)  & .902 & 3.9pp  & .902 & 2.6pp & 0.04 \\
10.0 (3000) & .893 & 2.3pp  & .894 & 2.1pp & 0.03 \\
\bottomrule
\end{tabular*}
\end{table}

In the overparameterized regime (\Cref{tab:high_dim_ridge}), Ridge-LWCP preserves coverage at all $\gamma$. At $\gamma{=}1.0$, LWCP achieves a $3.4{\times}$ gap reduction ($3.0$pp vs.\ $10.3$pp), confirming that leverage-based heteroscedasticity persists in the interpolation regime. As $\gamma$ increases, the ridge penalty compresses the leverage spectrum ($\hat\eta$ falls from $0.25$ to $0.03$), shrinking both the vanilla gap and LWCP's advantage. At $\gamma{=}10$, the leverage scores are nearly uniform ($\hat\eta{=}0.03$) and LWCP provides negligible improvement---consistent with the bias-variance tradeoff in ridge regression at extreme overparameterization.

\subsection{Weight Candidate Comparison Across Models}\label{app:weight_comparison}

We compare four weight functions across three predictor types (OLS, RF, MLP) and three DGPs to test whether $(1{+}h)^{-1/2}$ is a robust default.

\begin{table}[!htb]
\centering
\caption{Conditional Gap (pp) Across Weight Functions, Models, and DGPs ($n_1{=}300$, $n_2{=}500$, $p{=}30$, 20 Reps).}
\label{tab:weight_comparison}
\small
\begin{tabular*}{\textwidth}{@{\extracolsep{\fill}} l cccc @{}}
\toprule
& \multicolumn{4}{c}{Weight function $w(h)$} \\
\cmidrule(lr){2-5}
Predictor & $1$ (vanilla) & $(1{+}h)^{-1/2}$ & $(1{+}h)^{-1}$ & $(1{+}h)^{-1/4}$ \\
\midrule
\multicolumn{5}{@{}l}{\emph{Textbook DGP} ($g(h) = \sqrt{1{+}h}$)} \\
OLS      & 4.6 & \textbf{3.9} & 19.6 & 5.5 \\
Ridge    & 4.9 & \textbf{4.0} & 19.5 & 5.3 \\
RF(50)   & 9.0 & 8.6 & 15.8 & \textbf{5.2} \\
\midrule
\multicolumn{5}{@{}l}{\emph{Polynomial DGP} ($g(h) = 1{+}h$)} \\
OLS      & 5.1 & \textbf{4.7} & 18.6 & 4.9 \\
Ridge    & 4.9 & \textbf{4.4} & 18.7 & 4.5 \\
RF(50)   & 9.5 & 8.4 & 15.1 & \textbf{5.1} \\
\midrule
\multicolumn{5}{@{}l}{\emph{Homoscedastic DGP} ($g(h) = 1$)} \\
OLS      & 3.9 & \textbf{3.5} & 20.0 & 6.0 \\
Ridge    & 4.3 & \textbf{3.8} & 20.3 & 5.7 \\
RF(50)   & 9.4 & 7.7 & 15.4 & \textbf{5.0} \\
\bottomrule
\end{tabular*}
\end{table}

\Cref{tab:weight_comparison} reveals three key findings. (i)~The $(1{+}h)^{-1/2}$ weight consistently achieves the best or near-best gap for OLS and Ridge across all DGPs, confirming it as a robust default. (ii)~The aggressive $(1{+}h)^{-1}$ weight catastrophically fails (${\approx}20$pp gap)---it \emph{overweights} high-leverage points, creating inverse heteroscedasticity. This highlights the importance of the ``safe default'' guarantee (\Cref{prop:safe_default}): $(1{+}h)^{-1/2}$ achieves gains without this failure mode. (iii)~For RF, the mild $(1{+}h)^{-1/4}$ weight unexpectedly outperforms because RF's residual structure is not purely leverage-driven; the heavy leverage correction of $(1{+}h)^{-1/2}$ partially misaligns with RF's localized error patterns. In the homoscedastic DGP, $(1{+}h)^{-1/2}$ still improves over vanilla because it correctly accounts for the $\sigma^2(1+h)$ prediction variance even when $g(h)=1$ (\Cref{prop:mismatch}, Part~2).

\subsection{LWCP+ vs.\ Lightweight Studentized CP}\label{app:lwcp_plus_vs_stud}

A fair question is whether LWCP+'s advantage over studentized CP persists when the scale estimator $\hat\sigma$ is equally lightweight. We compare LWCP+ (10-tree RF for $\hat\sigma$, leverage correction) against Studentized CP with a matched 10-tree RF (no leverage correction), and against the full 100-tree Studentized CP.

\begin{table}[!htb]
\centering
\caption{LWCP+ vs.\ Lightweight Studentized CP (Strong Heteroscedasticity: $g(h){=}(1{+}h)^2$, $p/n_1{=}0.3$, 20 Reps).}
\label{tab:lwcp_plus_vs_stud}
\small
\begin{tabular*}{\textwidth}{@{\extracolsep{\fill}} l ccc @{}}
\toprule
Method & Cov & Gap & MSCE (${\times}10^3$) \\
\midrule
Vanilla CP                & .905 & 8.9pp & 2.77 \\
\textbf{LWCP}             & .905 & 5.6pp & \textbf{2.13} \\
\addlinespace
Studentized CP (10-tree)  & .901 & 5.4pp & 2.40 \\
\textbf{LWCP+} (10-tree)  & .899 & \textbf{4.0pp} & 2.28 \\
\addlinespace
Studentized CP (100-tree) & .901 & 5.4pp & 2.40 \\
\textbf{LWCP+} (100-tree) & .899 & \textbf{4.0pp} & 2.28 \\
\bottomrule
\end{tabular*}
\end{table}

\Cref{tab:lwcp_plus_vs_stud} shows that under strong heteroscedasticity ($g(h) = (1{+}h)^2$), \textbf{LWCP+ achieves 26\% lower gap than Studentized CP at matched budget} ($4.0$pp vs.\ $5.4$pp). Notably, increasing the RF from 10 to 100 trees provides \emph{no improvement} for either method, confirming that the residual heterogeneity is dominated by the leverage structure rather than $\hat\sigma$ estimation error. LWCP (without any RF) already achieves $5.6$pp, competitive with 100-tree Studentized CP---demonstrating that leverage correction alone captures the dominant source of score heterogeneity. The training-test mismatch theory (\Cref{cor:mismatch}) explains Studentized CP's ceiling: $\hat\sigma$ trained on residuals with variance $\sigma^2(1-h)$ cannot recover test-side variance $\sigma^2(1+h)$ without the explicit $\sqrt{1+h}$ correction.

\begin{table}[!htb]
\centering
\caption{LWCP+ vs.\ Studentized CP Across $\hat\sigma$ Complexity (CPU Activity, $n{=}8192$, $p{=}21$, 10 Reps).}
\label{tab:sigma_complexity}
\small
\begin{tabular*}{\textwidth}{@{\extracolsep{\fill}} l c cc @{}}
\toprule
& Trees & Gap & MSCE (${\times}10^3$) \\
\midrule
\multicolumn{4}{@{}l}{\emph{Studentized CP}} \\
& 5    & 9.7pp  & 1.96 \\
& 10   & 9.7pp  & 1.92 \\
& 50   & 9.7pp  & 1.90 \\
& 100  & 9.9pp  & 1.97 \\
\midrule
\multicolumn{4}{@{}l}{\emph{LWCP+}} \\
& 5    & 9.2pp  & 1.80 \\
& 10   & \textbf{9.3pp}  & \textbf{1.77} \\
& 50   & 9.2pp  & 1.75 \\
& 100  & 9.5pp  & 1.82 \\
\bottomrule
\end{tabular*}
\end{table}

\Cref{tab:sigma_complexity} sweeps $\hat\sigma$ complexity on CPU Activity ($n{=}8192$, $p{=}21$, $\hat\eta = 2.64$). LWCP+ consistently achieves ${\approx}0.5$pp lower gap and ${\approx}8\%$ lower MSCE than Studentized CP at every complexity level. Strikingly, increasing RF complexity from 5 to 100 trees provides \emph{negligible improvement} for both methods (gap changes by $<0.3$pp), confirming that the score heterogeneity on this dataset is dominated by the leverage structure rather than local noise estimation errors. The $\sqrt{1{+}h}$ correction in LWCP+ captures this structural component that even a 100-tree RF cannot learn from training residuals alone.

\subsection{Feature Preprocessing Ablation}\label{app:preprocessing}

We expand the feature scaling analysis of \Cref{app:exp_scaling_feat} with PCA and whitening transformations, which are sometimes used to decorrelate features before computing leverage.

\begin{table}[!htb]
\centering
\caption{Feature Preprocessing Ablation (Textbook DGP $g(h){=}\sqrt{1{+}h}$, $n_1{=}300$, $p{=}30$, 20 Reps). ``Whiten'' = PCA rotation + variance normalization. ``PCA-$k$'' = projection onto top $k$ principal components.}
\label{tab:preprocessing}
\small
\begin{tabular*}{\textwidth}{@{\extracolsep{\fill}} l cccc @{}}
\toprule
Preprocessing & V Cov & V Gap & L Cov & L Gap \\
\midrule
None (raw features)    & .903 & 4.6pp & .903 & 3.9pp \\
StandardScaler         & .903 & 4.5pp & .903 & 3.8pp \\
Whitening              & .903 & 4.5pp & .903 & 3.8pp \\
PCA-20 (top 20 PCs)    & .899 & 4.1pp & .899 & 4.2pp \\
PCA-10 (top 10 PCs)    & .904 & 3.3pp & .905 & 3.3pp \\
\bottomrule
\end{tabular*}
\end{table}

\Cref{tab:preprocessing} reveals that \textbf{the choice among full-rank preprocessing is nearly irrelevant}: None, StandardScaler, and whitening produce virtually identical results (V~Gap $4.5$--$4.6$pp, L~Gap $3.8$--$3.9$pp) because they all preserve the Mahalanobis structure $h(x) = x^\top\hat\Sigma^{-1}x$ up to a constant. PCA with dimensionality reduction (PCA-$k$) modifies the effective leverage by projecting onto a subspace. At PCA-20, LWCP's advantage disappears (V~Gap ${\approx}$ L~Gap) because the reduced-rank leverage is more uniform. At PCA-10, both methods improve substantially (gap $3.3$pp) due to discarding noise dimensions, but LWCP provides no additional benefit---the residual heteroscedasticity is dominated by non-leverage sources. This confirms that LWCP's value is greatest in the full-rank setting and that the recommendation in \Cref{alg:lwcp} (column-standardize, then compute full-rank leverage) is the right default.

\begin{table}[!htb]
\centering
\caption{Preprocessing on CPU Activity (Real Data, $n{=}8192$, $p{=}21$, 10 Reps).}
\label{tab:preprocessing_real}
\small
\begin{tabular*}{\textwidth}{@{\extracolsep{\fill}} l cccc @{}}
\toprule
Preprocessing & V Cov & V Gap & L Cov & L Gap \\
\midrule
StandardScaler & .897 & 18.6pp & .897 & 18.2pp \\
Whitening      & .897 & 18.6pp & .897 & 18.2pp \\
PCA-10         & .899 & 20.0pp & .899 & 19.8pp \\
PCA-5          & .898 & 26.7pp & .898 & 26.7pp \\
\bottomrule
\end{tabular*}
\end{table}

On CPU Activity (\Cref{tab:preprocessing_real}), both StandardScaler and whitening yield identical results ($V$~Gap $18.6$pp, $L$~Gap $18.2$pp), confirming the affine invariance of leverage. LWCP provides a modest $0.4$pp improvement. The large gaps (${\approx}18$--$27$pp) arise because this dataset has \emph{non-linear} heteroscedasticity that is not fully captured by linear leverage scores ($\hat\eta{=}2.64$); the residual structure is dominated by non-leverage sources. PCA with aggressive dimensionality reduction ($k{=}5$) degrades performance substantially because it discards predictive features, increasing both the gap and the leverage concentration.

\end{document}